\PassOptionsToPackage{table}{xcolor}
\documentclass[a4paper,fleqn]{cas-dc}

\usepackage{amsmath}
\usepackage{amssymb}
\usepackage[numbers,sort&compress]{natbib}
\usepackage{verbatim}
\usepackage{cases}
\allowdisplaybreaks[4]
\usepackage{enumerate}
\usepackage{booktabs}
\usepackage{balance}
\usepackage{mathbbol}
\usepackage{multirow}
\usepackage{adjustbox}
\usepackage{subcaption}
\usepackage{algorithm}
\usepackage{algorithmicx}
\usepackage{algpseudocode}
\usepackage{mathrsfs}
\usepackage{threeparttable}
\usepackage{makecell}
\usepackage{enumitem}

\newtheorem{definition}{Definition}
\newtheorem{proof}{Proof}
\newtheorem{theorem}{Theorem}

\newtheorem{proposition}{Proposition}
\newtheorem{remark}{Remark}

\newtheorem{corollary}{Corollary}

\newtheorem{assumption}{Assumption}

\def\g{{\bf g}}

\def\0{{\bf 0}}
\def\1{{\bf 1}}

\begin{document}
\let\WriteBookmarks\relax
\def\floatpagepagefraction{1}
\def\textpagefraction{.001}

\shorttitle{Sim2Real-AD: A Modular Sim-to-Real Framework}
\shortauthors{Z. Huang et al.}

\title[mode=title]{Sim2Real-AD: A Modular Sim-to-Real Framework for Deploying VLM-Guided Reinforcement Learning in Real-World Autonomous Driving}

\author[1]{Zilin Huang}
\author[1]{Zhengyang Wan}
\author[1]{Zihao Sheng}
\author[1]{Boyue Wang}
\author[1]{Junwei You}
\author[1]{Sikai Chen}
\cormark[1]
\ead{sikai.chen@wisc.edu}
\affiliation[1]{organization={Department of Civil and Environmental Engineering, University of Wisconsin-Madison},
    city={Madison},
    state={WI},
    postcode={53706},
    country={USA}}
\cortext[cor1]{Corresponding author}

\begin{abstract}
Autonomous driving is central to intelligent and increasingly electrified transportation. Vision-language-model (VLM)-guided reinforcement learning (RL) has recently attracted significant attention for it, replacing brittle hand-crafted rewards with semantically grounded signals; however, deploying such simulation-trained policies on real vehicles remains a fundamental challenge, because they rely on simulator-native observations and simulator-coupled action semantics with no counterpart on physical hardware. We identify a general principle: the simulation-to-reality gap decomposes into two largely orthogonal axes, a sensing-and-dynamics domain gap and a task-and-geometry gap, the former closable without real-world policy training by re-projecting real perception and control onto the policy's training manifold. We formalize this as a transfer guarantee that bounds the deployment gap by three independently controllable error terms, and instantiate it as \textbf{Sim2Real-AD}, which combines a Geometric Observation Bridge, a Physics-Aware Action Mapping, a Two-Phase Progressive Training curriculum, and a Real-time Deployment Pipeline. As a proof of concept, a CARLA-trained VLM-guided RL policy is transferred zero-shot to a full-scale battery-electric Ford E-Transit van in Madison, WI, USA, and drives across car-following, obstacle-avoidance, and stop-sign scenarios using no real-world training data. To our knowledge, this is among the first zero-shot closed-loop deployments of a CARLA-trained VLM-guided RL policy on a full-scale real vehicle, and the decomposition offers a principled, broadly applicable route for moving simulation-trained, foundation-model-guided policies into the physical world, supporting energy-efficient intelligent driving on electrified transportation platforms.
The demo video, code, and model checkpoint are
available at:
\href{https://zilin-huang.github.io/Sim2Real-AD-website/}
{\textcolor{magenta}{https://zilin-huang.github.io/Sim2Real-AD-website/}}.
\end{abstract}


\maketitle
\section{Introduction}
Foundation models, including large language and vision-language models (VLMs), have achieved striking success across perception, reasoning, and generation, and are now reshaping how control policies are built: they supply semantic supervision that lets agents learn complex, safety-critical behaviors that were previously hard to specify by hand \citep{wang2025alpamayo}. The overwhelming majority of these advances, however, are demonstrated in \emph{simulation}, where observations are clean and privileged and actions carry idealized semantics. The broader ambition of \emph{Physical AI}, in which foundation-model-trained policies actually perceive and act in the real world, requires crossing from simulation to physical hardware, where the simulator-native observations and simulator-coupled action semantics these policies were trained with simply do not exist. Closing this simulation-to-physical gap, without retraining in the physical world, is a central open problem on the path toward embodied, physically deployed intelligence.

Autonomous driving is a flagship instance of this challenge and a cornerstone of intelligent, increasingly electrified transportation~\citep{jia2025lane}: reliable operation in open-world traffic remains difficult because real roads contain long-tail events, uncertain human behavior, and continuously changing conditions \citep{tang2026hermes,xu2025wod,qu2025metassc}, and dependable deployment beyond restricted operational design domains is still unresolved. Because collecting such experience on real vehicles is costly and unsafe, high-fidelity simulation has become a cornerstone of modern
autonomous driving development. Platforms such as CARLA provide scalable and
risk-free environments for training and evaluation under diverse road layouts, 
weather conditions, and traffic configurations \citep{dosovitskiy2017carla}. 
Closed-loop benchmarks such as NoCrash \citep{codevilla2019exploring}, the CARLA
Leaderboard \citep{carla_leaderboard}, and Bench2Drive \citep{jia2024bench2drive}
measure progress in increasingly realistic settings. Within this
simulation-centered paradigm, two major learning strategies have emerged. 
Imitation learning (IL) methods, such as UniAD \citep{hu2023planning}, learn policies by mimicking expert
demonstrations and achieve strong benchmark performance, but their
behavior is fundamentally bounded by the coverage of the demonstration
data and degrades in rare or unseen situations \citep{codevilla2019exploring}. Reinforcement 
learning (RL), in contrast, optimizes policies through interaction and 
reward-driven exploration, offering a pathway to discover behaviors beyond 
those explicitly present in human data \citep{kiran2021deep}. A particularly promising line integrates VLMs into RL reward design, replacing brittle hand-crafted rewards with semantically grounded signals, as in VLM-RL \citep{huang2025vlm} and DriveVLM-RL \citep{huang2026drivevlmrl}. Like the broader VLM-guided RL literature, however, these methods are validated entirely within CARLA, leaving their transfer from simulator-native observations and simulator-coupled control semantics to full-scale physical vehicles an open problem.

This deployment difficulty is fundamentally a simulation-to-reality (sim-to-real) gap, arising because simulator
training and real-world deployment differ in both observation
and control semantics \citep{salvato2021crossing,daza2023sim,li2024platform}.
In this work, we focus on two dominant and practically decisive components of 
this gap, as illustrated in Fig.~\ref{fig:intro_overview}. The first is the 
\textit{observation gap}. Many RL driving policies in CARLA rely on 
privileged or simulator-native observations, such as ground-truth 
bird's-eye-view (BEV) semantic masks, that are clean, structured, and 
spatially complete \citep{dosovitskiy2017carla}. In contrast, a real vehicle 
must infer scene structure from camera inputs that are noisy, partially 
occluded, and restricted by field of view, creating a substantial cross-domain 
mismatch \citep{tobin2017domain,zhu2017unpaired,ganin2016domain}. The second 
is the \textit{dynamics gap}. In simulation, policy outputs are often 
interpreted directly as low-level control commands, such as steering and 
throttle/brake. On a physical platform, however, the same commands can 
induce substantially different vehicle responses because of differences in 
wheelbase, steering ratio, actuator latency, tire-road interaction, and 
longitudinal dynamics \citep{daza2023sim,salvato2021crossing,li2024platform}. 
These two gaps compound each other: when a policy trained on perfect BEV 
receives the noisy, limited-coverage observation produced by a real monocular 
camera, and its control outputs are simultaneously misinterpreted by a vehicle 
with different physical dynamics, performance degrades catastrophically. 


\begin{figure*}
    \centering
    \includegraphics[width=0.99\linewidth]{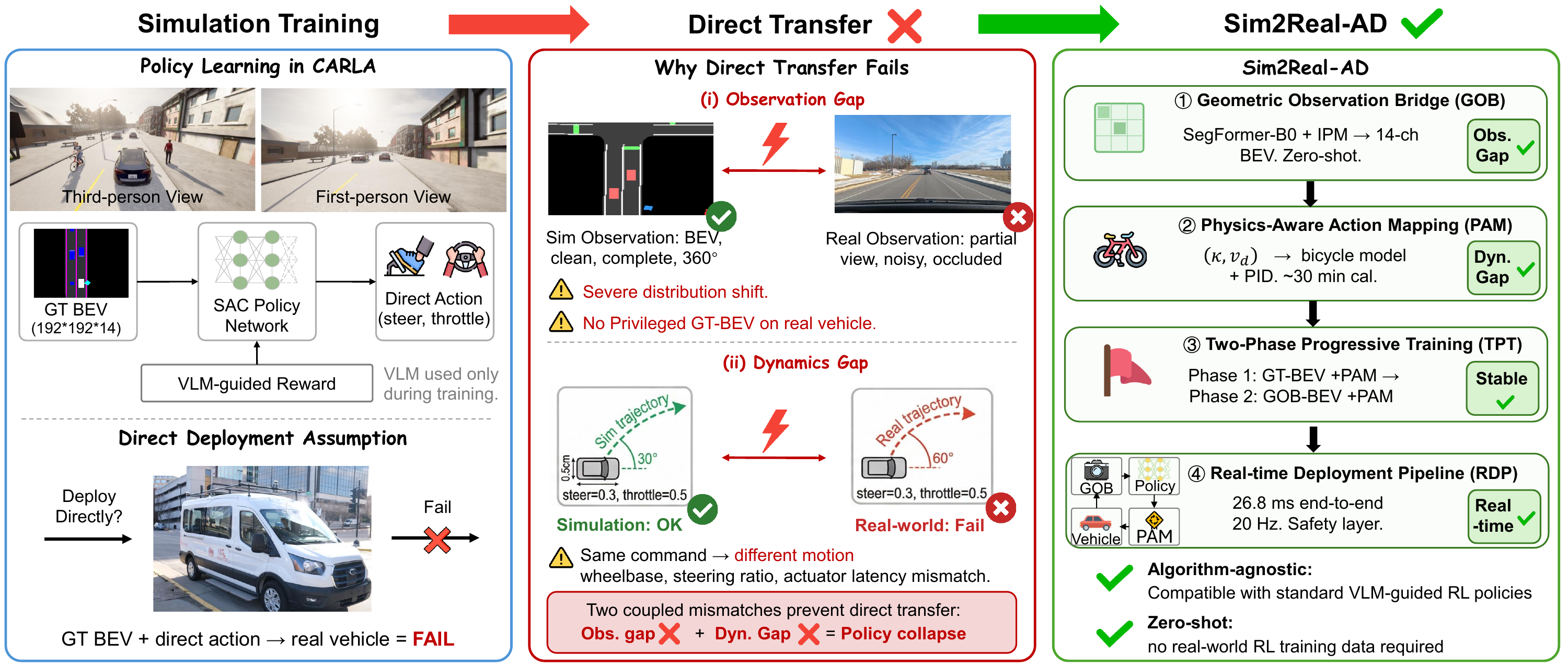}
    \caption{Overview of the sim-to-real challenge and the proposed Sim2Real-AD framework. Direct transfer fails because of the coupled observation and dynamics gaps, while Sim2Real-AD addresses them through GOB, PAM, TPT, and RDP.}
    \label{fig:intro_overview}
\end{figure*}

Existing sim-to-real methods only partially address this setting. Most tackle either the observation gap \emph{or} the dynamics gap in isolation, and the few that consider both rely on learned black-box alignment modules that are hard to debug, require substantial real-world data, and must be retrained when the deployment vehicle changes \citep{tobin2017domain,zhu2017unpaired,daza2023sim,salvato2021crossing,li2024platform}. To address this challenge, we start from a simple principle: the sim-to-real problem decomposes into two largely orthogonal axes, a sensing-and-dynamics domain gap that subsumes both the observation and dynamics gaps above and a task-and-geometry gap arising from differences in route topology and scene complexity, and the former can be closed \emph{without any real-world policy training} by re-projecting real perception and control onto the manifold on which the policy was trained. We formalize this as a transfer guarantee (Theorem~\ref{thm:main_informal}) that bounds the deployment performance gap by three independently controllable error terms, and instantiate it as \textbf{Sim2Real-AD}, a modular sim-to-real deployment framework for VLM-guided RL autonomous driving.
Concretely, Sim2Real-AD comprises four modules: a \textit{Geometric Observation Bridge (GOB)} that turns monocular front-view images into simulator-compatible BEV observations, a \textit{Physics-Aware Action Mapping (PAM)} that recasts policy outputs as platform-agnostic physical commands, a \textit{Two-Phase Progressive Training (TPT)} curriculum that adapts the action and observation interfaces in sequence rather than simultaneously, and a \textit{Real-time Deployment Pipeline (RDP)} that integrates them into a closed-loop real-vehicle system. We detail each module in Section~\ref{sec:framework}.

We validate Sim2Real-AD through extensive simulation experiments in CARLA and zero-shot closed-loop deployment on a full-scale battery-electric Ford E-Transit van in Madison, WI, USA, using only lightweight platform calibration and no real-world RL training data. As a proof-of-concept case study, the transferred policy drives across car-following, obstacle-avoidance, and stop-sign scenarios. \textbf{To the best of our knowledge, this study is among the first to demonstrate zero-shot closed-loop deployment of a CARLA-trained VLM-guided RL policy on a full-scale real vehicle without any real-world RL training data.}

The main contributions of this paper are summarized as follows:
\begin{itemize}
    \item \textbf{A transferable principle with a transfer guarantee.} We show that the sim-to-real gap decomposes into two largely orthogonal axes, a sensing-and-dynamics domain gap and a task-and-geometry gap, the former closable \emph{without any real-world policy training} by re-projecting real perception and control onto the policy's training manifold. We formalize this as a transfer guarantee bounding the deployment gap by three independently controllable error terms, to our knowledge the first for a foundation-model-guided RL policy on a physical vehicle and not specific to driving.

    \item Building on this principle, we instantiate it as \textbf{Sim2Real-AD}, a modular and reward-agnostic sim-to-real deployment framework for VLM-guided RL autonomous driving that explicitly decomposes transfer into an observation-space bridge and an action-space bridge, enabling real-world deployment without real-world policy training or learning-based domain adaptation. The framework is broadly compatible with RL-based driving policies that use structured simulator-native observations.
    
    \item We introduce a geometric observation bridge that 
    transforms monocular camera images into a unified BEV semantic 
    representation using pre-trained segmentation and inverse perspective 
    mapping, substantially reducing cross-domain observation discrepancy in 
    a fully interpretable, training-free manner.
    
    \item We design a physics-aware action mapping together 
    with a two-phase progressive training strategy, which 
    decouple policy learning from platform-specific control semantics and 
    support zero-shot policy transfer through lightweight platform calibration 
    requiring only $\sim$30 minutes and no real-world driving data.
\end{itemize}

The remainder of this paper is organized as follows. 
Section~\ref{sec:related} reviews related work. 
Section~\ref{sec:preliminaries} introduces the necessary background and 
problem formulation. Section~\ref{sec:framework} presents the proposed 
Sim2Real-AD framework in detail. Section~\ref{sec:experiments} reports 
simulation studies in CARLA and zero-shot real-vehicle deployment results. 
Section~\ref{sec:limitations} discusses limitations, and
Section~\ref{sec:conclusion} concludes the paper and outlines future directions.

\section{Related Work}
\label{sec:related}

\subsection{Learning-based Driving Policies in Simulation}

End-to-end autonomous driving has advanced along two 
main paradigms. IL methods, such as UniAD \citep{hu2023planning},
learn by mimicking expert demonstrations and achieve
strong benchmark performance, but are bounded by
demonstration coverage and fail in
out-of-distribution scenarios \citep{codevilla2019exploring}.
RL offers a complementary path through reward-driven 
exploration \citep{sheng2024traffic,huang2024human,jiang2025alphadrive,chen2025personalized,li2025investigation}.
However, manually engineered rewards are labor-intensive
and generalize poorly \citep{delavari2025comprehensive};
VLM-guided RL instead encodes semantic goals through
encoders such as CLIP \citep{radford2021learning} to
replace hand-crafted objectives with grounded reward signals. VLM-RL \citep{huang2025vlm} first established the Contrasting Language Goal (CLG) paradigm, using CLIP-based rewards from paired positive and negative language descriptions in CARLA, while DriveMind \citep{wasif2025drivemind} and Found-RL \citep{qu2026found} explore related grounded-reward designs. DriveVLM-RL \citep{huang2026drivevlmrl} extended it with a neuroscience-inspired dual-pathway architecture whose attention-gated large VLM (LVLM) performs semantic risk reasoning, achieving state-of-the-art collision avoidance. Notably, these methods use the VLM and LVLM only during offline training and remove them at deployment, leaving a lightweight policy with no test-time inference overhead. A parallel line explores
Vision-Language-Action (VLA) models that unify
perception, reasoning, and control end-to-end \citep{renz2025simlingo,wang2025alpamayo,qian2025agentthink,zhou2025autovla}, but they face an even more acute deployment challenge, since large-model inference at test time makes real-time control infeasible without specialized decoupling \citep{jiang2025survey}.
Across both lines, the focus remains on policy learning
within simulation, and transferring such policies from
simulator-native observations and simulator-coupled
control semantics to full-scale real vehicles remains
largely unanswered. 
This reflects a broader open problem in robotics and embodied control: moving policies trained with foundation-model guidance in simulation onto physical systems where their privileged observations and simulator-coupled actions do not exist. We address it for driving with a principled, theoretically grounded deployment framework for CARLA-trained VLM-guided RL policies.

\subsection{Sim-to-Real Methods for Autonomous Driving 
and Robotics}

Domain randomization \citep{tobin2017domain} improves robustness via diverse
training perturbations, but neither applies to structured
multi-channel BEV observations nor bridges action
semantics across platforms.
Domain adaptation aligns distributions through 
image translation (CycleGAN \citep{zhu2017unpaired}) or 
adversarial feature learning (DANN \citep{ganin2016domain}), 
but requires substantial real-world data and retraining 
when the deployment vehicle changes.
Canonical representation mapping bridges the
observation gap via a shared space: learned
approaches such as Lift-Splat-Shoot \citep{philion2020lift}
and BEVFormer \citep{li2024bevformer} offer strong
camera-to-BEV perception but require large-scale
supervised training. A lightweight alternative is Inverse 
Perspective Mapping (IPM) \citep{bertozzi1998gold}, which 
projects road-plane pixels into a top-down BEV without 
training data. Combined with a pre-trained segmentation 
model \citep{xie2021segformer}, IPM yields an 
interpretable, calibration-efficient bridge transferable 
across platforms via simple camera recalibration, the 
foundation of our GOB module.
System identification and controller calibration 
\citep{daza2023sim,salvato2021crossing,li2024platform} 
address the dynamics gap by fitting vehicle models to 
real behavior, but target low-level tracking controllers 
rather than the simulator-coupled action semantics of 
end-to-end RL policies.
Curriculum-based adaptation
\citep{bengio2009curriculum,salvato2021crossing,zhao2020sim} 
reduces optimization difficulty by progressively exposing 
the policy to more realistic distributions. Our TPT 
strategy applies this principle to the dual-gap setting, 
decoupling action-space adaptation from observation-space 
adaptation. Unlike Rapid Motor Adaptation (RMA) 
\citep{kumar2021rma}, which requires online real-world 
data, TPT relies entirely on geometry-based and 
physics-based bridging.
Related efforts span digital twins, reality-gap
modeling, deployment-oriented pipelines
\citep{voogd2023reinforcement,daza2023sim,li2024platform,huang2025sky},
and BEV generation and segmentation systems
\citep{jun2025comparative}. In contrast, our framework jointly addresses both gaps through interpretable modules and a transfer-error analysis linking module-level imperfections to sim-to-real degradation. To our knowledge, no prior sim-to-real driving framework couples such training-free bridging with a formal guarantee on the resulting transfer gap.

\subsection{Real-World Deployment of Simulation-Trained 
Driving Policies}

Although simulation-based learning has advanced rapidly,
real-world deployment remains limited. Most
simulator-trained policies are tightly coupled to their
training environment: their observations rely on
simulator-native representations (e.g., privileged
BEV masks) unavailable on real vehicles, and their action
semantics are calibrated to simulator dynamics. Without
mechanisms to bridge both couplings, direct deployment
fails even when the policy performs well in simulation.
Many studies validate only in CARLA or evaluate 
deployment-oriented ideas through offline replay, 
shadow-mode analysis, or scaled platforms rather than 
full-scale closed-loop operation 
\citep{delavari2025comprehensive,voogd2023reinforcement,liu2026learning}. 
Prior sim-to-real work often targets specialized tasks 
such as drifting or parking \citep{toth2024sim} or 
requires additional real-world adaptation after 
simulation training \citep{lin2025model}. World-model 
and generative-simulation approaches  
\citep{you2024bench2drive,ji2026world} strengthen training and
evaluation infrastructure but do not demonstrate direct 
deployment of simulator-trained RL policies on physical 
vehicles.
Foundation model-based methods have shown strong 
simulation performance, but zero-shot closed-loop 
deployment on full-scale real vehicles remains very 
limited \citep{li2024platform,voogd2023reinforcement}. 
The challenge is particularly acute for VLM-guided RL 
\citep{wasif2025drivemind,huang2025vlm,huang2026drivevlmrl}: 
these policies depend on simulator-privileged BEV 
observations no real sensor can replicate, and their 
action spaces are implicitly calibrated to simulator 
dynamics. Yet no work starts from a CARLA-trained VLM-guided RL policy and deploys it zero-shot on a full-scale vehicle, nor provides a formal guarantee bounding the resulting transfer gap. We close this gap by combining GOB, PAM, TPT, and RDP into a complete pipeline, supported by a decomposition principle and transfer guarantee that, beyond driving, help characterize when simulation-trained, foundation-model-guided policies can be deployed on physical systems without real-world training.


\section{Preliminaries and Problem Formulation}
\label{sec:preliminaries}

\subsection{Preliminaries}

\textbf{Partially Observable Markov Decision Process.}
In a closed-loop simulator, autonomous driving is 
formulated as a Partially Observable Markov Decision 
Process (POMDP) $\mathcal{M}_{\mathrm{sim}} = 
(\mathcal{S}, \mathcal{A}, \mathcal{T}, \mathcal{O}, 
\mathcal{R}, \phi, \gamma, d_0)$, where $\mathcal{S}$ 
is the state space, $\mathcal{A}$ is the action space, 
$\mathcal{T}(s' \mid s, a)$ is the transition dynamics, 
$\mathcal{O}$ is the observation space, 
$\mathcal{R}(s, a)$ is the reward function, 
$\phi(o \mid s)$ is the observation emission function, 
$\gamma \in (0,1)$ is the discount factor, and $d_0$ is 
the initial state distribution. Because the agent does 
not observe the full state $s$ directly, it receives 
an observation $o_t \in \mathcal{O}$ at each timestep 
and selects an action $a_t \in \mathcal{A}$ according 
to its policy $\pi_\theta(a_t \mid o_t)$. The policy 
is trained to maximize the expected discounted return 
\citep{sutton1998reinforcement,kiran2021deep,huang2025pe}:
\begin{equation}
J(\pi_\theta) = \mathbb{E}_{\pi_\theta}
\!\left[\sum_{t=0}^{T} \gamma^t r_t \right].
\label{eq:rl_objective}
\end{equation}

\textbf{VLM-Guided RL Reward.}
In recent VLM-guided RL frameworks 
\citep{huang2025vlm,wasif2025drivemind}, the observation 
$o_t$ is a structured simulator-native representation 
comprising a BEV semantic tensor, route information, 
and vehicle states. The reward function combines 
conventional driving objectives with semantic 
supervision derived from VLMs or CLIP-style encoders. 
Following the formulation of DriveVLM-RL 
\citep{huang2026drivevlmrl}, the reward can be 
expressed in the general form:
\begin{equation}
r_t = \mathcal{F}(r_t^{\mathrm{task}},\, 
r_t^{\mathrm{sem}},\, s_t),
\label{eq:reward_combined}
\end{equation}
where $r_t^{\mathrm{task}}$ is a sparse task reward 
(e.g., collision penalty), $r_t^{\mathrm{sem}}$ is a 
dense semantic reward derived from VLM-based 
visual-language alignment, and $s_t$ denotes vehicle 
state information. The function $\mathcal{F}$ 
represents the reward synthesis mechanism, which in 
DriveVLM-RL takes a multiplicative hierarchical form 
that integrates both components with vehicle dynamics 
constraints. Critically, all VLM components used to 
compute $r_t^{\mathrm{sem}}$ operate 
\emph{exclusively during training}: once training is 
complete, the reward computation is discarded entirely, 
and the deployed policy $\pi_\theta$ executes as a
lightweight neural network with no VLM inference at
test time.

\subsection{Problem Formulation}
\label{sec:problem}
Let $\pi_\theta^{\mathrm{sim}}$ denote a driving policy 
trained entirely in CARLA under the VLM-guided RL 
setting above, where the observation space 
$\mathcal{O}^{\mathrm{sim}}$ consists of 
simulator-native BEV representations and the action 
space $\mathcal{A}^{\mathrm{sim}}$ is calibrated to 
simulator dynamics. Let $\mathcal{O}^{\mathrm{real}}$ 
and $\mathcal{T}^{\mathrm{real}}$ denote the 
observation space and transition dynamics of the 
target real vehicle. Direct deployment of 
$\pi_\theta^{\mathrm{sim}}$ fails because two 
systematic mismatches arise:
\begin{equation}
\mathcal{O}^{\mathrm{sim}} \not\approx 
\mathcal{O}^{\mathrm{real}},
\qquad
\mathcal{T}^{\mathrm{sim}} \not\approx 
\mathcal{T}^{\mathrm{real}}.
\end{equation}

The \emph{observation gap} arises because 
$\pi_\theta^{\mathrm{sim}}$ is trained on privileged 
simulator BEV observations 
$\hat{o}_t^{\mathrm{sim}} \in 
\hat{\mathcal{O}}^{\mathrm{sim}} \subset 
\mathbb{R}^{192 \times 192 \times 14}$ that are 
spatially complete and semantically clean, whereas 
the real platform can only provide raw monocular 
front-view images from which BEV-like representations 
must be reconstructed with inherent noise, limited 
field of view, and segmentation imperfections. The 
\emph{dynamics gap} arises because even the same 
nominal control command induces substantially different 
motion in simulation and on a real vehicle due to 
differences in wheelbase, steering ratio, actuator 
latency, tire-road interaction, and low-level control 
response.

Importantly, the two gaps interact: a policy receiving
unfamiliar observations will generate unreliable action
outputs, which are then further distorted by dynamics
it was not trained to account for. As we show in
Section~\ref{sec:ablation}, naively transferring the
original policy under the real-vehicle observation
without any bridging falls to a lower bound well below
the deployable pipeline, motivating a principled
sequential curriculum that bridges the action and
observation gaps in turn rather than confronting their
combined distributional shift at once.

The sim-to-real deployment problem is therefore to 
construct an \emph{observation bridge} 
$\mathcal{G}: \mathcal{O}^{\mathrm{real}} \rightarrow 
\hat{\mathcal{O}}^{\mathrm{sim}}$ and an 
\emph{action bridge} 
$\mathcal{M}: \hat{\mathcal{A}} \rightarrow 
\mathcal{A}^{\mathrm{real}}$, where 
$\hat{\mathcal{O}}^{\mathrm{sim}} \subset 
\mathbb{R}^{192 \times 192 \times 14}$ denotes the 
simulator-compatible BEV observation space and 
$\hat{\mathcal{A}}$ is a platform-agnostic physical 
action space, such that the composed policy:
\begin{equation}
\pi^{\mathrm{real}}(a_t^{\mathrm{real}} \mid 
o_t^{\mathrm{real}})
= \mathcal{M}\!\left(
\pi_\theta^{\mathrm{sim}}\!\left(
\hat{o}_t^{\mathrm{sim}},\, s_t,\, w_t
\right)\right),
\label{eq:composed_policy}
\end{equation}
where $\hat{o}_t^{\mathrm{sim}} = 
\mathcal{G}(o_t^{\mathrm{real}}) \in 
\hat{\mathcal{O}}^{\mathrm{sim}}$ denotes the 
simulator-compatible BEV observation produced by the 
GOB module, $s_t$ denotes the ego vehicle state 
(speed, steering, throttle), and $w_t$ denotes the 
future waypoint sequence, both available from 
on-board sensors and a GPS route provider at 
deployment time. The goal is for this composed policy 
to achieve safe and stable closed-loop driving on the 
real vehicle under zero-shot transfer, using only 
lightweight calibration and without any real-world 
RL training data.


\section{Framework: Sim2Real-AD}
\label{sec:framework}

\subsection{Overview}
\label{sec:overview}

The role of Sim2Real-AD is not to redesign the 
VLM-guided RL algorithm itself, but to make such a simulator-trained policy transferable to the real world. Operationally, GOB and PAM close the sensing-and-dynamics domain gap by re-projecting real observations and control onto the manifold on which the policy was trained, while TPT controls the residual distribution shift, instantiating the decomposition formalized in Section~\ref{sec:theory_informal}. In this work, we instantiate the framework 
using DriveVLM-RL~\citep{huang2026drivevlmrl}, a 
representative and state-of-the-art VLM-guided RL 
framework for safe autonomous driving in CARLA. In DriveVLM-RL, the policy is trained on
simulator-native observations (BEV semantic masks,
route information, and vehicle states) with VLM-derived
semantic rewards. This yields effective, safety-oriented
policies in simulation, but the result is not directly
deployable on a physical vehicle because both its
observation interface and its action semantics are tied
to the simulator. Crucially, however, DriveVLM-RL invokes its VLM and LVLM components \emph{only} during offline training and removes them at deployment, so the deployed policy is a lightweight network with no VLM inference at test time~\citep{huang2026drivevlmrl}. This makes DriveVLM-RL a particularly suitable deployment target: the remaining barrier is not the computational cost of foundation models, but solely the simulator-bound observation and action interfaces, which is precisely what Sim2Real-AD bridges. As illustrated in Fig.~\ref{fig:overview}, the
proposed framework consists of four main components: 
GOB, PAM, TPT, and RDP. 

The GOB reduces the observation gap by 
converting real monocular front-view images into a 
simulator-compatible BEV semantic representation 
$\hat{o}_t^{\mathrm{sim}} \in 
\hat{\mathcal{O}}^{\mathrm{sim}}$. The PAM reduces 
the dynamics gap by redefining the policy output in 
terms of platform-agnostic physical quantities and 
translating them into executable vehicle commands. 
The TPT strategy stabilizes transfer by decoupling 
action-space adaptation from observation-space 
adaptation. Finally, the RDP integrates all modules 
into a complete closed-loop system for deployment on 
a real vehicle. Together they replace DriveVLM-RL's simulator-dependent
observation and action interfaces with deployment-oriented
bridges while preserving its simulation-training advantages,
enabling zero-shot closed-loop deployment on a real vehicle
with only lightweight platform calibration and no real-world
RL training or fine-tuning. Since the four modules operate
purely on the policy's observation and action interfaces,
Sim2Real-AD remains agnostic to the reward design used during
simulator training.

\begin{figure*}[t]
    \centering
    \includegraphics[width=0.99\linewidth]{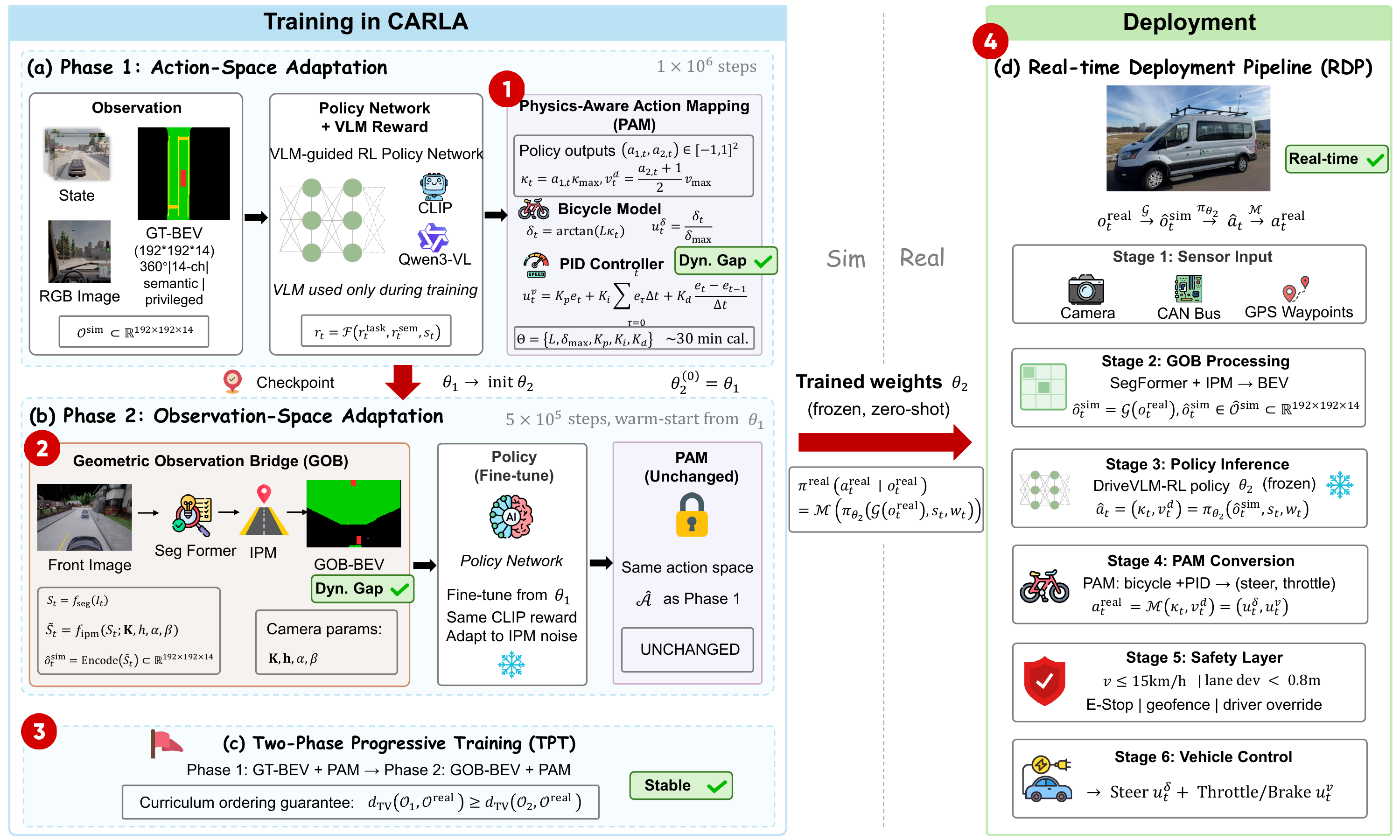}
    \caption{Overview of Sim2Real-AD. The framework bridges 
sim-to-real transfer through four components: the 
Geometric Observation Bridge (GOB), the Physics-Aware 
Action Mapping (PAM), the Two-Phase Progressive 
Training strategy (TPT), and the Real-time Deployment 
Pipeline (RDP). Instantiated here with DriveVLM-RL as the backbone, the framework is reward-agnostic and is demonstrated across multiple RL reward paradigms. } 

    \label{fig:overview}
\end{figure*}

\subsection{Geometric Observation Bridge}
\label{sec:gob}
The observation gap defined in 
Section~\ref{sec:problem} prevents direct deployment 
of $\pi_\theta^{\mathrm{sim}}$ because the policy 
expects structured BEV observations 
$\hat{o}_t^{\mathrm{sim}} \in 
\hat{\mathcal{O}}^{\mathrm{sim}}$ that are unavailable 
on a real vehicle. The GOB realizes the mapping 
$\mathcal{G}: \mathcal{O}^{\mathrm{real}} \rightarrow 
\hat{\mathcal{O}}^{\mathrm{sim}}$ by converting raw 
monocular front-view images into this 
simulator-compatible format through a two-step 
geometric pipeline. Unlike learned multi-camera BEV 
pipelines that aim to maximize perception accuracy 
under richer sensor setups 
\citep{philion2020lift,li2024bevformer,jun2025comparative}, 
our goal is to construct a lightweight 
simulator-compatible observation interface under the 
minimal monocular deployment configuration considered 
in this work.

\subsubsection{Unified BEV Representation}

Rather than letting the policy consume raw RGB images 
whose appearance differs substantially across domains, 
we process both simulated and real front-view images 
through the same deterministic camera-to-BEV pipeline 
and expose the policy only to the resulting BEV 
tensor. In this way, the policy always receives 
observations in the same spatially structured format, 
even though the underlying image source differs.

Let $I_t \in \mathbb{R}^{H \times W \times 3}$ denote 
the monocular RGB image captured at time step $t$. 
We first apply a semantic segmentation network to 
obtain a pixel-wise semantic map:
\begin{equation}
S_t = f_{\mathrm{seg}}(I_t),
\end{equation}
where $f_{\mathrm{seg}}(\cdot)$ denotes the 
segmentation model. In our implementation we use 
SegFormer-B0~\citep{xie2021segformer}, which provides 
a favorable trade-off between segmentation quality 
and inference efficiency: it achieves competitive mIoU 
while running at over 30 FPS on a single GPU, 
satisfying the real-time constraint. No fine-tuning 
on domain-specific data is performed; the model is 
applied zero-shot to both simulation and real-world 
images. The segmentation output contains 
driving-relevant semantic regions, including road 
surface, lane markings, vehicles, pedestrians, and 
traffic-related classes.

Next, we apply Inverse Perspective Mapping 
(IPM)~\citep{bertozzi1998gold} to project the 
segmented front-view image into a top-down BEV space:
\begin{equation}
\tilde{S}_t = f_{\mathrm{ipm}}(S_t;\, \mathbf{K},\, 
h,\, \alpha,\, \beta),
\label{eq:ipm}
\end{equation}
where $\mathbf{K} \in \mathbb{R}^{3\times 3}$ is the 
camera intrinsic matrix and $(h, \alpha, \beta)$ 
denote the camera mounting height, pitch, and roll. 
Under the planar-ground assumption, IPM maps image 
pixels on the road plane into a fixed BEV coordinate 
system representing a metric crop of 
$20\,\text{m} \times 20\,\text{m}$ centered on the 
ego vehicle. Camera calibration uses a standard checkerboard
procedure~\citep{zhang2000flexible} (about 15~min) and is performed
once per deployment vehicle.

\subsubsection{Multi-Channel BEV Construction}

The projected semantic map is encoded into a 
multi-channel BEV tensor:
\begin{equation}
\hat{o}_t^{\mathrm{sim}} = 
\mathrm{Encode}(\tilde{S}_t)
= \mathrm{Encode}\!\left(
    f_{\mathrm{ipm}}\!\left(
      f_{\mathrm{seg}}(I_t);\,
      \mathbf{K}, h, \alpha, \beta
    \right)
  \right),
\label{eq:gob_full}
\end{equation}
with size $192 \times 192 \times 14$, where each of 
the 14 channels corresponds to a binary occupancy 
mask for a specific semantic category (road surface, 
lane markings, vehicles, pedestrians, sidewalks, 
etc.). This 14-channel representation matches exactly 
the observation format expected by 
$\pi_\theta^{\mathrm{sim}}$, ensuring input 
compatibility without any policy modification. The 
GOB output is therefore:
\begin{equation}
\hat{o}_t^{\mathrm{sim}} = 
\mathcal{G}(o_t^{\mathrm{real}}),
\quad \hat{o}_t^{\mathrm{sim}} \in 
\hat{\mathcal{O}}^{\mathrm{sim}} \subset 
\mathbb{R}^{192 \times 192 \times 14},
\label{eq:gob_output}
\end{equation}
which is fed directly to the policy 
$\pi_\theta^{\mathrm{sim}}(a_t \mid 
\hat{o}_t^{\mathrm{sim}}, s_t, w_t)$ at each 
timestep. Because the same semantic channel layout is 
preserved across domains, the policy can continue to 
operate on the same input tensor format without any 
architectural modification. This compatibility is 
essential for transferring a simulator-trained policy 
to a real monocular perception stack. The full GOB 
pipeline is illustrated in Fig.~\ref{fig:gob}.

\begin{figure*}[t]
    \centering
    \includegraphics[width=0.99\linewidth]
    {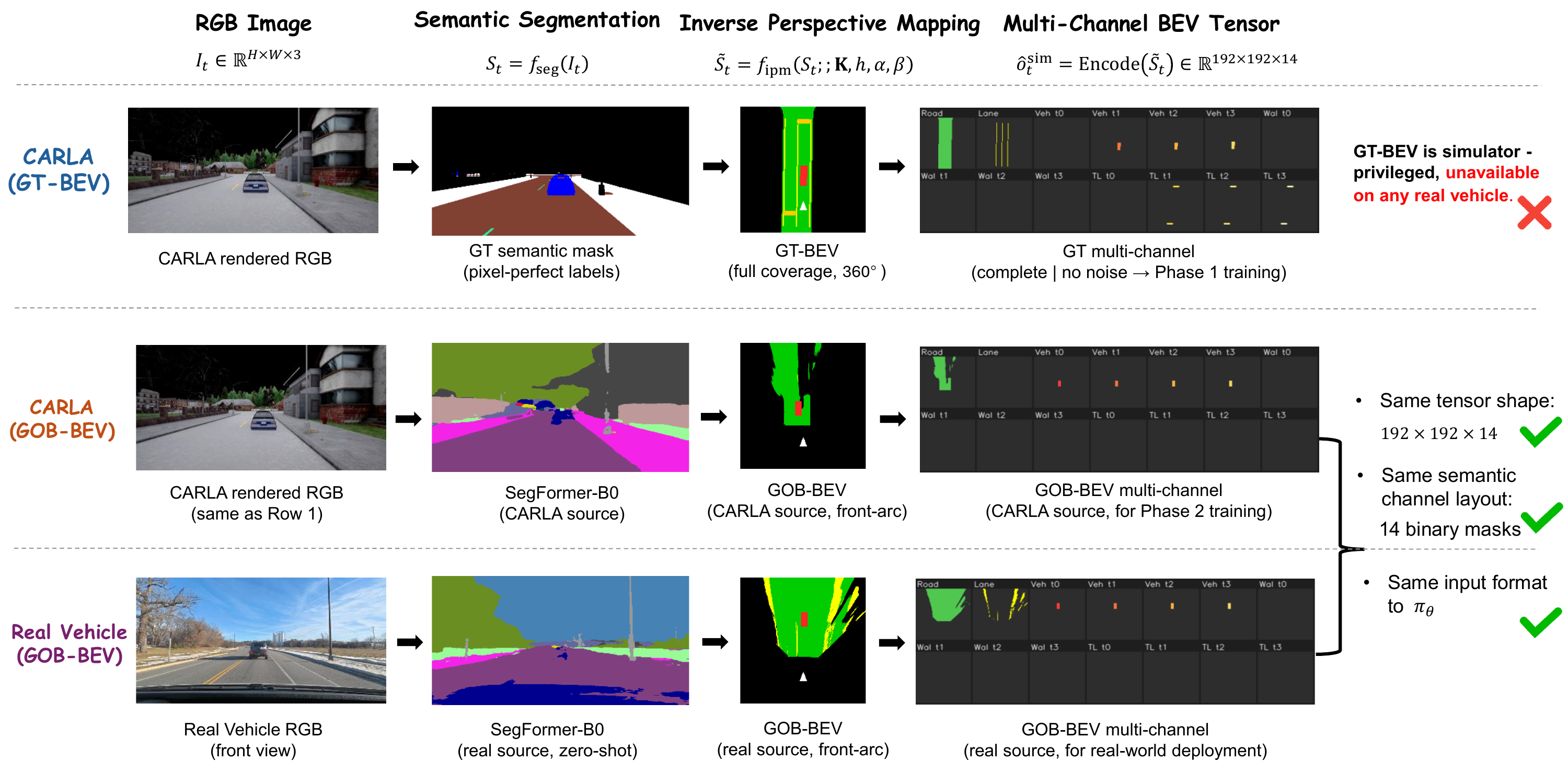}
    \caption{Geometric Observation Bridge: monocular front-view images to simulator-compatible BEV. Phase 1 uses simulator-privileged GT-BEV; Phase 2 and real-world deployment use GOB-generated BEV with the same tensor shape and channel layout.}
    \label{fig:gob}
\end{figure*}

\subsubsection{Observation Transfer Properties}
Under the lightweight deployment setting considered 
in this paper, the real vehicle is equipped only with 
a single front-view monocular camera. As a result, 
the generated $\hat{o}_t^{\mathrm{sim}}$ provides 
reliable information mainly in a front-facing region 
and degrades with distance. In particular, far-field 
road pixels are heavily compressed in the image plane 
and become more sensitive to segmentation noise after 
projection. Nevertheless, this degradation is 
acceptable because near-field structure is the most 
critical information for lane keeping, obstacle 
avoidance, and short-horizon control. More 
importantly, the observation \emph{format} remains 
unchanged: both simulator GT-BEV and IPM-generated 
$\hat{o}_t^{\mathrm{sim}}$ share the same tensor 
shape and semantic channel layout in 
$\hat{\mathcal{O}}^{\mathrm{sim}}$, differing mainly 
in coverage and quality rather than representation 
structure. This structural consistency is why 
observation adaptation can be handled through 
fine-tuning rather than retraining from scratch.

To characterize precisely how much the two BEV 
distributions differ across domains, we introduce the 
following notion of domain invariance for 
$\mathcal{G}$, which connects the perceptual gap to 
the downstream policy performance bound in 
Theorem~\ref{thm:main_informal}.

\begin{definition}[Domain-Invariant BEV 
Representation]
\label{def:bev}
A BEV encoding function 
$\mathcal{G}: \mathcal{I} \rightarrow 
\hat{\mathcal{O}}^{\mathrm{sim}}$ is 
\emph{domain-invariant} with tolerance 
$\epsilon > 0$ if:
\begin{equation}
\mathbb{E}_{\substack{I \sim \mathcal{D}^{\text{sim}}
\\ I' \sim \mathcal{D}^{\text{real}}}}
\!\left[\,
\bigl\| \mathcal{G}(I) - \mathcal{G}(I') \bigr\|_1
\;\Big|\;
\mathrm{scene}(I) = \mathrm{scene}(I')
\right] \leq \epsilon,
\end{equation}
where $\mathrm{scene}(\cdot)$ denotes the semantic 
content of an image and the expectation is conditioned 
on scene-matched pairs. A smaller $\epsilon$ means 
the two domains are more indistinguishable from the 
policy's perspective.
\end{definition}

\begin{remark}
GOB achieves a small $\epsilon$ by operating on 
semantic categories rather than raw pixel values: 
appearance differences due to lighting, texture, and 
rendering style are largely absorbed by the segmentation step 
and therefore do not directly propagate into 
$\hat{o}_t^{\mathrm{sim}}$. In practice, this tolerance 
$\epsilon$ is influenced primarily by the quality of the 
perception module, together with geometric approximation, 
camera calibration error, and projection distortion introduced 
by IPM, especially in far-field regions. Empirically, the 
segmentation component remains small on standard road 
scenes~\citep{xie2021segformer}, which helps keep the overall 
cross-domain discrepancy low. This tolerance $\epsilon$ is 
the key quantity governing the GOB contribution in 
Theorem~\ref{thm:main_informal}: reducing perception and 
projection errors directly tightens the sim-to-real 
performance bound, as formalized in 
Proposition~\ref{prop:gob_bound} (\ref{app:theory}).
\end{remark}

\subsection{Physics-Aware Action Mapping}
\label{sec:pam}

The dynamics gap defined in Section~\ref{sec:problem} 
persists even after $\mathcal{G}$ reduces the 
observation gap: the policy output 
$a_t = \pi_\theta^{\mathrm{sim}}(
\hat{o}_t^{\mathrm{sim}}, s_t, w_t)$ remains a 
simulator-calibrated command that produces different 
path curvatures on a real vehicle, 
$\kappa_{\mathrm{sim}}(a_t) \neq 
\kappa_{\mathrm{real}}(a_t)$, due to differences in 
wheelbase, steering ratio, actuator characteristics, 
and low-level control delay. PAM realizes the mapping 
$\mathcal{M}: \hat{\mathcal{A}} \rightarrow 
\mathcal{A}^{\mathrm{real}}$ by redefining the policy 
output in terms of platform-agnostic physical 
quantities rather than simulator-specific commands.

\subsubsection{Platform-Agnostic Action Space}

The key idea is to let the policy predict 
\emph{driving intent}, while leaving the final 
actuation conversion to a calibrated 
platform-dependent controller. Instead of predicting 
direct steering and throttle/brake, the policy outputs 
two normalized action variables 
$(a_{1,t}, a_{2,t}) \in [-1,1]^2$, which are mapped 
to a platform-agnostic action space:
\begin{equation}
\hat{a}_t = (\kappa_t, v_t^{d}) \in \hat{\mathcal{A}},
\end{equation}
where $\kappa_t$ denotes the desired path curvature 
and $v_t^{d}$ denotes the desired speed. The mapping 
is defined as:
\begin{equation}
\kappa_t = a_{1,t}\,\kappa_{\max}, \qquad
v_t^{d} = \frac{a_{2,t}+1}{2}\,v_{\max},
\end{equation}
where $\kappa_{\max}$ and $v_{\max}$ are predefined 
platform-independent limits. These quantities encode 
geometric and kinematic driving intent more directly 
than raw low-level commands, and thus remain more 
stable across platforms. In this formulation, the 
policy is responsible for deciding \emph{how sharply 
the vehicle should turn} and \emph{how fast it should 
move}, while the platform-specific controller converts 
this intent into executable actuation.

\subsubsection{Low-Level Control Conversion}

To convert the desired curvature into a steering 
command, we use a kinematic bicycle 
model~\citep{rajamani2006vehicle}. Let $L$ denote the 
wheelbase of the target vehicle. The desired front 
wheel steering angle is computed as:
\begin{equation}
\delta_t = \arctan(L \kappa_t),
\end{equation}
and the corresponding normalized steering command is:
\begin{equation}
u_t^{\delta} = \frac{\delta_t}{\delta_{\max}},
\end{equation}
where $\delta_{\max}$ denotes the maximum steering 
angle of the target vehicle. The same curvature 
command can thus be interpreted consistently across 
different platforms through lightweight calibration 
rather than policy retraining.

For longitudinal control, the desired speed $v_t^{d}$ 
is tracked by a low-level PID 
controller~\citep{astrom2006advanced}. The speed 
tracking error is:
\begin{equation}
e_t = v_t^{d} - v_t,
\end{equation}
and the corresponding longitudinal command is:
\begin{equation}
u_t^{v} = K_p e_t + K_i \sum_{\tau=0}^{t} 
e_\tau \Delta t + K_d 
\frac{e_t - e_{t-1}}{\Delta t},
\end{equation}
where $K_p$, $K_i$, $K_d$ are the controller gains 
and the integral term is clipped to avoid windup. 
The final executable vehicle command is:
\begin{equation}
a_t^{\mathrm{real}} = \mathcal{M}(\kappa_t, v_t^{d})
  = (u_t^{\delta},\, u_t^{v}).
\label{eq:pam_output}
\end{equation}

\subsubsection{Cross-Platform Calibration and Transfer Properties}

This design decouples the policy from the low-level 
actuation details of a particular vehicle platform. 
The policy no longer needs to implicitly learn a 
specific steering ratio or throttle response from the 
simulator; instead, it predicts motion intent in a 
platform-agnostic form. Deployment-time calibration 
requires adjusting only a small set of physical 
parameters:
\begin{equation}
\Theta = \{L, \delta_{\max}, K_p, K_i, K_d\},
\end{equation}
without modifying $\pi_\theta^{\mathrm{sim}}$ itself. 
In practice, $L$ and $\delta_{\max}$ are obtained 
from vehicle specifications, while the PID gains are 
calibrated through short step-response tests on the
target platform ($\approx$10~min).

The key structural reason this calibration is so 
lightweight is that $\mathcal{M}$ is not a monolithic 
mapping from simulator actions to real actuator 
commands, but a \emph{factored} one: it routes through 
a platform-agnostic physical intermediate space 
$\mathcal{P}$ that is independent of both the 
simulator and the target vehicle. We formalize this 
property below, as it is also what allows the PAM 
tracking error to be bounded and connected to the 
performance guarantee in 
Theorem~\ref{thm:main_informal}.
 
\begin{definition}[Platform-Agnostic Action Interface]
\label{def:pam}
An action mapping $\mathcal{M}: \hat{\mathcal{A}} 
\rightarrow \mathcal{A}^{\mathrm{real}}$ is 
\emph{platform-agnostic} if it factors through a 
physical intermediate space $\mathcal{P} = \mathbb{R}^2$ 
(target curvature $\kappa$ and desired speed $v^d$):
\begin{equation}
\mathcal{M} = \mathcal{M}_{\text{real}} \circ 
\mathcal{M}_{\text{sim}}^{-1},
\end{equation}
where $\mathcal{M}_{\text{sim}}^{-1}: \hat{\mathcal{A}} 
\rightarrow \mathcal{P}$ depends only on the simulator 
vehicle's kinematic parameters and 
$\mathcal{M}_{\text{real}}: \mathcal{P} \rightarrow 
\mathcal{A}^{\mathrm{real}}$ depends only on the real 
vehicle's kinematic parameters and PID gains $\Theta$.
\end{definition}

\begin{remark}[Why platform-agnostic actions improve 
transfer]
\label{rem:action_gap}
The factored structure of Definition~\ref{def:pam} 
has two direct consequences. First, it enables 
lightweight cross-platform deployment: to transfer to 
a new vehicle, only $\mathcal{M}_{\text{real}}$ needs 
to be updated by recalibrating $\Theta$, while 
$\mathcal{M}_{\text{sim}}^{-1}$ and the policy 
$\pi_\theta^{\mathrm{sim}}$ remain entirely unchanged. 
Second, it makes the dynamics gap quantifiable: 
because $\mathcal{P}$ consists of physical quantities 
($\kappa$, $v^d$) with well-defined tracking dynamics, 
the execution error introduced by imperfect PID 
control can be bounded analytically. Specifically, 
the lateral position error accumulated over a control 
horizon of $T$ steps is bounded by
$v_{\max}^2 T^2 \epsilon_{\mathrm{pid}}/2$,
where $\epsilon_{\mathrm{pid}}$ is the per-step 
curvature tracking error 
(Assumption~\ref{ass:pid}). This bound is the 
quantity that controls the PAM error term 
$C_2\,\epsilon_{\mathrm{pid}}$ in 
Theorem~\ref{thm:main_informal}, and is formalized 
in Proposition~\ref{prop:pam_tracking} 
(\ref{app:theory}).
\end{remark}

\subsection{Two-Phase Progressive Training}
\label{sec:tpt}

In our preliminary implementation, we initially 
attempted to introduce both changes at once, namely 
replacing simulator ground-truth BEV observations 
with IPM-generated BEV observations while 
simultaneously replacing simulator-coupled control 
outputs with physics-aware actions. However, this 
one-stage adaptation strategy was found to be 
unstable. A likely reason is that the policy must 
simultaneously adapt to two different forms of 
distribution shift: degraded perceptual inputs on 
the observation side and a newly defined action 
semantics on the control side. This substantially 
increases optimization difficulty and often leads 
to slow convergence or unstable learning. We 
therefore adopt a TPT strategy that decomposes 
sim-to-real adaptation into two sequential stages, 
as detailed in Algorithm~\ref{alg:tpt}.

\subsubsection{Phase~1: Action-Space Adaptation}

In the first phase, the policy is trained from 
scratch in CARLA for $1 \times 10^6$ steps using 
simulator ground-truth BEV observations 
$\hat{o}_t^{\mathrm{sim}} \in 
\hat{\mathcal{O}}^{\mathrm{sim}}$, while the 
original simulator-coupled action interface is 
replaced by the platform-agnostic action space 
$\hat{\mathcal{A}}$. Formally, the policy predicts:
\begin{equation}
\hat{a}_t = (\kappa_t, v_t^d) \in \hat{\mathcal{A}},
\end{equation}
while receiving clean and spatially complete 
simulator BEV inputs. The VLM-guided reward signal 
Eq.~\eqref{eq:reward_combined} provides semantic 
supervision throughout. Since the observation 
structure remains unchanged, the policy can focus 
entirely on learning how curvature and desired speed 
affect lane keeping, route following, and obstacle 
avoidance. Phase~1 thus isolates the dynamics-side 
transfer problem and resolves it before observation 
adaptation begins.

\subsubsection{Phase~2: Observation-Space Adaptation}

In the second phase, the action space $\hat{\mathcal{A}}$ 
is kept unchanged, but the observation source is replaced by the GOB-generated BEV produced from simulator front-view RGB images. That is, the policy now receives IPM-generated observations 
$\hat{o}_t^{\mathrm{sim}} = 
\mathcal{G}(I_t^{\mathrm{sim}}) \in 
\hat{\mathcal{O}}^{\mathrm{sim}}$ rather than 
privileged ground-truth BEV, for $5 \times 10^5$ 
additional steps. The same VLM-guided reward signal 
Eq.~\eqref{eq:reward_combined} continues to supervise 
training, ensuring that the VLM-guided reward 
structure is maintained under the new observation 
interface.

Let $\pi_{\theta_1}$ denote the policy learned in 
Phase~1. Phase~2 training is initialized from the 
Phase~1 checkpoint:
\begin{equation}
\theta_2^{(0)} = \theta_1,
\end{equation}
and further optimized under the observation 
distribution $\mathcal{O}_2$ induced by $\mathcal{G}$. 
This progressive initialization allows the policy to 
retain already-learned driving behavior and action 
semantics while specializing to the noisier and more 
limited observation distribution that better matches 
real deployment.

\subsubsection{Training Rationale and Deployment Alignment}

By separating the two sources of transfer difficulty, TPT turns sim-to-real adaptation into a curriculum-like process
\citep{bengio2009curriculum,salvato2021crossing,
zhao2020sim}: the IPM-generated BEV distribution $\mathcal{O}_2$ used in Phase 2 is designed to be closer to the real deployment distribution $\mathcal{O}^{\text{real}}$ than the ground-truth BEV distribution $\mathcal{O}_1$ used in Phase 1. We formalize this distributional 
ordering below, as it is precisely what determines 
the TPT residual term $C_3\,\delta/(1-\gamma)^2$ in 
Theorem~\ref{thm:main_informal}.

\begin{definition}[Progressive Observation Curriculum]
\label{def:tpt}
A two-phase training schedule 
$(T_1, T_2, \mathcal{O}_1, \mathcal{O}_2)$ is a 
\emph{progressive observation curriculum} if 
$T_1 + T_2 = T$ (total training steps) and the 
observation distributions satisfy:
\begin{equation}
d_{\mathrm{TV}}\!\left(\mathcal{O}_1,\; 
\mathcal{O}^{\text{real}}\right)
\;\geq\;
d_{\mathrm{TV}}\!\left(\mathcal{O}_2,\; 
\mathcal{O}^{\text{real}}\right),
\end{equation}
where $d_{\mathrm{TV}}$ denotes total variation 
distance, $\mathcal{O}_1$ is the ground-truth BEV 
distribution (Phase~1), and $\mathcal{O}_2$ is the 
IPM-generated BEV distribution (Phase~2). In other 
words, each phase exposes the policy to observations 
that are no further from real deployment than the 
previous phase.
\end{definition}

\begin{remark}[Motivation for progressive training]
\label{rem:progressive_training}
TPT satisfies Definition~\ref{def:tpt} by 
construction: ground-truth BEV ($\mathcal{O}_1$) is 
further from the real distribution than IPM-generated 
BEV ($\mathcal{O}_2$), because IPM operates on real 
camera geometry and already introduces the same class 
of projection artifacts and limited field of view 
present at deployment, whereas ground-truth BEV does 
not. This means Phase~2 reduces the residual 
distribution gap 
$d_{\mathrm{TV}}(\mathcal{O}_2, \mathcal{O}^{\text{real}})$ 
relative to what it would be after Phase~1 alone. 
The significance of this reduction is quantified by 
Theorem~\ref{thm:tpt_bound} 
(\ref{app:theory}):
\begin{equation*}
J(\pi_\theta^{\mathrm{sim}}, \mathcal{O}^{\text{real}}) 
\geq J(\pi_\theta^{\mathrm{sim}}, \mathcal{O}_2) -
\frac{2R_{\max}\,
d_{\mathrm{TV}}(\mathcal{O}_2, 
\mathcal{O}^{\text{real}})}{(1-\gamma)^2},
\end{equation*}
showing that the performance gap scales directly with 
$d_{\mathrm{TV}}(\mathcal{O}_2, 
\mathcal{O}^{\text{real}})$. This is precisely the 
TPT residual $C_3\,\delta/(1-\gamma)^2$ in 
Theorem~\ref{thm:main_informal}: longer Phase~2 
training further reduces 
$d_{\mathrm{TV}}(\mathcal{O}_2, 
\mathcal{O}^{\text{real}})$, directly tightening 
this term. A single-stage schedule that skips Phase~1 
would instead need 
$d_{\mathrm{TV}}(\mathcal{O}_1, 
\mathcal{O}^{\text{real}})$ in the bound, a 
substantially larger quantity, leaving the TPT 
residual much worse.
\end{remark}


\subsection{Real-Time Deployment Pipeline}
\label{sec:rdp}

After training, the final policy $\pi_{\theta_2}$ is 
integrated into a RDP for closed-loop vehicle 
operation. The purpose of this pipeline is to connect 
perception, route input, vehicle-state feedback, 
policy inference, action conversion, and safety 
monitoring into a complete execution stack that 
implements the composed policy 
Eq.~\eqref{eq:composed_policy} on a physical vehicle.

\subsubsection{Pipeline and Timing}
The deployment loop follows a 
perception--inference--control structure:
\begin{equation}
o_t^{\mathrm{real}}
\;\xrightarrow{\;\mathcal{G}\;}\;
\hat{o}_t^{\mathrm{sim}}
\;\xrightarrow{\;\pi_{\theta_2}\;}\;
\hat{a}_t
\;\xrightarrow{\;\mathcal{M}\;}\;
a_t^{\mathrm{real}},
\label{eq:rdp_pipeline}
\end{equation}
where the final command is further checked 
by the safety layer before execution, as illustrated in 
Fig.~\ref{fig:overview}(d).

The deployment loop consists of image acquisition, GOB processing, policy inference, PAM conversion, and command transmission. In steady-state pipelined execution, GOB processing dominates the critical path, and the measured average onboard compute latency is about 26.8 ms per cycle, well within the 50 ms budget of the 20 Hz control loop. Policy inference takes approximately 2 ms, while PAM conversion and PID computation take less than 1 ms.

\subsubsection{Waypoint Generation}

The policy requires route information and vehicle 
states consistent with the simulator training 
interface. To provide route input, we use a waypoint 
provider that supplies a sequence of future path 
points to the policy. In our primary implementation, 
the route is obtained from a pre-recorded GPS 
trajectory: the current vehicle position is matched 
to the nearest point on the route, and a set of 
future waypoints is extracted and transformed into 
the vehicle-relative coordinate frame. This design 
makes the real-world waypoint input compatible with 
the route representation used during simulator 
training. When GPS-based routing is unavailable, a 
vision-based fallback can also be constructed by 
extracting a lane centerline from the road surface 
and lane marking channels of 
$\hat{o}_t^{\mathrm{sim}}$. A vehicle-state 
interface provides real-time low-level feedback from 
the platform, including current speed, steering 
status, and other controller-relevant signals. In our 
implementation, these quantities are read from the 
vehicle through the CAN bus and converted into the 
normalized format expected by the policy and 
controller. Camera frames, waypoint updates, and 
vehicle-state messages are time-aligned before policy 
inference to ensure stable closed-loop execution.

\subsubsection{Safety Layer}
To improve operational safety, we place a safety 
layer on top of the learned policy and low-level 
controller. It enforces hard motion constraints 
including a maximum speed of 15~km/h during initial 
testing, consistent with campus low-speed autonomous 
vehicle testing 
protocols, a lateral
deviation limit of 0.8~m from the detected lane 
center, and a steering-rate limit 
$|\delta_t - \delta_{t-1}| \leq \Delta\delta_{\max}$ 
to prevent abrupt steering commands. Emergency 
braking is triggered when any of the following 
conditions is met: invalid policy outputs (NaN/Inf); 
camera frame timeout or CAN bus disconnection; 
emergency-stop button activation; vehicle exit from 
a predefined geofenced region; or safety driver 
intervention detected via steering-wheel torque 
exceeding a threshold. All safety checks execute at 
higher priority than the policy output and cannot be 
overridden. The system immediately releases control 
authority upon safety-driver takeover.

\subsubsection{Real-Time Execution and Modularity}

The observation bridge, route and state interfaces, 
policy, action mapping, and safety layer are loosely 
coupled and can be upgraded independently. For 
example, the IPM-based BEV generation can be replaced 
by a stronger camera-to-BEV model, or the PID 
controller replaced by a model predictive controller, 
without changing $\pi_{\theta_2}$ itself. The same 
modularity makes calibration lightweight: only the 
small parameter set $\Theta$ specified in 
Section~\ref{sec:pam} and the camera parameters 
$(\mathbf{K}, h, \alpha, \beta)$ must be specified 
before deployment.

Overall, RDP closes the gap between simulator 
training and physical execution by turning the 
composed policy:
\begin{equation}
\pi^{\mathrm{real}}(a_t^{\mathrm{real}} \mid 
o_t^{\mathrm{real}})
=
\mathcal{M}
\!\left(
\pi_{\theta_2}
\!\left(
\hat{o}_t^{\mathrm{sim}},\, s_t,\, w_t
\right)
\right)
\end{equation}
into a complete real-world driving stack that 
requires no VLM inference at test time, enabling 
safe and responsive closed-loop operation on a 
full-scale vehicle.

\subsection{Theoretical Guarantee}
\label{sec:theory_informal}

We conclude the framework section with a simplified statement of the main
theoretical guarantee, where the explicit dependence of each error coefficient
on $L_r$, $L_\pi$, $R_{\max}$, $v_{\max}$, $T$, and $\gamma$ is absorbed
into named constants for readability. The full expressions for these constants
and complete proofs are given in \ref{app:theory}.
 
\begin{theorem}[Zero-Shot Transfer Guarantee]
\label{thm:main_informal}
Under bounded segmentation error $\epsilon_{\mathrm{seg}}$, bounded PID
tracking error $\epsilon_{\mathrm{pid}}$, and bounded observation
distribution gap $d_{\mathrm{TV}}(\mathcal{O}_2, \mathcal{O}^{\text{real}})
\leq \delta$, the expected cumulative reward of the Sim2Real-AD policy
$\pi^{\text{real}}$ on the real vehicle satisfies:
\begin{equation}
\begin{aligned}
\mathbb{E}\!\left[\sum_{t=0}^{T} \gamma^t r_t^{\text{real}}\right]
\;\geq\;{}&
\mathbb{E}\!\left[\sum_{t=0}^{T} \gamma^t r_t^{\text{sim}}\right]
- \underbrace{C_1\, \epsilon_{\mathrm{seg}}}_{\text{GOB error}} \\
&- \underbrace{C_2\, \epsilon_{\mathrm{pid}}}_{\text{PAM error}}
- \underbrace{\dfrac{C_3\, \delta}{(1-\gamma)^2}}_{\text{TPT residual}},
\end{aligned}
\label{eq:main_bound}
\end{equation}
where $C_1, C_2, C_3 > 0$ are constants depending on the policy's Lipschitz
constant $L_\pi$, the reward Lipschitz constant $L_r$, the reward bound
$R_{\max}$, the maximum speed $v_{\max}$, the horizon $T$, and the discount
factor $\gamma$ (see \ref{app:main} for explicit expressions),
and $r_t^{\text{sim}}$ is the VLM-guided reward Eq.~\eqref{eq:reward_combined}
from the RL objective Eq.~\eqref{eq:rl_objective}.
\end{theorem}
 
\begin{remark}
Theorem~\ref{thm:main_informal} identifies \emph{three independent sources}
of sim-to-real performance degradation and shows that each is separately
controllable: the GOB error decreases with better segmentation, the PAM error
decreases with tighter PID calibration, and the TPT residual decreases with
longer Phase~2 training or higher-fidelity IPM. When all three error terms
vanish, real-vehicle performance converges to simulation performance, as
formalized in Corollary~\ref{cor:convergence} (\ref{app:theory}).
To the best of our knowledge, this is among the first transfer-error decompositions for zero-shot sim-to-real deployment of a VLM-guided RL policy in autonomous driving.
\end{remark}


\section{Experiments}
\label{sec:experiments}

The experiments are structured to address the 
following research questions:
\textbf{RQ1:} Can the proposed Sim2Real-AD framework
transfer policies trained under different reward
paradigms through a single pipeline, keeping them
functional after sim-to-real transfer (reward-agnostic
generality)?
\textbf{RQ2:} What is the individual contribution of 
each module (GOB, PAM, TPT) to overall transfer 
performance?
\textbf{RQ3:} How much observation fidelity does GOB 
preserve compared to ground-truth BEV?
\textbf{RQ4:} Can the framework achieve zero-shot
deployment on a real vehicle?

\subsection{Simulation Experiments}

\subsubsection{Experimental Setup}
\label{sec:sim_setup}

\textit{1) Simulation Environment.}
All simulation experiments are conducted in CARLA 
0.9.13~\citep{dosovitskiy2017carla} with synchronous 
mode at 20~FPS. Consistent with~\citep{huang2026drivevlmrl}, 
models are trained exclusively on Town~2, a compact 
European-style urban layout with residential districts, 
commercial zones, single-lane roads, and signalized 
intersections that provides diverse driving conditions 
including straight roads, curved segments, T-junctions, 
and varying road geometries. For the cross-map limitation
analysis (Section~\ref{sec:limitations}), we additionally evaluate
on Towns~1, 3, 4, and~5, which present progressively harder
distribution shifts in road topology and traffic patterns.

\textit{2) Traffic Configuration.}
To evaluate robustness under realistic urban conditions 
with heterogeneous road users, we construct a complex 
traffic environment following~\citep{huang2026drivevlmrl}. 
Specifically, the simulation includes: 20 vehicles 
generating natural traffic flow interactions; 
20 pedestrians with randomized walking speeds 
(0.8--1.5~m/s) moving around sidewalks and crosswalks; 
20 motorcycles with short following distances (2.0~m) 
and $\pm$30\% speed variance, frequently producing 
cut-in behaviors; and 20 bicycles traveling at 
approximately 80\% below the speed limit, requiring 
safe and patient overtaking maneuvers. 

\textit{3) Navigation Routes.}
We employ dynamic route assignment during both 
training and evaluation. At each episode reset, two 
distinct spawn points are randomly selected from the 
101 predefined locations in Town~2 and the shortest 
path is computed via the A* algorithm. Episodes 
continue until the cumulative driving distance reaches 
3,000~m, providing comprehensive coverage of diverse 
navigation scenarios within a single episode. For 
evaluation, we use 10 predefined routes not 
encountered during training. These evaluation routes are defined within Town 2 using held-out spawn-point pairs.

\textit{4) Episode Termination.}
Each episode terminates upon: (i)~collision with 
static infrastructure, vehicles, pedestrians, 
cyclists, or motorcyclists; (ii)~the ego vehicle 
remaining stationary (speed $<$ 1~km/h) for more 
than 90 consecutive seconds, indicating a stuck 
condition; or (iii)~lateral deviation from the lane 
center exceeding 3~m.

\subsubsection{Training Configuration}
\label{sec:train_config}

All algorithms use Soft Actor-Critic 
(SAC)~\citep{haarnoja2018soft} with automatic 
entropy tuning as the backbone RL optimizer. 
Table~\ref{tab:hyperparams} summarizes the key 
hyperparameters for the two-phase progressive 
training.

\begin{table}[!t]
\centering
\caption{Training hyperparameters for two-phase
progressive training (TPT). Phase~2 warm-starts
from the Phase~1 checkpoint with a reduced
learning rate. The Phase~1 step count is the nominal schedule; VLM-RL and DriveVLM-RL train Phase~1 to $\approx$1.1$\times$10$^6$ steps (see the per-algorithm training-curve figures).}
\label{tab:hyperparams}
\renewcommand{\arraystretch}{1.25}
\setlength{\tabcolsep}{4pt}
\small
\begin{tabular}{lcc}
\toprule
Parameter & Phase~1 & Phase~2 \\
\midrule
Observation source     & GT-BEV          
  & GOB-BEV (GOB)  \\
Action space           & Physics (PAM)   
  & Physics (PAM)  \\
Training steps         & $1 \times 10^6$ 
  & $5 \times 10^5$ \\
Initialization         & Random          
  & Phase~1 checkpoint \\
Learning rate          
  & $1{\times}10^{-4}$ to $5{\times}10^{-7}$
  & $5{\times}10^{-5}$ to $5{\times}10^{-7}$ \\
Replay buffer size     & $10^5$  & $10^5$ \\
Batch size             & 256     & 256 \\
Discount $\gamma$      & 0.98    & 0.98 \\
Soft update $\tau$     & 0.02    & 0.02 \\
Train frequency        & 64 steps & 64 steps \\
Gradient steps         & 64      & 64 \\
BEV feature extractor  
  & \multicolumn{2}{c}{CustomCNN, 256-d} \\
\bottomrule
\end{tabular}

\end{table}

\textit{1) Observation and Action Spaces.}
The policy receives a dictionary observation 
comprising: (i)~a 14-channel semantic BEV mask at 
$192 \times 192$ resolution encoding 
road surface, lane markings, vehicles, pedestrians, 
traffic lights, and other semantic categories; 
(ii)~the next 15 waypoints along the planned route 
expressed as $(x,y)$ ego-centric coordinates at 
2~m intervals; and (iii)~vehicle state measurements 
(speed in km/h, steering angle, throttle command). 
In Phase~2 of TPT, the BEV is produced by the GOB 
module (IPM pipeline from a front-facing monocular 
camera) rather than the simulator's ground-truth 
semantic renderer.

For Sim2Real-AD, the policy outputs 
$\hat{a} = (a_\kappa, a_v) \in [-1,1]^2$, which 
PAM maps to platform-agnostic quantities:
\begin{equation}
\kappa = a_\kappa \cdot \kappa_{\max}, \qquad
v_{\mathrm{des}} = \tfrac{a_v + 1}{2} \cdot 
v_{\max},
\end{equation}
where $\kappa_{\max} = \tan(\delta_{\max}) / L$ 
with wheelbase $L = 2.875$~m and 
$\delta_{\max} = 70^\circ$ (CARLA default vehicle; 
Ford E-Transit van platform parameters are given in 
Section~\ref{sec:platform}), and 
$v_{\max} = 35$~km/h. A PID controller 
($K_p = 0.5$, $K_i = 0.05$, $K_d = 0.1$) converts 
desired speed to throttle/brake commands.

\textit{2) VLM Configuration.}
For algorithms employing CLIP-based reward shaping 
(DriveVLM-RL and VLM-RL), we use OpenCLIP's 
ViT-bigG-14~\citep{ilharco2021openclip} pretrained 
on LAION-2B, with reward blending coefficient 
$\alpha = 0.5$ and CLIP batch size of 64 frames. 
Since the CLIP reward is computed from first-person camera images rather than BEV observations, the GOB change in Phase 2 does not directly alter the reward computation interface, helping keep the training objective comparable across both phases. For DriveVLM-RL, 
a Qwen3-VL-4B LVLM with attention-gated YOLOv8s 
triggering provides dynamic semantic 
descriptions~\citep{huang2026drivevlmrl}.

\subsubsection{Evaluation Metrics}
\label{sec:metrics}

We report the following metrics, evaluated over 
10 predefined routes per town following 
\citep{huang2026drivevlmrl,huang2025vlm}:

\begin{itemize}
\item \textbf{Average Speed (AS, km/h$\uparrow$)}: 
Mean ego-vehicle speed during the episode.

\item \textbf{Route Completion (RC$\uparrow$)}: 
Fraction of the planned route completed before 
termination.

\item \textbf{Total Distance (TD, m$\uparrow$)}: 
Cumulative distance traveled per episode.

\item \textbf{Collision Speed (CS, 
km/h$\downarrow$)}: Mean speed at impact; reflects 
collision severity rather than frequency.

\item \textbf{Success Rate (SR$\uparrow$)}: 
Fraction of routes completed without collision 
or intervention.

\item \textbf{Average Collision (AC$\downarrow$)}: 
Mean collision count per evaluation route.

\item \textbf{Performance Retention (PR, 
\%$\uparrow$)}: 
\begin{equation}
\mathrm{PR}_{m} = \frac{M_{\mathrm{sim2real}}}
{M_{\mathrm{Original}}} \times 100\%,
\end{equation}
where $M_{\mathrm{Original}}$ is the corresponding Original-setup metric under GT-BEV and direct-action evaluation, and $M_{\mathrm{sim2real}}$ is the metric after applying the Sim2Real-AD transfer pipeline. PR measures how much of the Original reference performance is retained after transfer.
\end{itemize}

\subsubsection{Baseline Algorithms}
\label{sec:baselines}

Because its four modules (GOB, PAM, TPT, RDP)
operate entirely on the policy's input and output
interfaces, Sim2Real-AD is reward-agnostic: it can
wrap any RL-based driving policy, whether the reward
is hand-crafted, LLM-generated, or VLM-guided, and
transfer it without modification. To validate this,
we apply the identical pipeline to three algorithms
spanning distinct reward paradigms, each the
top-performing representative of its paradigm in the
13-method comparison of~\citep{huang2026drivevlmrl}:

\begin{itemize}
\item \textbf{ChatScene-SAC}~\citep{zhang2024chatscene}:
SAC with an expert-designed smoothness-focused
reward. Purely static and analytically computed
without any VLM component. We use the SAC variant
(rather than the PPO version reported
in~\citep{huang2026drivevlmrl}) so that all three
methods share an identical SAC backbone and differ
only in the reward paradigm.

\item \textbf{VLM-RL}~\citep{huang2025vlm}: SAC 
with static shaped reward + CLIP reward using
Contrasting Language Goal (CLG). Fixed prompts 
contrast ``clear road'' vs.\ ``collision'' using 
single-frame BEV images.

\item \textbf{DriveVLM-RL}~\citep{huang2026drivevlmrl} 
(our primary backbone, top-performing overall): 
SAC with a dual-pathway reward: a static pathway 
(CLIP-based CLG on BEV images) and a dynamic 
pathway (attention-gated LVLM reasoning). The 
dynamic pathway captures scene-level semantic risks 
through multi-frame visual understanding, achieving 
the highest SR and lowest CS among all compared 
methods.
\end{itemize}

All three share the same SAC backbone, network
architecture, observation space, and environment;
only the reward differs. We deliberately fix the
optimizer to SAC rather than mixing in, e.g., PPO:
since PAM already redefines the action space in
physical terms (curvature and desired speed),
changing the RL algorithm would alter both the reward
paradigm and the action-space optimization at once,
confounding the effect we aim to isolate. Fixing SAC
thus isolates the effect of reward design on
transferability. Accordingly, ``across algorithms''
here denotes these three VLM-guided RL methods, which
differ in reward paradigm while sharing a common SAC
backbone.

\begin{figure*}[!t]
    \centering
    \includegraphics[width=\linewidth]{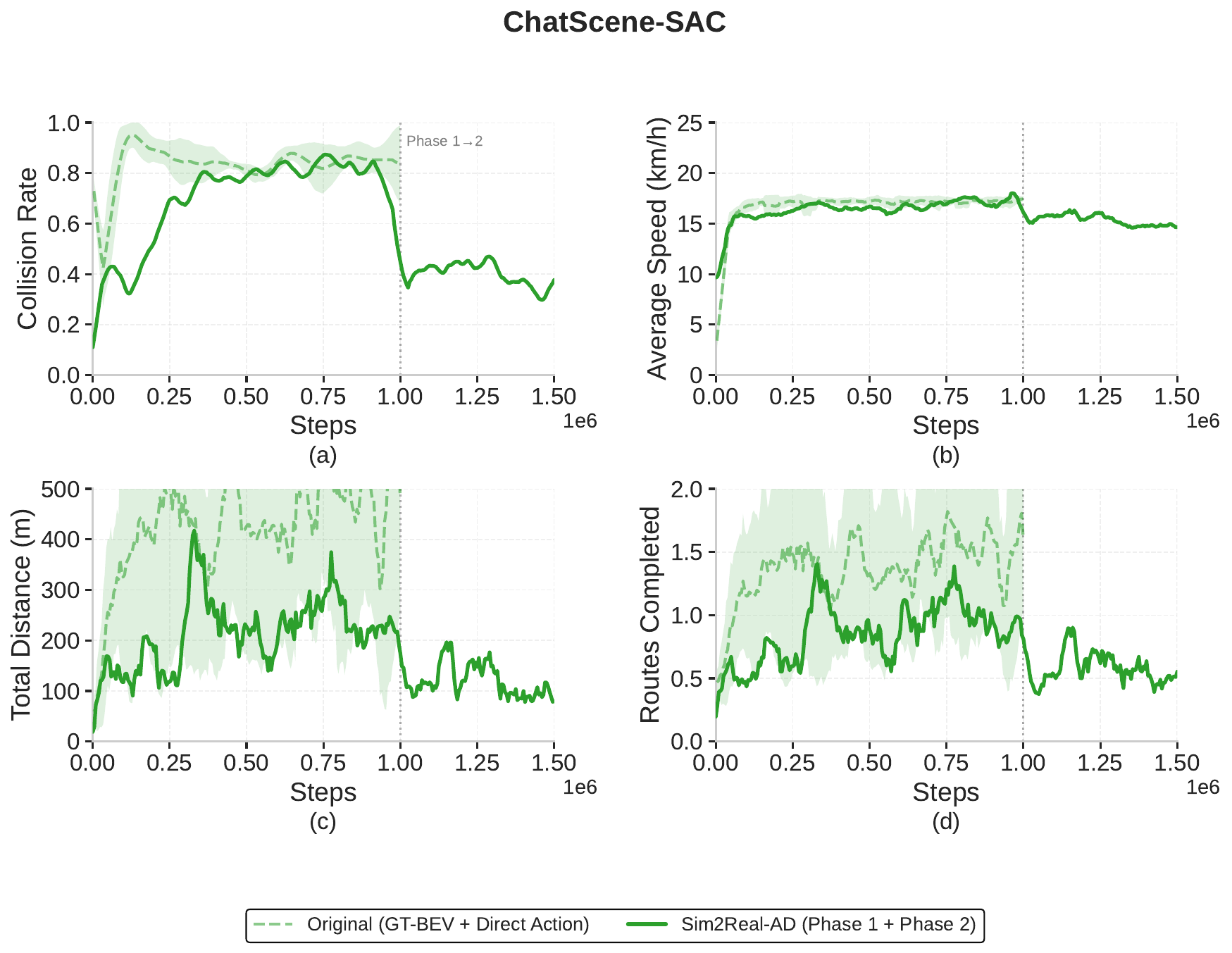}
    \caption{ChatScene-SAC training curves: Original 
    (dashed, GT-BEV + Direct Action, $\pm 1\sigma$ 
    over 3 seeds) vs.\ Sim2Real-AD (solid, Phase~1: 
    GT-BEV + PAM, Phase~2: GOB-BEV + PAM). Vertical 
    dotted line marks Phase~1-to-2 at 
    $1{\times}10^6$ steps. (a)~Collision rate. 
    (b)~Average speed. (c)~Total distance. 
    (d)~Routes completed. }
    \label{fig:train_chatscene}
\end{figure*}

\begin{figure*}[!t]
    \centering
    \includegraphics[width=\linewidth]{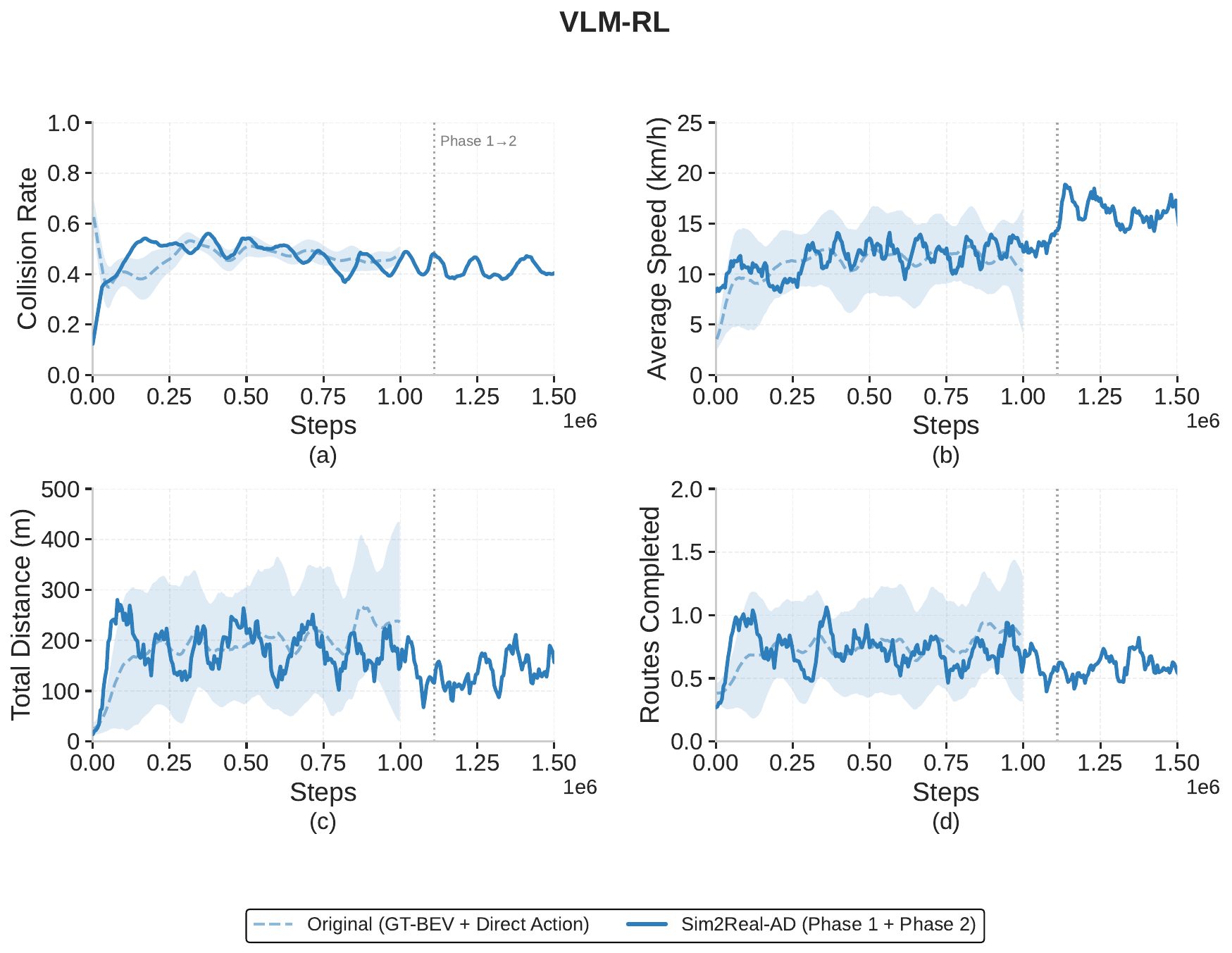}
    \caption{VLM-RL training curves: Original
    (dashed, GT-BEV + Direct Action, $\pm 1\sigma$
    over 3 seeds) vs.\ Sim2Real-AD (solid, Phase~1:
    GT-BEV + PAM, Phase~2: GOB-BEV + PAM). Vertical
    dotted line marks Phase~1-to-2 at
    $\approx\!1.1{\times}10^6$ steps. (a)~Collision rate.
    (b)~Average speed. (c)~Total distance. 
    (d)~Routes completed.}
    \label{fig:train_vlmrl}
\end{figure*}

\begin{figure*}[!t]
    \centering
    \includegraphics[width=\linewidth]{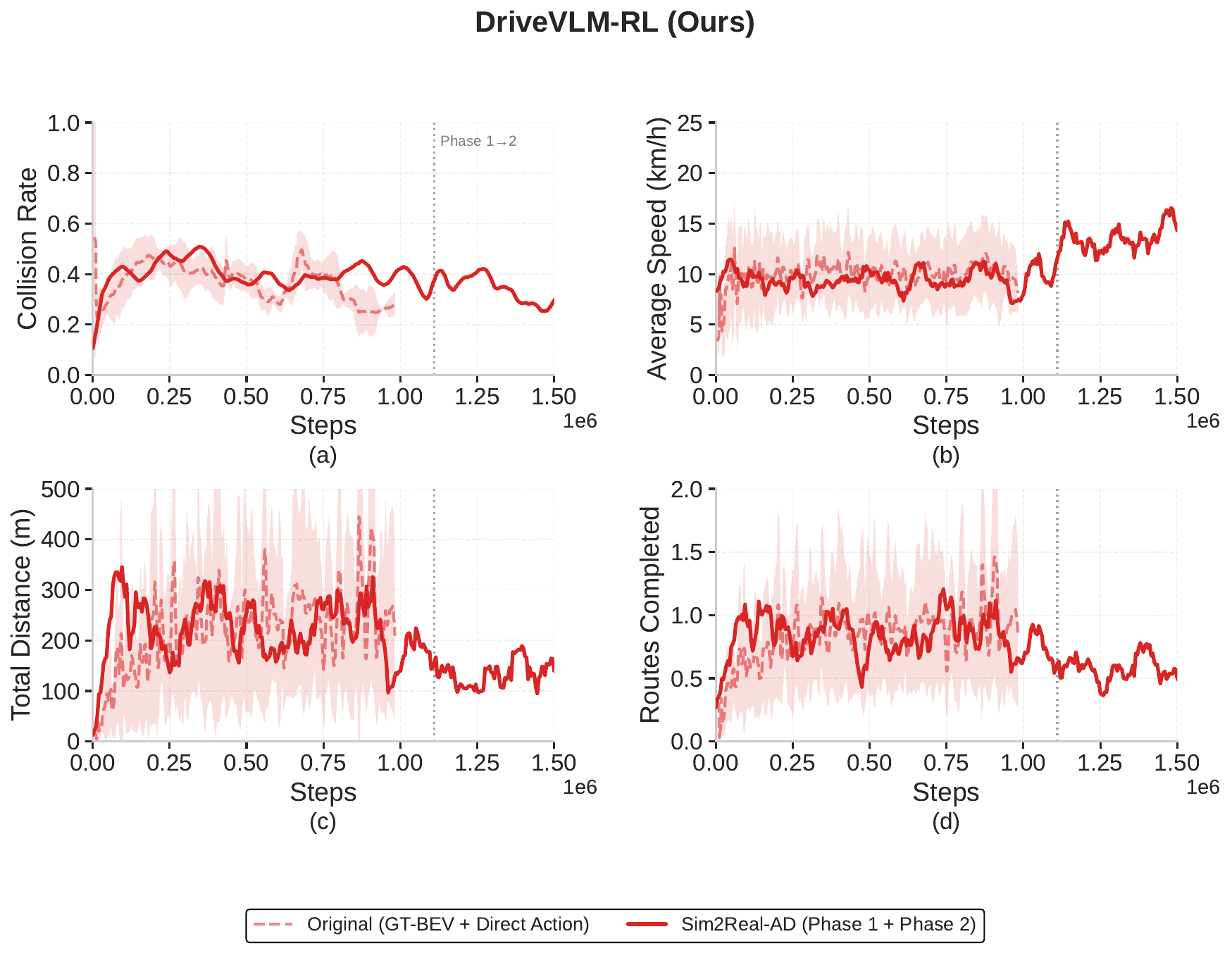}
    \caption{DriveVLM-RL training curves: Original
    (dashed, GT-BEV + Direct Action, $\pm 1\sigma$
    over 3 seeds) vs.\ Sim2Real-AD (solid, Phase~1:
    GT-BEV + PAM, Phase~2: GOB-BEV + PAM). Vertical
    dotted line marks Phase~1-to-2 at
    $\approx\!1.1{\times}10^6$ steps. (a)~Collision rate.
    (b)~Average speed. (c)~Total distance. 
    (d)~Routes completed.}
    \label{fig:train_drivevlm}
\end{figure*}

\subsubsection{Main Results: Sim-to-Real Transfer across Algorithms (RQ1)}
\label{sec:main_result}

We evaluate all three algorithms in two phases: a \emph{training phase} that examines
convergence stability under the TPT curriculum, and a \emph{testing phase} that reports
final policy performance across the three progressive transfer stages.

\textit{1) Training Performance Analysis.}

Figs.~\ref{fig:train_chatscene}--\ref{fig:train_drivevlm} present per-algorithm
training curves comparing the \emph{original} setup (dashed lines: GT-BEV +
Direct Action) against
\emph{Sim2Real-AD} (solid lines: Phase~1 GT-BEV + PAM to Phase~2 GOB-BEV + PAM).
In each figure, the vertical dotted line marks the Phase~1-to-2 transition, i.e.,
the step at which Phase~2 resumes from the Phase~1 checkpoint: this occurs at
$1{\times}10^6$ steps for ChatScene-SAC and at approximately $1.1{\times}10^6$ steps for
VLM-RL and DriveVLM-RL. Examining each algorithm in turn reveals how the
sim-to-real modules reshape training dynamics.

\textbf{ChatScene-SAC} (Fig.~\ref{fig:train_chatscene}) shows the most striking
contrast. Its original training (dashed) reaches the highest collision rate of the
three (0.85--0.95) with high speed ($\sim$17~km/h) and large total distance
(300--500~m). PAM alone does not tame this: throughout Phase~1 the collision rate
stays high and variable (roughly 0.45--0.85), only modestly below the original even
as distance and route completion drop, because the PID controller cannot reproduce
ChatScene-SAC's abrupt, aggressive commands. The collision rate falls sharply only at
the Phase~1-to-2 transition (to about 0.31--0.48), as the degraded GOB-BEV makes the
policy drive far more conservatively with much shorter trajectories. Its apparent
safety gain under full Sim2Real-AD is thus largely a by-product of shorter, more
cautious episodes rather than genuinely improved risk avoidance. Importantly, no
performance collapse occurs at the transition, confirming that TPT stabilizes the
observation-modality switch.

\textbf{VLM-RL} (Fig.~\ref{fig:train_vlmrl}) shows the \emph{closest tracking}
between its original and Sim2Real-AD curves: collision rates nearly overlap
(0.40--0.55) throughout Phase~1, indicating that its single-frame CLIP reward
(``clear road'' vs.\ ``collision'' on BEV images) yields behavior largely
invariant to the action-space change. In Phase~2, speed rises to 15--18~km/h,
exceeding the original; this acceleration likely reflects the noisier IPM
observation reducing caution. Route completion degrades mildly in Phase~2,
consistent with the single-frame CLIP reward being more sensitive to observation
shift than multi-frame VLM rewards.

\textbf{DriveVLM-RL} (Fig.~\ref{fig:train_drivevlm}). Under Sim2Real-AD, its
collision rate rises modestly in Phase~1 (to about 0.30--0.45, slightly above its
original level) as PAM reduces reactivity, then settles to 0.26--0.42 in Phase~2,
remaining the \emph{lowest} among the three throughout training. Average speed is
comparable to the original in Phase~1 (8--12~km/h) and rises to 11--16~km/h after
the transition, suggesting the GOB-BEV observation encourages more confident
acceleration once adapted. The distance and route-completion gap is largest in
Phase~2 while Phase~1 tracks closely, indicating that its semantic safety reasoning
transfers robustly even as the noisier observation shortens episodes.

\textbf{Cross-algorithm summary.} Three patterns hold across all algorithms:
(i)~no performance collapse occurs at the Phase~1-to-2 boundary, validating TPT's
progressive curriculum; (ii)~PAM (isolated in Phase~1) reduces distance and route
completion, but the Phase~1 collision rate remains governed by the reward paradigm,
increasing slightly for the already-safe DriveVLM-RL, staying comparable for VLM-RL,
and remaining high for ChatScene-SAC, whose sharp collision drop emerges only in
Phase~2 via shorter, more conservative trajectories; (iii)~Phase~2 speed rises for
VLM-RL and DriveVLM-RL but falls for ChatScene-SAC, as the noisier observation
encourages more aggressive throttle in the two lower-speed, semantically rewarded
policies.

Table~\ref{tab:train_perf} consolidates the final training-phase values behind
Figs.~\ref{fig:train_chatscene}--\ref{fig:train_drivevlm} across the three training
setups: the original policy (GT-BEV + Direct Action, 3 seeds), the Phase~1 policy
(GT-BEV + PAM), and the full Phase~2 policy (GOB-BEV + PAM + TPT, single physics seed). It
quantifies the consistent pattern visible in the curves: progressively introducing PAM
and then the degraded GOB observation trades travel distance and route completion for a
lower training-time collision rate, with all three methods ending at a similar low
collision rate ($0.30$--$0.40$). DriveVLM-RL additionally retains the longest mean
interval between collisions (ICT) at every setup, consistent with its safety-oriented
reward.

\begin{table*}[!t]
\centering
\caption{Final training-phase performance (mean over the last $5{\times}10^4$ steps
of each setup). \emph{Original} (GT-BEV + Direct Action, 3 training seeds),
\emph{Phase~1} (GT-BEV + PAM), and \emph{Phase~2} (full GOB-BEV + PAM + TPT,
single seed, no std); the same setups appear on the test routes in
Table~\ref{tab:transfer_stages}. AS: average speed (km/h); RC: routes completed;
TD: total distance (m); CR: collision rate; ICT: mean interval between collisions
(steps); DCF: collisions per km. Metrics follow~\citep{huang2026drivevlmrl}.}
\label{tab:train_perf}
\renewcommand{\arraystretch}{1.25}
\setlength{\tabcolsep}{5pt}
\small
\begin{tabular*}{\textwidth}{@{\extracolsep{\fill}}llcccccc@{}}
\toprule
Algorithm & Setup & AS$\uparrow$ & RC$\uparrow$ & TD$\uparrow$ & CR$\downarrow$ & ICT$\uparrow$ & DCF$\downarrow$ \\
\midrule
\multirow{3}{*}{ChatScene-SAC}
 & Original                 & 17.06 & 1.10 & 319.9 & 0.85 & 859 & 11.3 \\
 & Sim2Real-AD (Phase~1)    & 17.46 & 0.71 & 150.6 & 0.65 & 564 & 16.3 \\
 & Sim2Real-AD (Phase~2) & 14.69 & 0.42 & 72.1 & 0.33 & 643 & 25.7 \\
\midrule
\multirow{3}{*}{VLM-RL}
 & Original                 & 10.95 & 0.88 & 236.5 & 0.46 & 1640 & 8.9 \\
 & Sim2Real-AD (Phase~1)    & 11.64 & 0.61 & 121.8 & 0.40 & 1014 & 20.6 \\
 & Sim2Real-AD (Phase~2) & 18.21 & 0.49 & 101.2 & 0.40 & 743 & 17.6 \\
\midrule
\multirow{3}{*}{DriveVLM-RL}
 & Original                 & 9.82 & 0.91 & 235.5 & 0.27 & 4086 & 9.7 \\
 & Sim2Real-AD (Phase~1)    & 9.19 & 0.59 & 131.7 & 0.31 & 2342 & 13.7 \\
 & Sim2Real-AD (Phase~2) & 11.90 & 0.46 & 86.1 & 0.30 & 1130 & 22.4 \\
\bottomrule
\end{tabular*}
\end{table*}

\textit{2) Performance Evaluation in Testing.}
\label{sec:transfer_stages}

Table~\ref{tab:transfer_stages} consolidates all
three evaluation stages into a single unified view:
Original (pre-transfer GT-BEV baseline), Phase~1
(PAM action space only), and Phase~2 (full
Sim2Real-AD with GOB + PAM + TPT). For Original, the
VLM-RL and DriveVLM-RL upper bounds are taken from the
DriveVLM-RL test results~\citep{huang2026drivevlmrl},
while ChatScene-SAC is evaluated under the same
protocol; PR values in Phases~1 and~2 are computed
relative to Original. For Phase~1 and Phase~2, only a
single physics training seed exists, so the reported
$\pm$ std is over 3 evaluation runs of that one policy
with distinct route/traffic seeds.

\begin{table*}[!t]
\centering
\caption{Unified sim-to-real transfer results across three progressive setups (Town 2, 10 routes). Original is mean $\pm$ std over 3 training seeds (VLM-RL and DriveVLM-RL from~\citep{huang2026drivevlmrl}; ChatScene-SAC newly evaluated under the same protocol); Phase~1 and Phase~2 are mean $\pm$ std over 3 evaluation runs of the single transferred policy, so their std reflects evaluation robustness, not training reproducibility. PR(SR) and PR(TD) are relative to Original, a non-deployable GT-BEV upper bound. Best per column within each setup in \textbf{bold}.}
\label{tab:transfer_stages}
\renewcommand{\arraystretch}{1.25}
\setlength{\tabcolsep}{3pt}
\small
\begin{tabular*}{\textwidth}{@{\extracolsep{\fill}}
llcccccccc@{}}
\toprule
Setup & Algorithm
  & AS$\uparrow$
  & RC$\uparrow$
  & TD$\uparrow$
  & CS$\downarrow$
  & SR$\uparrow$
  & AC$\downarrow$
  & PR(SR)
  & PR(TD) \\
\midrule

\multirow{3}{*}{\shortstack[l]{
  \textbf{Original}\\GT-BEV\\Direct Action\\
  (Upper bound)}}
& ChatScene-SAC
  & \textbf{17.54}{\tiny$\pm$0.14}
  & 0.47{\tiny$\pm$0.13}
  & 160.45{\tiny$\pm$9.90}
  & 10.68{\tiny$\pm$1.98}
  & \textbf{0.63}{\tiny$\pm$0.15}
  & 0.37{\tiny$\pm$0.15}
  & --- & --- \\
& VLM-RL
  & 14.38{\tiny$\pm$1.53}
  & 0.51{\tiny$\pm$0.08}
  & 138.08{\tiny$\pm$16.68}
  & 10.09{\tiny$\pm$5.93}
  & 0.40{\tiny$\pm$0.00}
  & \textbf{0.10}{\tiny$\pm$0.10}
  & --- & --- \\
& \cellcolor{green!10}DriveVLM-RL
  & \cellcolor{green!10}14.54{\tiny$\pm$1.81}
  & \cellcolor{green!10}\textbf{0.57}{\tiny$\pm$0.03}
  & \cellcolor{green!10}\textbf{186.59}{\tiny$\pm$14.00}
  & \cellcolor{green!10}\textbf{1.75}{\tiny$\pm$3.02}
  & \cellcolor{green!10}0.57{\tiny$\pm$0.15}
  & \cellcolor{green!10}0.20{\tiny$\pm$0.26}
  & \cellcolor{green!10}---
  & \cellcolor{green!10}--- \\
\midrule

\multirow{3}{*}{\shortstack[l]{
  \textbf{Phase~1}\\GT-BEV\\+PAM only}}
& ChatScene-SAC
  & \textbf{16.85}{\tiny$\pm$0.15} & 0.42{\tiny$\pm$0.04} & 127.05{\tiny$\pm$12.49}
  & 6.95{\tiny$\pm$0.38} & 0.20{\tiny$\pm$0.10} & 0.80{\tiny$\pm$0.10} & 31.7\% & 79.2\% \\
& VLM-RL
  & 15.03{\tiny$\pm$1.00} & \textbf{0.46}{\tiny$\pm$0.03} & \textbf{137.61}{\tiny$\pm$30.69}
  & \textbf{0.14}{\tiny$\pm$0.13} & \textbf{0.37}{\tiny$\pm$0.15} & \textbf{0.23}{\tiny$\pm$0.06} & \textbf{91.7\%} & \textbf{99.7\%} \\
& \cellcolor{green!10}DriveVLM-RL
  & \cellcolor{green!10}14.61{\tiny$\pm$0.94} & \cellcolor{green!10}0.42{\tiny$\pm$0.07} & \cellcolor{green!10}117.58{\tiny$\pm$6.07}
  & \cellcolor{green!10}1.18{\tiny$\pm$2.03} & \cellcolor{green!10}\textbf{0.37}{\tiny$\pm$0.06} & \cellcolor{green!10}0.30{\tiny$\pm$0.20} & \cellcolor{green!10}64.3\% & \cellcolor{green!10}63.0\% \\
\midrule

\multirow{3}{*}{\shortstack[l]{
  \textbf{Phase~2}\\GOB-BEV\\+PAM+TPT\\
  (Full S2R)}}
& ChatScene-SAC
  & 15.36{\tiny$\pm$0.01} & 0.34{\tiny$\pm$0.00} & 70.28{\tiny$\pm$0.03}
  & 3.46{\tiny$\pm$1.09} & 0.20{\tiny$\pm$0.00} & 0.40{\tiny$\pm$0.00} & 31.7\% & 43.8\% \\
& VLM-RL
  & \textbf{16.07}{\tiny$\pm$2.10} & \textbf{0.39}{\tiny$\pm$0.06} & 95.02{\tiny$\pm$20.63}
  & 5.84{\tiny$\pm$4.74} & 0.20{\tiny$\pm$0.17} & \textbf{0.27}{\tiny$\pm$0.21} & \textbf{50.0\%} & \textbf{68.8\%} \\
& \cellcolor{green!10}DriveVLM-RL
  & \cellcolor{green!10}14.45{\tiny$\pm$1.57} & \cellcolor{green!10}0.38{\tiny$\pm$0.03} & \cellcolor{green!10}\textbf{99.98}{\tiny$\pm$6.64}
  & \cellcolor{green!10}\textbf{1.52}{\tiny$\pm$1.28} & \cellcolor{green!10}\textbf{0.27}{\tiny$\pm$0.06} & \cellcolor{green!10}0.30{\tiny$\pm$0.26} & \cellcolor{green!10}46.8\% & \cellcolor{green!10}53.6\% \\
\bottomrule
\end{tabular*}
\end{table*}

We draw the following observations from Tables~\ref{tab:transfer_stages} and~\ref{tab:failure_modes}.

\textbf{Reward-agnostic functional transfer (RQ1).}
The identical GOB/PAM/TPT pipeline applied to all three
reward paradigms keeps every transferred policy
\emph{operational} on the test routes, with non-trivial
route completion and travel distance at every stage. This
is what RQ1 targets: Sim2Real-AD carries a CARLA-trained
policy through a complete observation/action interface
replacement regardless of reward design, rather than
collapsing it to degenerate behavior.

\textbf{Performance degrades under transfer, but DriveVLM-RL keeps the best safety profile.}
From the Original upper bound to full Phase~2, success rate
drops markedly for all three methods (to $0.20$--$0.27$). At
Phase~2, DriveVLM-RL retains the highest success rate ($0.27$)
and by far the lowest collision severity ($1.52$~km/h, versus
$3.46$ and $5.84$~km/h for ChatScene-SAC and VLM-RL), consistent
with its safety-oriented dynamic-pathway reward; VLM-RL travels
farthest (PR(TD)~$=68.8\%$) but at much higher collision severity.
As Phases~1--2 use a single physics seed over 10 routes, the
absolute success rates carry non-trivial variance; the
collision-severity gap, however, is sizeable and aligns with the
real-vehicle results in Section~\ref{sec:real_world} (RQ4), which provide the primary real-world evidence of transfer fidelity.

\textbf{Stage-wise costs are consistent across methods.}
PAM (Phase~1) reduces travel distance and route completion
relative to the GT-BEV upper bound, because the curvature/PID
interface cannot reproduce the fine-grained maneuvers learned
in simulation; the GOB switch (Phase~2) reduces them further
(travel-distance retention falls to $43.8$--$68.8\%$ of
Original), as the monocular IPM BEV is noisier and
shorter-range. Average speed rises from Phase~1 to Phase~2 for
VLM-RL (from $15.03$ to $16.07$~km/h) while staying flat for DriveVLM-RL
(from $14.61$ to $14.45$~km/h), mirroring the training curves. Notably,
DriveVLM-RL is the only method with uniformly low collision
severity (CS~$\leq 1.75$~km/h at every stage), whereas both
baselines have high-severity stages (up to $\sim$10.7~km/h),
indicating that its semantic safety reward transfers robustly through the pipeline.

\textbf{At Phase~2, failures split between collisions and
red-light violations.}
Table~\ref{tab:failure_modes} categorizes the 30 Phase-2 episodes
per method. Collisions and red-light running are the two dominant
failure modes (stop-sign and stall terminations do not occur on
these Town~2 routes), both traceable to the noisier monocular IPM
BEV, which degrades obstacle clearance and traffic-light-state
estimation alike. DriveVLM-RL attains the most successful episodes
(8/30) and ChatScene-SAC the most collisions (12/30), while
red-light running dominates for VLM-RL (16/30). The residual
Phase-2 gap is thus driven as much by perception-limited signal
compliance as by collision avoidance, motivating the explicit
safety-monitoring layer in the RDP.

\begin{table*}[!t]
\centering
\caption{Phase~2 (full Sim2Real-AD) episode-outcome breakdown over
30 evaluation episodes per method (3 seeds $\times$ 10 routes,
Town~2). Counts (and \% of 30). ``Other'' aggregates stop-sign,
stall, and lane-invasion terminations, none of which occurred.}
\label{tab:failure_modes}
\renewcommand{\arraystretch}{1.25}
\setlength{\tabcolsep}{4pt}
\small
\begin{tabular*}{\textwidth}{@{\extracolsep{\fill}}lcccc@{}}
\toprule
Method & Success$\uparrow$ & Collision$\downarrow$ & Red-light$\downarrow$ & Other \\
\midrule
ChatScene-SAC & 6 (20\%) & 12 (40\%) & 12 (40\%) & 0 \\
VLM-RL        & 6 (20\%) & 8 (27\%) & 16 (53\%) & 0 \\
\rowcolor{green!10}DriveVLM-RL & \textbf{8 (27\%)} & 9 (30\%) & 13 (43\%) & 0 \\
\bottomrule
\end{tabular*}
\end{table*}

Taken together, the simulation transfer experiments establish
two things. First, Sim2Real-AD keeps a CARLA-trained
VLM-guided RL policy \emph{functional} after replacing its
observation and action interfaces, independently of the reward
paradigm. Second, and more importantly, the safety advantage of
the semantically grounded reward is \emph{preserved} through the
pipeline: DriveVLM-RL transfers with the highest Phase-2 success
rate and the lowest collision severity at every stage, while the
absolute magnitudes degrade under the compounded distribution
shift. This in-simulation ordering already foreshadows the
real-world outcome, and is confirmed directly by the zero-shot
closed-loop deployment on a full-scale vehicle in
Section~\ref{sec:real_world}, the central contribution of this
work.

\subsubsection{Ablation Study (RQ2)}
\label{sec:ablation}

We perform a module-level ablation using 
DriveVLM-RL as the base algorithm to quantify 
the contribution of each Sim2Real-AD component. 
Table~\ref{tab:ablation} reports evaluation 
results on Town~2.

\begin{table*}[!t]
\centering
\caption{Module ablation (DriveVLM-RL backbone, Town~2, 10 routes,
mean $\pm$ std). Only the full GOB~+~PAM~+~TPT configuration is
real-vehicle-deployable (\textbf{bold}); Original and
$^\dagger$PAM-only retain privileged GT-BEV (simulation-only), and
$^\ddagger$Direct Transfer feeds GOB-BEV to the unadapted GT-BEV
policy with no bridging, a proxy for naive deployment. Original and
Direct Transfer are the upper and lower bounds; the three stage rows
match the DriveVLM-RL entries of Table~\ref{tab:transfer_stages}.
Variance is over 3 training seeds (Original, Direct Transfer) or 3
evaluation seeds (PAM-only, full). PR(SR): success rate relative to Original.}
\label{tab:ablation}
\renewcommand{\arraystretch}{1.25}
\setlength{\tabcolsep}{2.5pt}
\small
\begin{tabular*}{\textwidth}{@{\extracolsep{\fill}}lccccccc@{}}
\toprule
Configuration & GOB & PAM & TPT
  & SR$\uparrow$ & AC$\downarrow$
  & TD$\uparrow$ & PR(SR) \\
\midrule
\rowcolor{gray!8}
Original (upper bound)
  & & &
  & 0.57{\tiny$\pm$0.15}
  & 0.20{\tiny$\pm$0.26}
  & 186.59{\tiny$\pm$14.00} & 100\% \\
\midrule
+ PAM only$^\dagger$
  & & \checkmark &
  & 0.37{\tiny$\pm$0.06}
  & 0.30{\tiny$\pm$0.20}
  & 117.58{\tiny$\pm$6.07} & 64.3\% \\
\rowcolor{green!10}
+ GOB + PAM + TPT (ours)
  & \checkmark & \checkmark & \checkmark
  & \textbf{0.27}{\tiny$\pm$0.06}
  & \textbf{0.30}{\tiny$\pm$0.26}
  & \textbf{99.98}{\tiny$\pm$6.64}
  & \textbf{46.8\%} \\
\midrule
\rowcolor{gray!8}
Direct transfer$^\ddagger$ (lower bound)
  & & &
  & 0.20{\tiny$\pm$0.00}
  & 0.20{\tiny$\pm$0.00}
  & 84.06{\tiny$\pm$6.20} & 35.0\% \\
\bottomrule
\end{tabular*}
\end{table*}

Three findings stand out. \textit{Each bridging stage trades simulator headroom for
real-vehicle deployability.} Relative to the GT-BEV upper bound
(SR~$0.57$), the PAM action bridge alone costs a moderate amount
(SR~$0.37$, PR~$64.3\%$), mainly from PID speed-tracking lag and
the nonlinear curvature-to-steering mapping at high curvatures, a
loss recoverable through tighter PID calibration
(Theorem~\ref{thm:main_formal}, PAM term); since it retains
privileged GT-BEV, PAM-only isolates this action-side cost but
cannot run on a real vehicle. Replacing GT-BEV with the deployable
GOB-BEV and applying the TPT curriculum yields the full pipeline at
SR~$0.27$ (PR~$46.8\%$); the further drop reflects the noisier,
shorter-range monocular IPM BEV, which distorts the anticipatory
cues for obstacle avoidance beyond $15$~m
(Section~\ref{sec:gob_analysis}). \textit{Given the deployable GOB-BEV observation, the framework
beats naive transfer.} Direct transfer, which feeds GOB-BEV to the
unadapted GT-BEV-trained policy with its original direct-action
output and no bridging, sets the lower bound (SR~$0.20$,
PR~$35.0\%$, TD~$84$~m). The full pipeline recovers to SR~$0.27$ and
TD~$100$~m, a $35\%$ relative SR gain, by adapting the observation
distribution through TPT while holding the action interface and
reward fixed (Definition~\ref{def:tpt},
Remark~\ref{rem:progressive_training}), consistent with the TPT
residual term in Theorem~\ref{thm:main_informal}. \textit{Only the full configuration is real-vehicle-deployable.}
Original and PAM-only rely on privileged GT-BEV, and direct transfer
on a direct action space producing no physical commands; none can
run on a real platform. The full GOB~+~PAM~+~TPT configuration is the
only one both deployable and above the naive lower bound, and is
therefore the configuration validated on the real vehicle in
Section~\ref{sec:real_world}.

\subsubsection{GOB Observation Quality Analysis 
(RQ3)}
\label{sec:gob_analysis}

\textit{1) CARLA GOB-BEV vs.\ GT-BEV Visual 
Comparison.}
Fig.~\ref{fig:gob_carla} compares the front-view camera input, the GOB-BEV, and the ground-truth semantic BEV across four representative CARLA frames (straight car-following, oncoming vehicles, cyclists, and an intersection). Two structural differences
are evident: (i)~GOB-BEV captures road surface and
lane markings only within the forward fan-shaped
field of view ($\sim$0--20~m ahead), whereas GT-BEV
provides full $360^\circ$ coverage; (ii)~front-view
vehicles are projected to approximate BEV positions
(red), while GT-BEV encodes all surrounding traffic.
Despite these coverage limits, GOB-BEV faithfully
reconstructs the \emph{road geometry directly ahead
of the ego vehicle}, the region most critical for the
policy's lane-keeping and collision-avoidance decisions.

\begin{figure*}[!t]
    \centering
    \includegraphics[width=\linewidth]
    {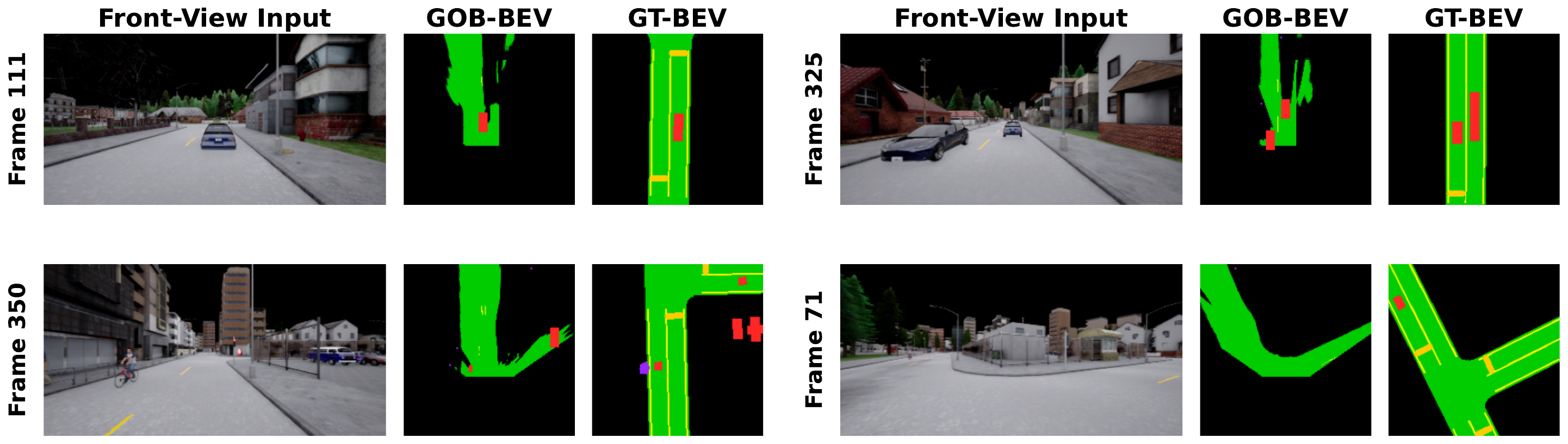}
    \caption{CARLA GOB evaluation on 4 
    representative frames (2 per row). Each 
    case shows (left to right): front-view 
    camera input, GOB-BEV, and GT-BEV. The 
    GOB-BEV captures near-field road geometry 
    (green), lane markings (yellow), and 
    vehicles (red) within the forward 
    fan-shaped field of view; the GT-BEV 
    provides full $360^\circ$ coverage.}
    \label{fig:gob_carla}
\end{figure*}

\textit{2) Quantitative Channel Analysis.}
We evaluate GOB-BEV fidelity on 200 paired frames
from CARLA using per-channel IoU between GOB-BEV and
GT-BEV and channel activation (fraction of nonzero
pixels) (Fig.~\ref{fig:channel_iou}).

The moderate road-channel IoU of $0.35$
(Fig.~\ref{fig:channel_iou}(a)) stems primarily from
\emph{over-projection} rather than missing coverage:
GOB-BEV labels $51.1\%$ of pixels as road versus
$31.1\%$ for GT-BEV (a $1.6\times$ ratio), so within
the camera footprint road \emph{recall} is high
($0.79$ of the forward GT road surface) but
\emph{precision} is only $0.42$. Coverage asymmetry is
minor: only $14\%$ of the GT road lies outside the
monocular field of view, and restricting the IoU to
the forward footprint raises it only from $0.35$ to
$0.38$. Lane, vehicle, walker, and traffic-light
channels show near-zero IoU ($<$0.01) because GOB uses
color-based lane detection ($0.2\%$ activation vs.\
GT's $3.1\%$), objects are projected to approximate
rather than exact positions, and walkers and lights
rarely occupy enough pixels after projection. The road
channel, which dominates the BEV and the policy's
lane-keeping, is thus the one GOB reconstructs most
faithfully, whereas object channels are not spatially
aligned with GT. The policy tolerates this degraded
observation not because GOB-BEV is pixel-accurate but
because Phase~2 fine-tuning adapts it to the GOB-BEV
distribution, which, though over-projected, is
temporally stable (Fig.~\ref{fig:temporal}).

\begin{figure*}[!t]
    \centering
    \includegraphics[width=\linewidth]
    {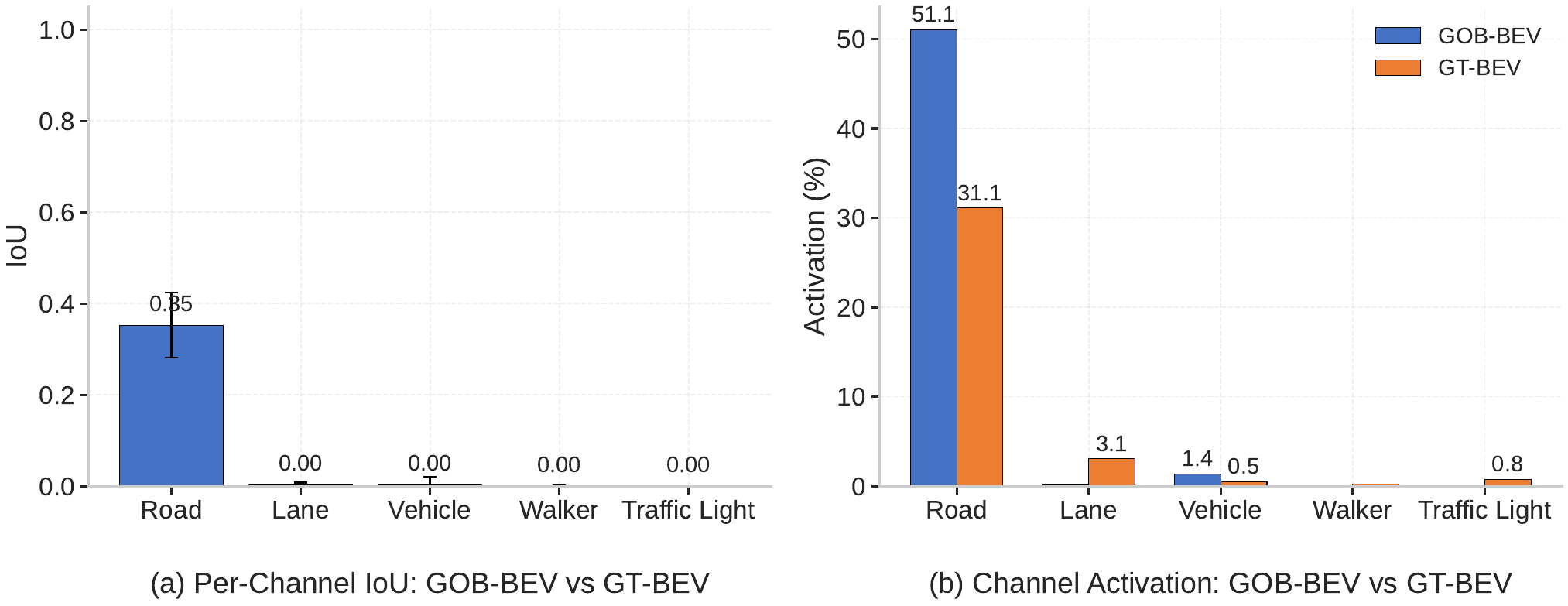}
    \caption{Quantitative GOB evaluation on 200 
    CARLA frames. (a)~Per-channel IoU between 
    GOB-BEV and GT-BEV: road achieves moderate 
    IoU ($0.35$) limited by road over-projection (precision $0.42$);
    other channels are near-zero due to 
    detection-method differences. (b)~Channel 
    activation comparison: GOB-BEV 
    over-activates road (51.1\% vs.\ 31.1\%) 
    due to forward-fan projection, while lane 
    and object channels have lower activation 
    than GT, reflecting the single-camera 
    field-of-view limitation.}
    \label{fig:channel_iou}
\end{figure*}

\textit{3) Temporal Consistency.}
For a deployed policy, temporal stability of the BEV
is as important as absolute accuracy, since flickering
masks cause erratic control. On the Ford E-Transit van we
evaluate two camera mounts: cam0 (forward-facing), the
primary deployment camera in
Section~\ref{sec:real_world}, and cam1 (elevated), as a
comparison. Over a 163-frame real-world sequence
(Fig.~\ref{fig:temporal}), cam0 attains a mean
consecutive-frame road-channel IoU of $0.944$ ($90.1\%$
of frames above $0.90$) and cam1 $0.969$ ($95.1\%$ above
$0.90$); cam1's higher stability reflects its elevated,
pitch-robust viewing angle. Brief dips below $0.90$
coincide with sharp heading changes during lane
transitions. This confirms that GOB produces smooth
observations suitable for closed-loop RL control at
20~Hz, with cam0 fully meeting the primary deployment
configuration's stability requirements.

\begin{figure*}[!t]
    \centering
    \includegraphics[width=0.95\linewidth]
    {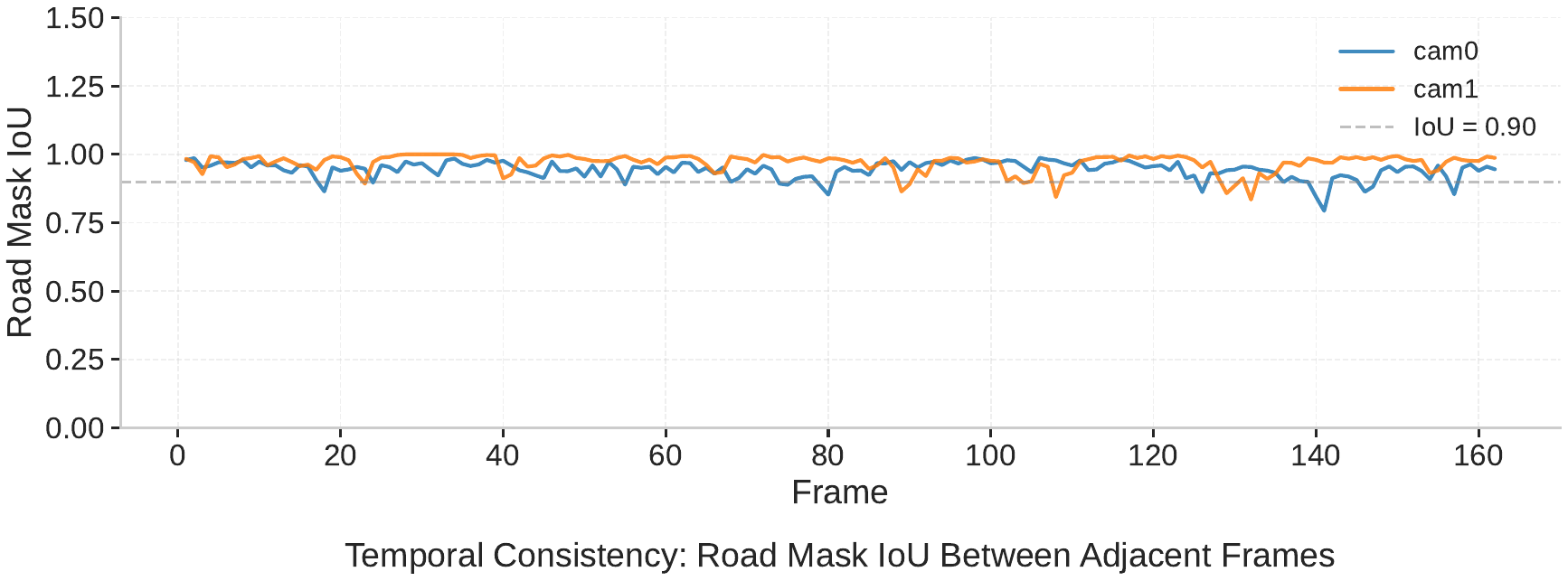}
    \caption{Temporal consistency of the road 
    mask: IoU between adjacent frames on 
    real-world driving data. cam0 
    (forward-facing) achieves 
    mean IoU~$= 0.944$; cam1 (elevated mount) achieves mean 
    IoU~$= 0.969$. The dashed line marks 
    IoU~$= 0.90$.}
    \label{fig:temporal}
\end{figure*}

\begin{figure*}[!t]
    \centering
    \includegraphics[width=\linewidth]
    {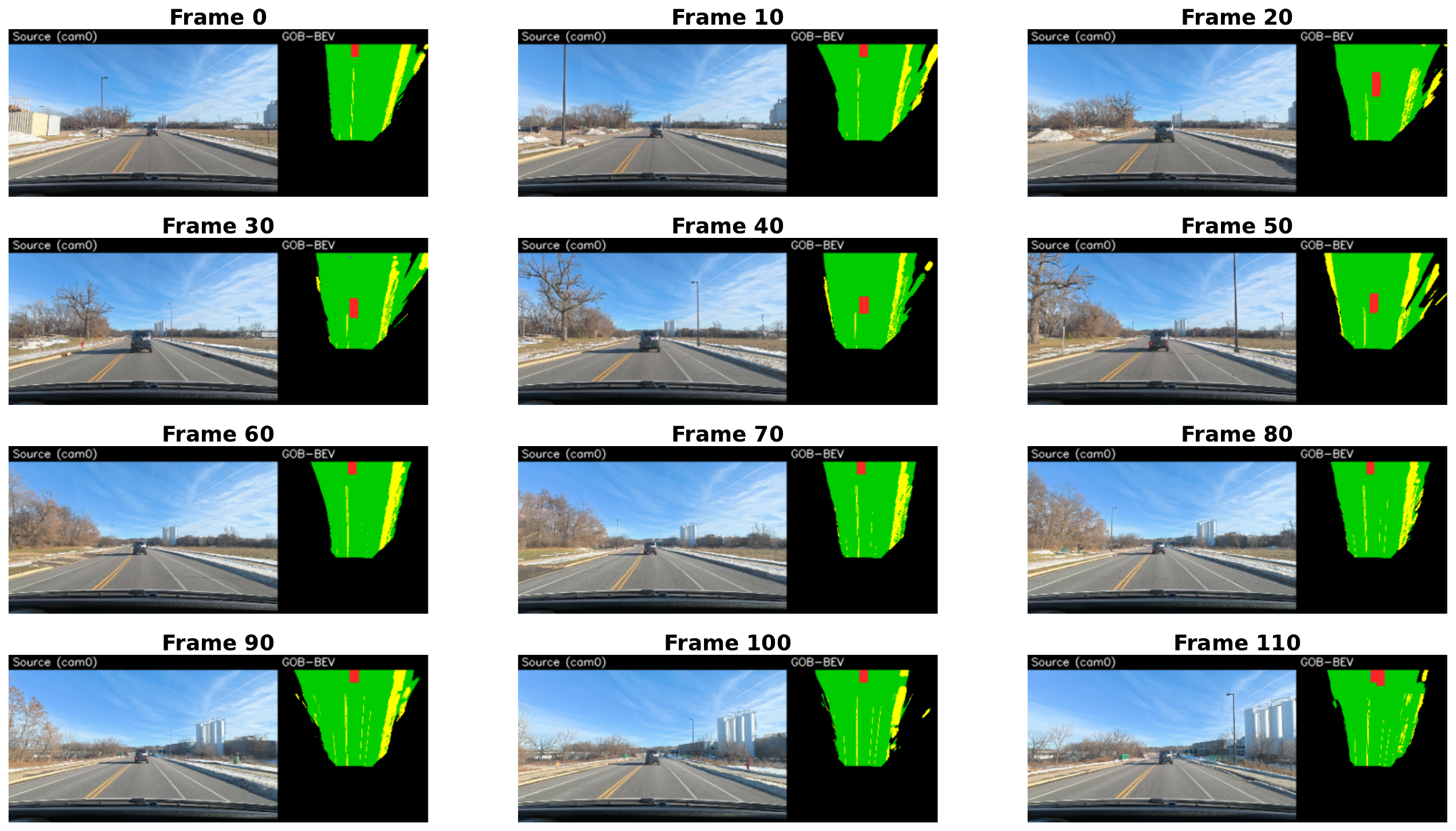}
    \caption{Real-world GOB output on 12 
    representative frames (Frame 0--110, 
    interval 10) from the Ford E-Transit van 
    primary forward-facing camera (cam0), 
    arranged in a 4$\times$3 grid. Each cell 
    shows the source image (left) and GOB-BEV 
    (right). Road surface (green), lane 
    markings (yellow), and vehicles (red) are 
    consistently detected across diverse road 
    conditions using zero-shot transfer from 
    Cityscapes-pretrained SegFormer-B0.}
    \label{fig:gob_real}
\end{figure*}

\textit{4) Real-World GOB Visualization.}
Fig.~\ref{fig:gob_real} presents 12 
representative GOB outputs from the primary 
forward-facing camera (cam0) on real-world 
driving footage, arranged in a 
4$\times$3 grid (frames labeled 0--110 at 
interval 10). The pipeline produces 
interpretable BEV masks across diverse 
conditions: straight roads with clear lane 
markings, gentle curves, and scenes with 
leading vehicles. Road surface (green) forms 
a consistent fan-shaped region ahead; lane 
markings (yellow) are detected via color 
filtering restricted to the semantic road 
region; and vehicles (red) are projected to 
approximate BEV positions using the 
bottom-center ground contact point of each 
detected bounding box as the IPM projection 
anchor. Notably, the Cityscapes-pretrained
SegFormer-B0 model generalizes zero-shot to 
these real-world winter conditions without
any domain-specific fine-tuning, supporting
the claim that the GOB pipeline is
platform-agnostic: the same off-the-shelf
Cityscapes segmenter feeds both the CARLA and
the real-vehicle pipelines with no CARLA- or
vehicle-specific retraining.

\subsection{Real-World Deployment Experiments (RQ4)}
\label{sec:real_world}

We deploy the complete Sim2Real-AD framework on
a full-scale autonomous vehicle to validate
zero-shot sim-to-real transfer without any
real-world training data. All policy parameters
are frozen after simulation training; no
real-world fine-tuning is performed at any stage. We deploy two VLM-guided RL backbones through this pipeline, DriveVLM-RL and VLM-RL, which share the same CLIP-based static (CLG) reward and differ mainly in DriveVLM-RL's added dynamic, attention-gated semantic-risk pathway; comparing them under the same bridge isolates whether the richer semantic reward, not the transfer framework, drives real-world safety. We omit ChatScene-SAC, whose hand-crafted reward is not semantically grounded.

\subsubsection{Platform and Hardware}
\label{sec:platform}

The experiments were conducted on a 
lab-developed full-scale electric Ford E-Transit 
autonomous van equipped with a drive-by-wire 
system. As illustrated in 
Fig.~\ref{fig:real-word-car}(a), the sensor 
configuration includes three LiDAR units, seven 
high-resolution RGB cameras, and a front-facing 
radar, providing complementary spatial, semantic, 
and velocity information. A key point of the 
deployment is that \emph{Sim2Real-AD uses only 
the front-facing monocular camera for policy 
inference}, demonstrating that the framework 
does not require expensive multi-sensor fusion 
for the observation bridge. Onboard computation 
is supported by an NVIDIA RTX A6000 GPU, which 
handles real-time inference for both the policy 
network and the perception stack. For data 
logging, telemetry, and remote monitoring, the 
vehicle is equipped with a NETGEAR Nighthawk 
M6 Pro 5G router, providing high-bandwidth and 
low-latency wireless connectivity during 
experiments.

\begin{figure*}[!t]
    \centering
    \includegraphics[width=1\linewidth]
    {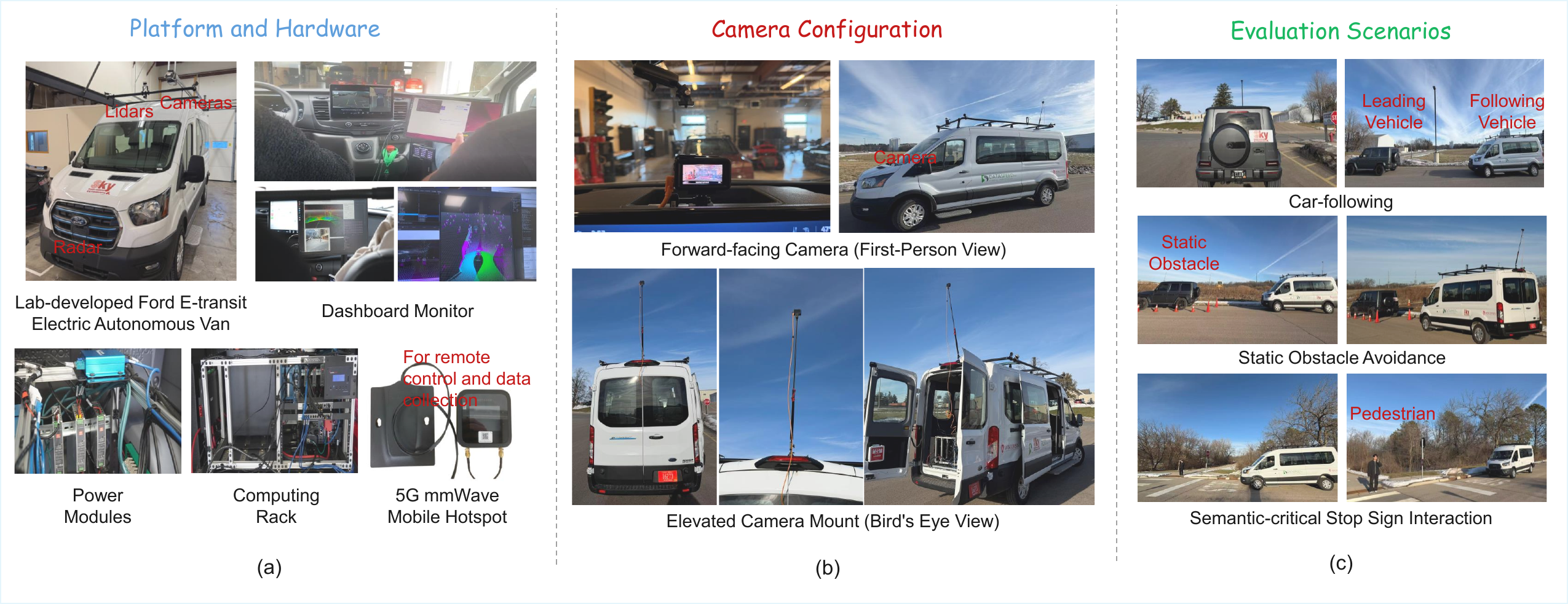}
    \caption{Real-world experimental platform 
    and evaluation scenarios. (a) Lab-developed 
    full-scale electric Ford E-Transit autonomous 
    van platform with drive-by-wire control, 
    multimodal perception sensors, onboard 
    computing, and wireless communication 
    infrastructure. (b) Camera configuration: 
    a forward-facing camera and an elevated 
    camera mount, matching the dual-camera setup 
    used in DriveVLM-RL training for reward 
    computation; only the forward-facing camera 
    is used for policy inference. (c) Real-world 
    driving scenarios used for evaluation: 
    routine car-following, static obstacle 
    avoidance, and semantic-critical stop sign 
    interaction.}
    \label{fig:real-word-car}
\end{figure*}

\subsubsection{Camera Configuration}
DriveVLM-RL's reward computation during 
training relies on two camera views: a 
forward-facing ego-centric view for the static 
CLIP-based pathway and an elevated view for 
the dynamic LVLM pathway, which benefits from 
a wider spatial context. To maintain consistency 
with this training configuration and enable 
shadow-mode reward analysis 
(Section~\ref{sec:deployment_protocol}), we
use two of the platform's cameras as
shown in Fig.~\ref{fig:real-word-car}(b): a
forward-facing camera and an
elevated camera mount. Both are 
calibrated using the checkerboard procedure 
described in Section~\ref{sec:calibration}. 
Critically, the RL policy itself receives only 
the forward-facing view as input through the 
GOB pipeline; the elevated camera is used 
exclusively for reward signal analysis and 
does not participate in the control loop.

\subsubsection{Platform Calibration and Deployment Preparation}
\label{sec:calibration}
A key practical advantage of Sim2Real-AD is its 
minimal setup effort compared to learning-based 
domain adaptation methods~\citep{zhu2017unpaired,ganin2016domain}, 
which typically require thousands of real-world 
images and iterative GPU training. 
Table~\ref{tab:calibration} summarizes the 
platform-specific calibration items together with 
the route preparation needed for real-world deployment. 
The vehicle calibration itself requires approximately
30~min and no training data, while route waypoint
recording adds about 5~min for a given test site. The PID gains for the E-Transit van differ from
the CARLA defaults owing to the vehicle's
greater mass and different throttle response
characteristics. Recalibration requires only
$\sim$10~min of step-response testing on a
straight road, while the remaining calibration
items are obtained through standard camera and
vehicle measurements. This efficiency follows
from the PAM design (Definition~\ref{def:pam}),
which isolates the vehicle-specific mapping and
allows the platform-dependent component to be
updated without modifying the policy.

\begin{table*}[!t]
\centering
\caption{Practical setup for real-world deployment 
on the Ford E-Transit van, including platform calibration 
and route preparation. Vehicle calibration requires 
approximately 30~min; total practical setup time is 
approximately 35~min. No training data are required.}
\label{tab:calibration}
\renewcommand{\arraystretch}{1.25}
\setlength{\tabcolsep}{4pt}
\small
\begin{tabular*}{\textwidth}{@{\extracolsep{\fill}}llll@{}}
\toprule
Parameter & Value & Method & Time \\
\midrule
Front camera field of view 
  & $110^\circ$ FOV 
  & Checkerboard~\citep{zhang2000flexible} 
  & 15\,min \\
Camera height ($h$)              
  & 1.7\,m          
  & Tape measure & 2\,min \\
Camera pitch ($\alpha$)          
  & $0^\circ$        
  & Included above & --- \\
Wheelbase ($L$)                  
  & 3.67\,m          
  & Vehicle manual & 1\,min \\
Max steer angle ($\delta_{\max}$) 
  & $38.5^\circ$    
  & Vehicle manual & 1\,min \\
PID gains $(K_p, K_i, K_d)$     
  & (0.8, 0.1, 0.15) 
  & Step-response test & 10\,min \\
\midrule
\textbf{Vehicle calibration subtotal} & & & \textbf{$\sim$30\,min} \\
\midrule
Route waypoints                  
  & 47 points        
  & One manual drive & 5\,min \\
\midrule
\textbf{Total practical setup time} & & & \textbf{$\sim$35\,min} \\
\bottomrule
\end{tabular*}

\end{table*}

\subsubsection{Deployment Protocol}
\label{sec:deployment_protocol}

The DriveVLM-RL-trained actor network runs at 20~Hz,
generating curvature and desired speed commands
that are translated to steering and throttle
via PAM. Owing to its lightweight architecture,
the policy satisfies real-time constraints and
runs fully on onboard computation. Although PAM maps policy outputs to desired speeds up to $v_{\max }$ = 35 km/h, the safety layer constrains actual vehicle speed to 15 km/h during initial real-world testing. To analyze the behavior of DriveVLM-RL's
dual-pathway reward structure under real-world 
conditions, we additionally run both reward 
pathways in shadow mode during deployment: 
they process live sensor data and produce 
reward signals, but their outputs have no 
effect on vehicle control decisions. This 
shadow-mode analysis is entirely separate from 
the control loop and is consistent with the 
training-deployment decoupling property 
described in Section~\ref{sec:preliminaries}: 
the deployed policy is a standalone lightweight 
network with no VLM inference in the control 
path. Shadow mode serves purely as a diagnostic 
tool to examine the alignment between semantic 
risk estimation and real-world behavioral 
responses, and does not alter the zero-shot 
transfer evaluation in any way.

\subsubsection{Evaluation Scenarios}
\label{sec:scenarios}

We evaluate the Sim2Real-AD framework across
three real-world driving scenarios of
increasing semantic complexity, as illustrated
in Fig.~\ref{fig:real-word-car}(c). As noted above, both DriveVLM-RL and VLM-RL are deployed through the identical Sim2Real-AD pipeline, enabling a controlled comparison of the two reward paradigms under the same bridge.

\begin{itemize}
\item \textbf{S1: Routine Car-Following.}
The agent maintains safe longitudinal control
behind a leading vehicle exhibiting
non-constant speed on an open road, evaluating
basic speed regulation and distance keeping.

\item \textbf{S2: Static Obstacle Avoidance.}
Static obstacles are placed along the driving
route, requiring the agent to detect and safely
maneuver around unexpected obstructions while
maintaining lane-level control.

\item \textbf{S3: Semantic-Critical Stop Sign
Interaction.} The agent must recognize a stop
sign and execute appropriate stopping behavior,
potentially in the presence of pedestrians
near the crosswalk, evaluating semantic
understanding and traffic rule compliance.
\end{itemize}

All experiments were conducted on December~23, 2025: Scenario~1 along Sprocket Drive, Madison, WI, and Scenarios~2 and~3 along Discovery Path, Madison, WI.

\subsubsection{Offline Policy Reaction Test on Real Recordings}
\label{sec:offline_test}

Before the closed-loop trials, we replay the recorded real-world cam0 frames through the deployed DriveVLM-RL checkpoint via GOB, without actuating the vehicle; the policy receives the real per-frame ego speed and straight-ahead waypoints (the sites are straight or lightly curved single-lane segments), with hazard frames labeled by YOLO. Running the actual checkpoint on real frames, this is a fully reproducible reaction check (Table~\ref{tab:offline_test}). The attention gate fires only on real hazards, rising from $0\%$ in routine following to $19\%$ at the static obstacle and $46\%$ in the stop-sign-with-pedestrian scenario, and the policy reacts accordingly: its commanded desired speed is essentially unchanged during routine following but drops by $0.7$~km/h at the obstacle and $2.4$~km/h at the pedestrian on gate-on frames, scaling with hazard severity. We read these as a responsiveness check rather than a quality measure, since open-loop single-frame replay only approximates closed-loop control: the reported magnitudes are raw commands near $v_{\max}$ on a clear road, whereas realized closed-loop speeds are far lower ($\sim$1--3~m/s, Fig.~\ref{fig:exp1_overall_results}), bounded by the 15~km/h safety cap and the lead vehicle's pace. The comparative safety advantage over VLM-RL is established by the closed-loop results (Fig.~\ref{fig:quantitative_safety}), and latency is profiled separately in Section~\ref{sec:realtime} (Table~\ref{tab:latency}).

\begin{table*}[!t]
\centering
\caption{Offline policy reaction test: recorded real-world cam0 frames replayed through the deployed DriveVLM-RL checkpoint via GOB. The attention gate fires selectively on real hazards, and the commanded desired speed (decoded as $(a_2+1)/2 \times v_{\max}$, $v_{\max}=35$~km/h) drops on gate-on frames, scaling with hazard severity. This is a reproducible responsiveness check, not a superiority claim (Fig.~\ref{fig:quantitative_safety}).}
\label{tab:offline_test}
\renewcommand{\arraystretch}{1.25}
\setlength{\tabcolsep}{4pt}
{\small

\begin{tabular*}{\textwidth}{@{\extracolsep{\fill}}lccccc@{}}
\toprule
Scenario & Frames & Gate activation & \multicolumn{2}{c}{Desired speed (km/h)} & $\Delta$ \\
\cmidrule(lr){4-5}
 & & & clear & gate-on & \\
\midrule
S1: Routine car-following   & 163 & 0/163 (0\%)  & 33.7 & n/a  & n/a \\
S2: Static obstacle         & 59  & 11/59 (19\%) & 32.9 & 32.2 & $-0.7$ \\
S3: Stop sign + pedestrian  & 65  & 30/65 (46\%) & 32.4 & 30.0 & $-2.4$ \\
\bottomrule
\end{tabular*}

}
\end{table*}

\begin{figure*}[!t]
    \centering
    \begin{subfigure}[b]{0.48\linewidth}
        \centering
        \includegraphics[width=\linewidth]
        {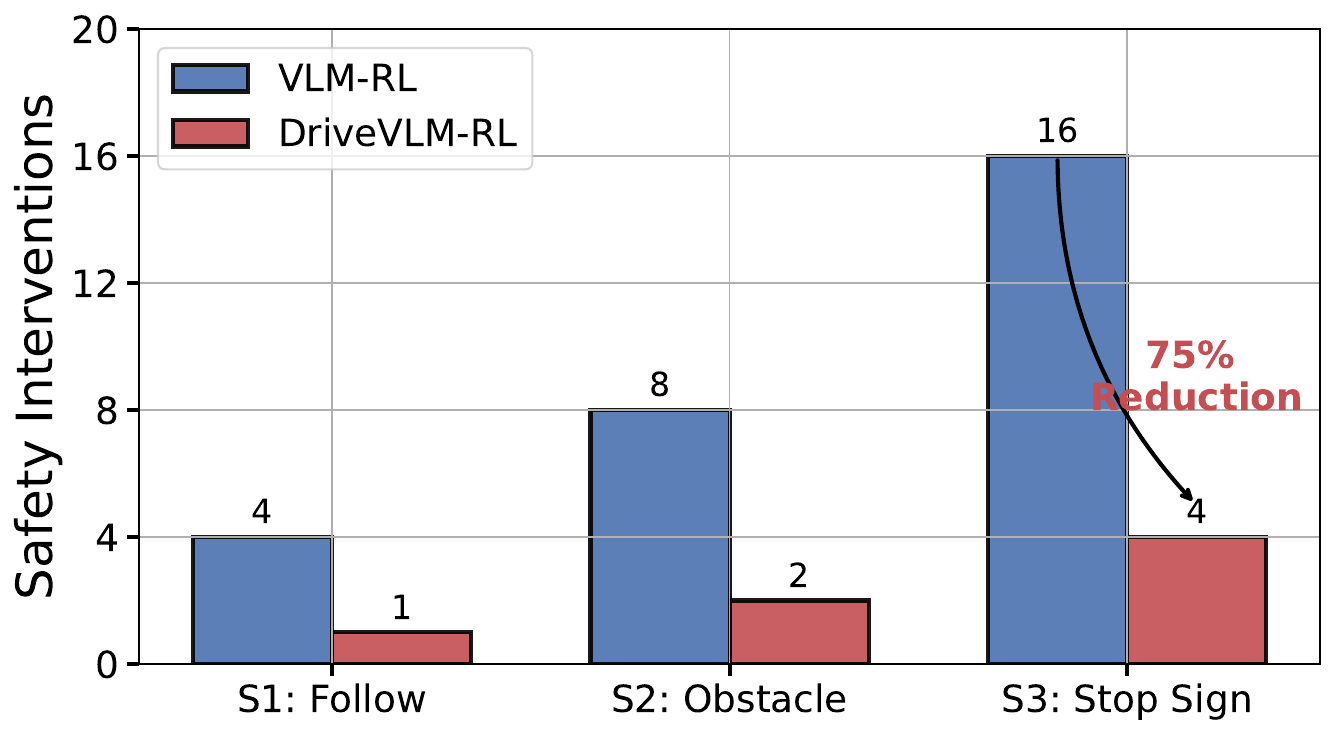}
        \label{fig:safety_a}
    \end{subfigure}
    \hfill
    \begin{subfigure}[b]{0.48\linewidth}
        \centering
        \includegraphics[width=\linewidth]
        {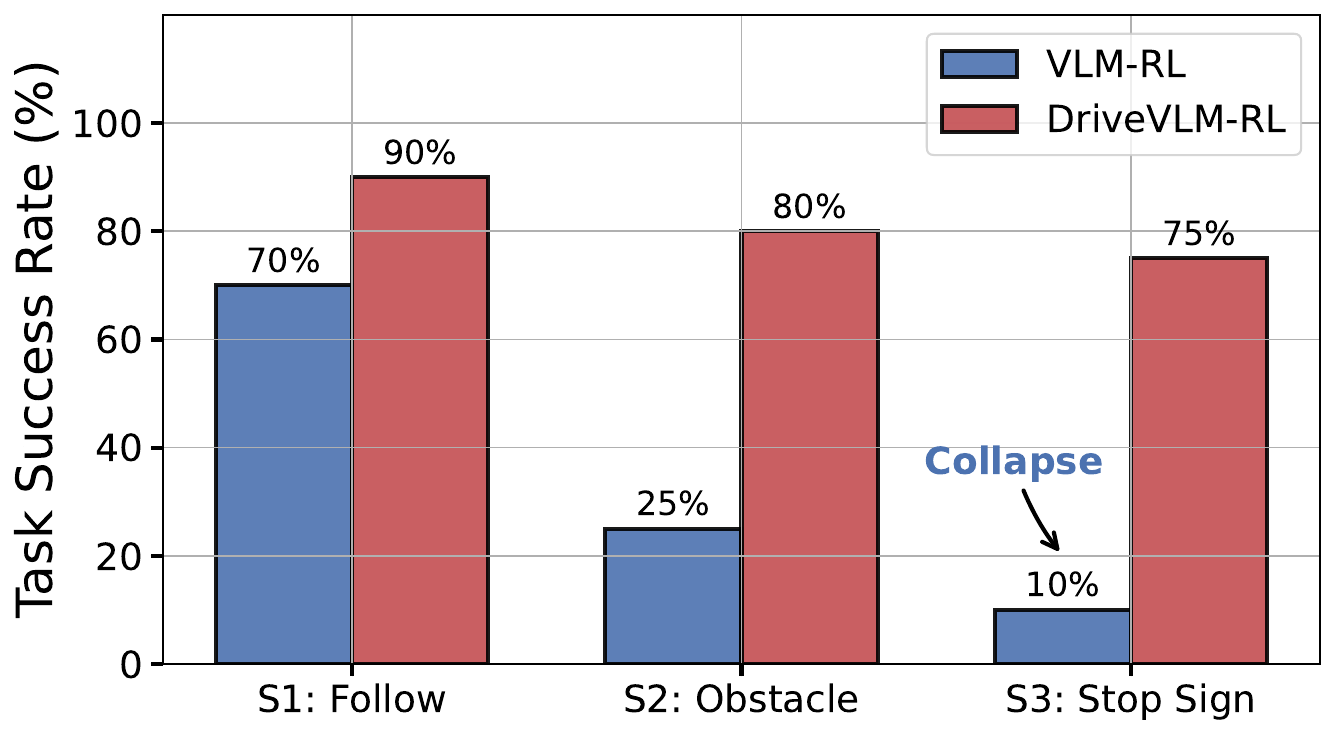}
        \label{fig:safety_b}
    \end{subfigure}
    \caption{Quantitative real-world evaluation
    results over 20 trials per scenario.
    (a) Safety driver interventions across
    three scenarios. (b) Task success rates.
    Both algorithms are deployed through the
    identical Sim2Real-AD pipeline.}
    \label{fig:quantitative_safety}
\end{figure*}

\subsubsection{Quantitative Results}
\label{sec:closedloop}
For each scenario, we perform 20 independent
trials under identical conditions. We
define three trial outcomes:

\begin{itemize}
\item \textbf{Success:} the task is completed 
fully autonomously, i.e., reaching the designated goal region or completing the required interaction without any human intervention and within a reasonable time budget.
\item \textbf{Safety Violation:} a trained 
safety driver intervenes to prevent a collision 
or traffic rule violation; such trials are not 
counted as successes.
\item \textbf{Stagnation:} the vehicle remains 
safe but fails to complete the task within the 
allotted time due to indecision or overly 
conservative behavior.
\end{itemize}

Results are summarized in 
Fig.~\ref{fig:quantitative_safety}. Safety 
violations are reported explicitly through 
safety driver intervention counts; stagnation 
cases are reflected implicitly in the task 
success rate. We emphasize that these on-vehicle scenarios are intentionally isolated and low-speed, and are therefore not directly comparable in difficulty to the dense-traffic, signal-regulated 3000\,m simulation routes; Section~\ref{sec:limitations} explains how the resulting success rates should be interpreted. Given the single platform, single test site, and single day, we present these closed-loop results as a proof-of-concept case study rather than a comprehensive field evaluation.

\begin{figure*}[!t]
    \centering
    \begin{subfigure}[b]{\linewidth}
        \centering
        \includegraphics[width=1\linewidth]{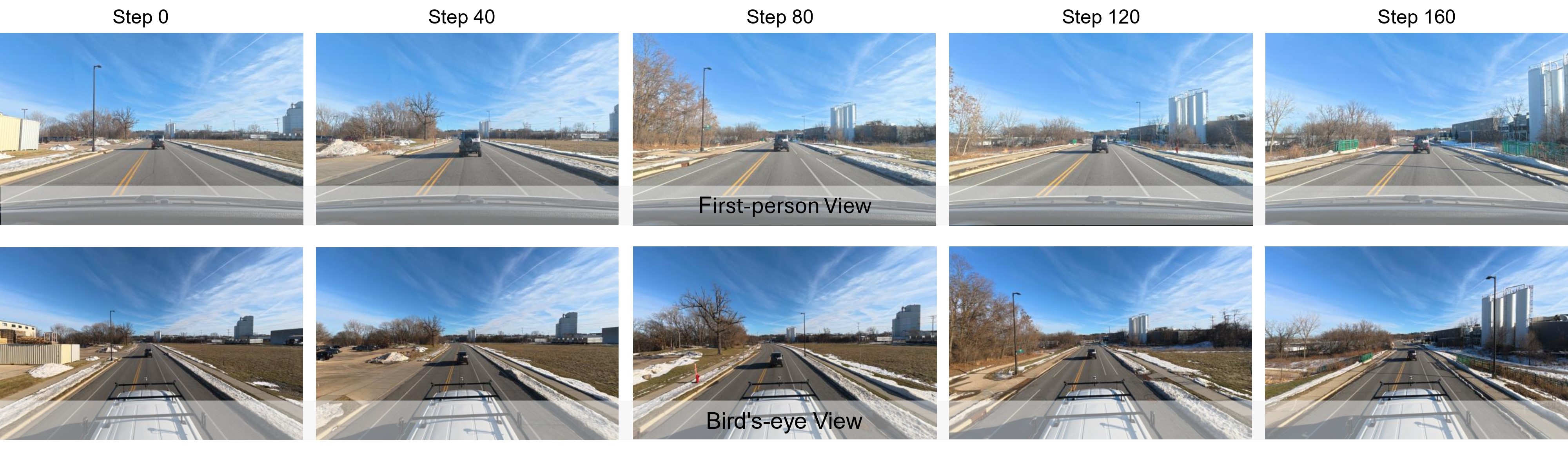}
        \caption{Representative car-following episode: first-person view (top) and bird's-eye view (bottom)}
        \label{fig:case_carfollow}
    \end{subfigure}

    \vspace{0.25cm}

    \begin{subfigure}[b]{0.48\linewidth}
        \centering
        \includegraphics[width=\linewidth]{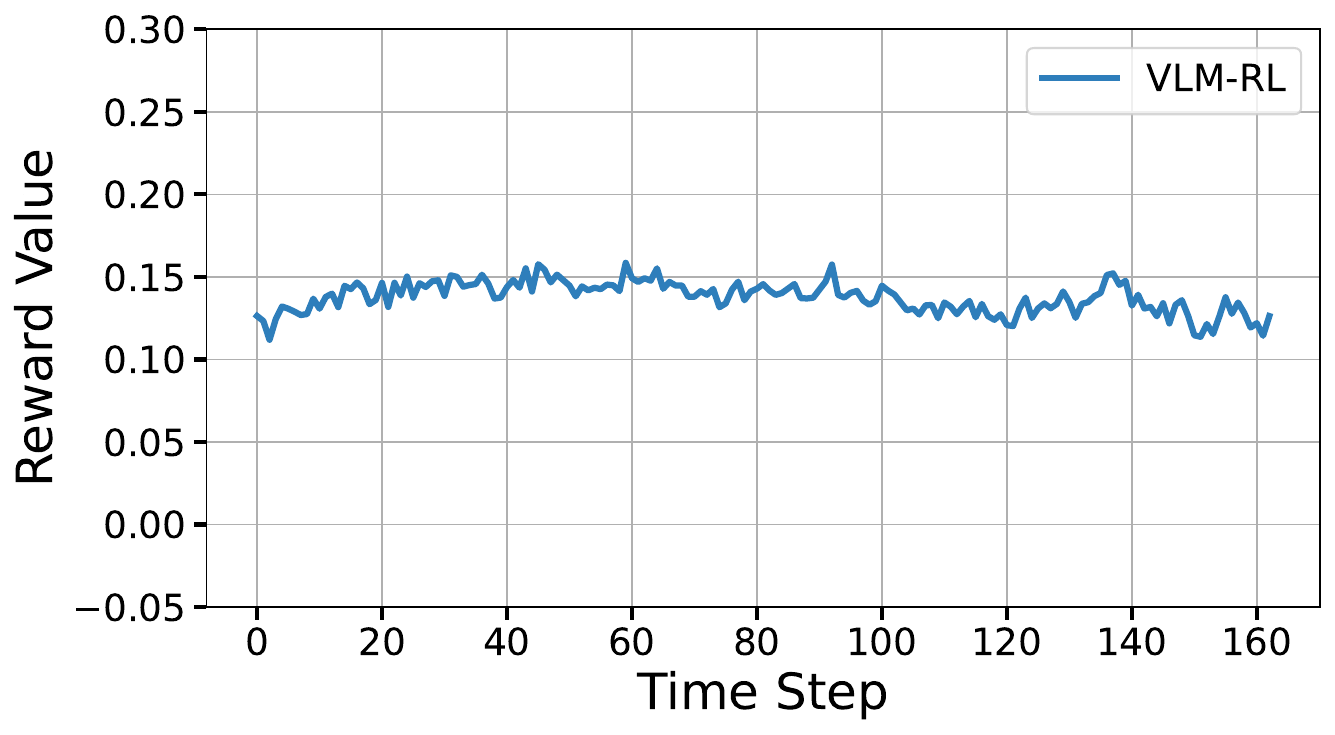}
        \caption{VLM-RL reward}
    \end{subfigure}\hfill
    \begin{subfigure}[b]{0.48\linewidth}
        \centering
        \includegraphics[width=\linewidth]{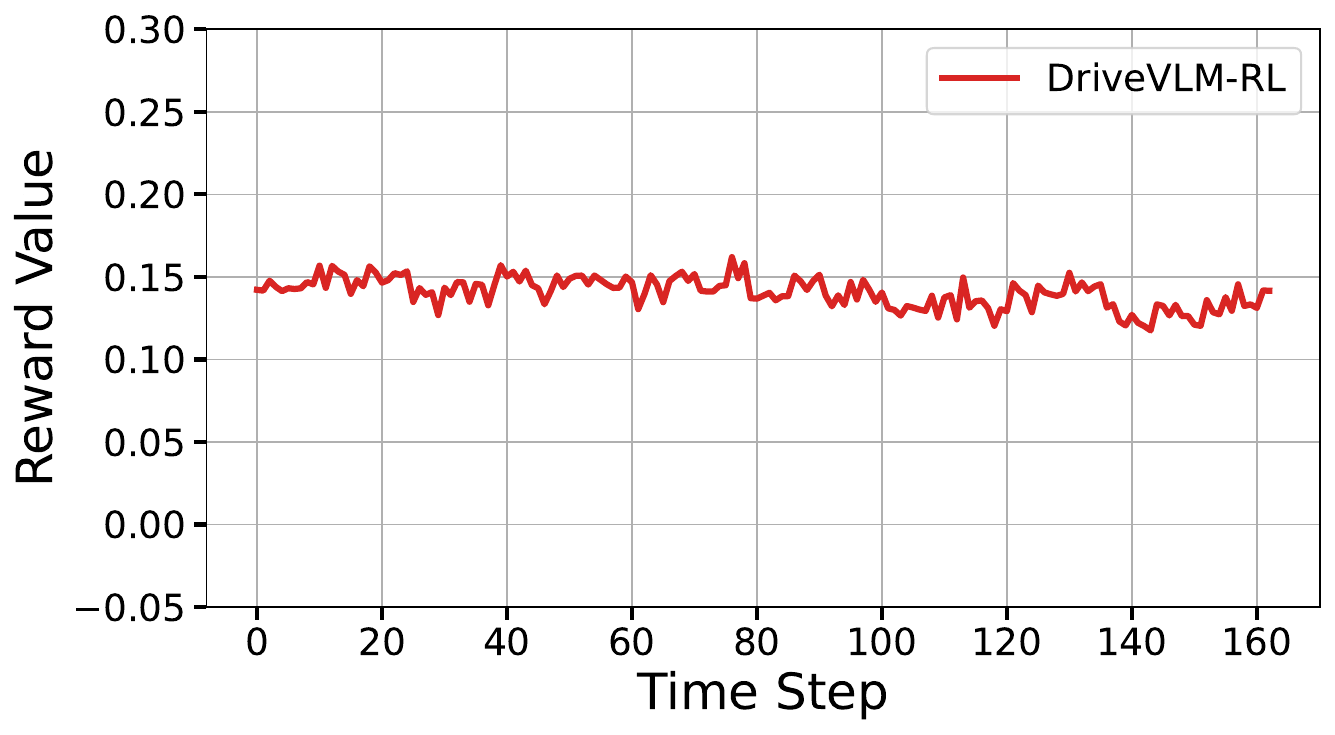}
        \caption{DriveVLM-RL reward}
    \end{subfigure}\\[2mm]
    \begin{subfigure}[b]{0.48\linewidth}
        \centering
        \includegraphics[width=\linewidth]{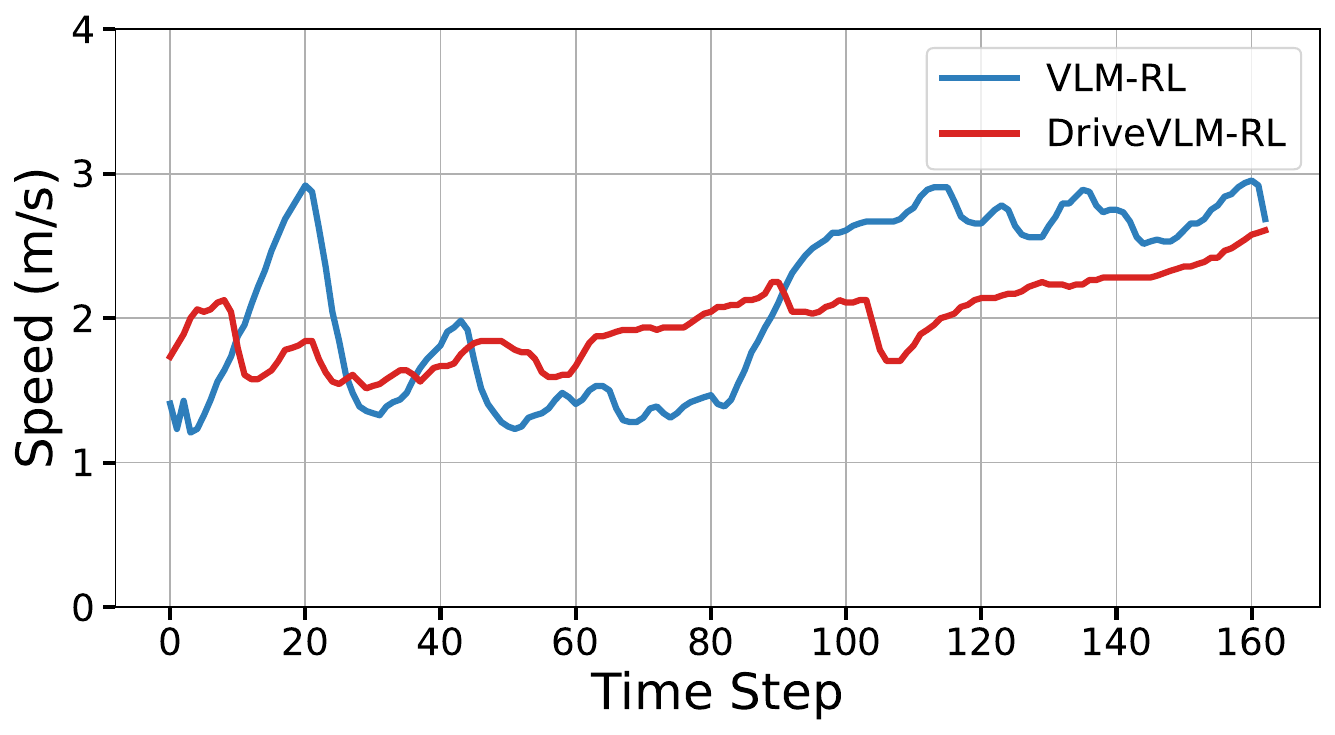}
        \caption{Speed profiles}
    \end{subfigure}\hfill
    \begin{subfigure}[b]{0.48\linewidth}
        \centering
        \includegraphics[width=\linewidth]{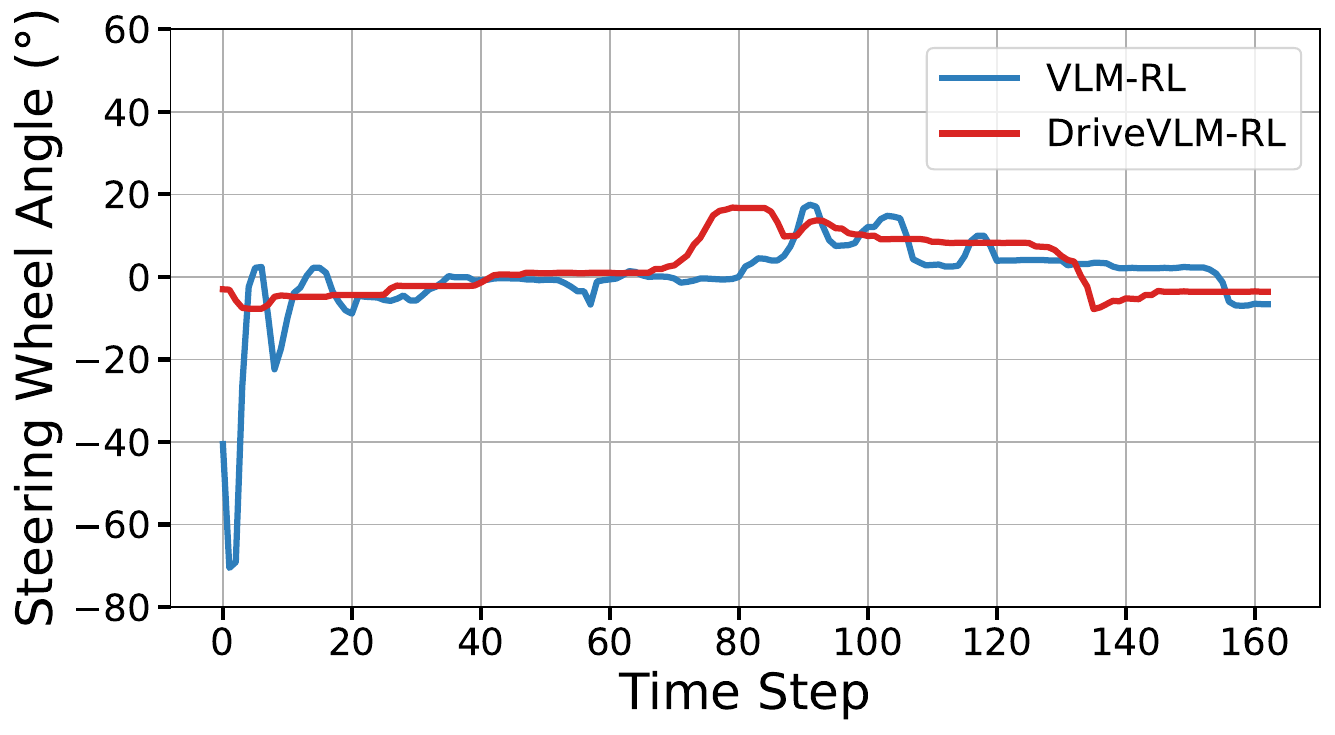}
        \caption{Steering profiles}
    \end{subfigure}

    \caption{Representative S1 (routine car-following) episode. (a) First-person (top, policy input) and bird's-eye (bottom, shadow-mode reward) views; the ego vehicle keeps a consistent following distance without triggering the dynamic pathway. (b,c) VLM-RL and DriveVLM-RL reward signals: both stay stable and positive, with no dynamic intervention in this routine scenario. (d) Speed: DriveVLM-RL regulates speed more smoothly than VLM-RL. (e) Steering: DriveVLM-RL makes more gradual adjustments.}
    \label{fig:exp1_overall_results}
\end{figure*}

Sim2Real-AD enables functional zero-shot
deployment. DriveVLM-RL deployed through Sim2Real-AD
achieves success rates of 90\%, 80\%, and 75\%
across S1--S3 (18, 16, and 15 of 20), without
any real-world training data or fine-tuning.
This is first-hand evidence that the four-module pipeline (GOB, PAM, TPT, RDP) supports closed-loop operation on a full-scale vehicle for these scoped scenarios, though not, on its own, general sim-to-real driving competence. The framework preserves semantic safety structure, not merely driving capability.
The largest difference appears in S3, where VLM-RL
requires safety-driver intervention in 16 of 20 trials
(80\%) versus DriveVLM-RL's 4 (20\%), a 75\% reduction
in critical safety violations; overall, VLM-RL succeeds in
only 14, 5, and 2. The S2 and S3 gaps are large, whereas the
S1 difference (18 vs.\ 14) lies within the binomial noise of
$n=20$ and is comparable. VLM-RL's static single-frame CLIP
reward cannot reliably distinguish a temporary obstacle or
stop-sign context from a leading vehicle, whereas
DriveVLM-RL's dynamic pathway supplies context-dependent
reasoning under real-world visual variability. Because both algorithms run through the
identical pipeline, these gaps isolate the reward design, not
the bridging modules, and match the simulation transfer
ordering, where DriveVLM-RL retained the highest post-transfer
success rate and lowest collision severity
(Table~\ref{tab:transfer_stages}). This pattern holds only within this scoped setting, and the small sample ($n=20$) warrants caution.

\begin{figure*}[!t]
    \centering
    \begin{subfigure}[b]{\linewidth}
        \centering
        \includegraphics[width=1\linewidth]{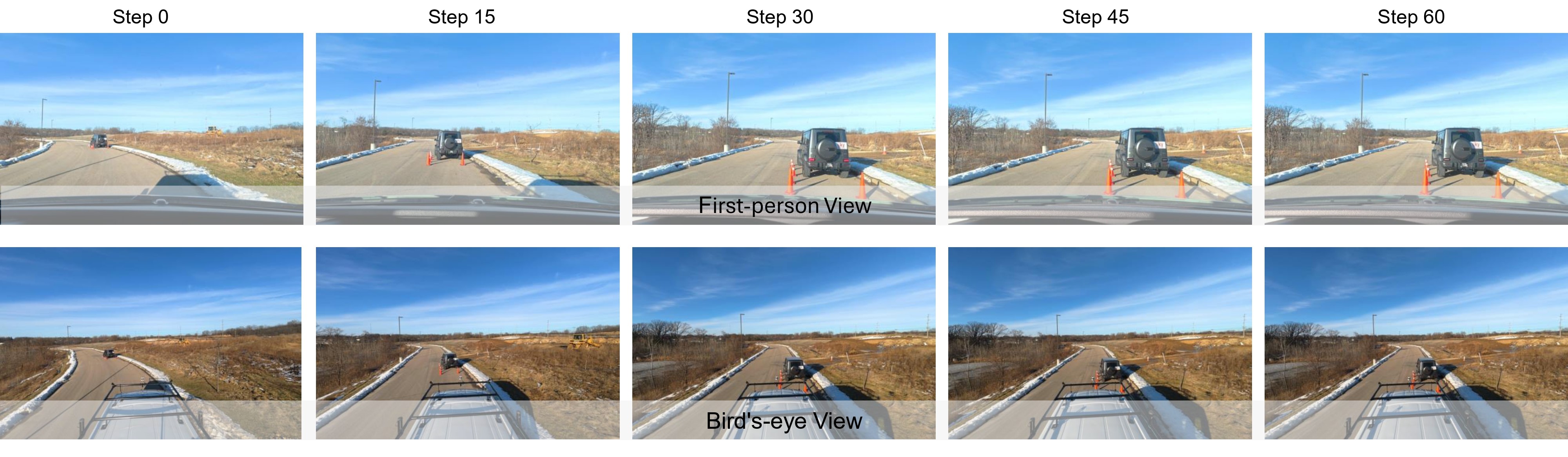}
        \caption{Representative VLM-RL failure case: first-person view (top) and bird's-eye view (bottom)}
        \label{fig:case_obstacle}
    \end{subfigure}

    \vspace{0.25cm}

    \begin{subfigure}[b]{0.48\linewidth}
        \centering
        \includegraphics[width=\linewidth]{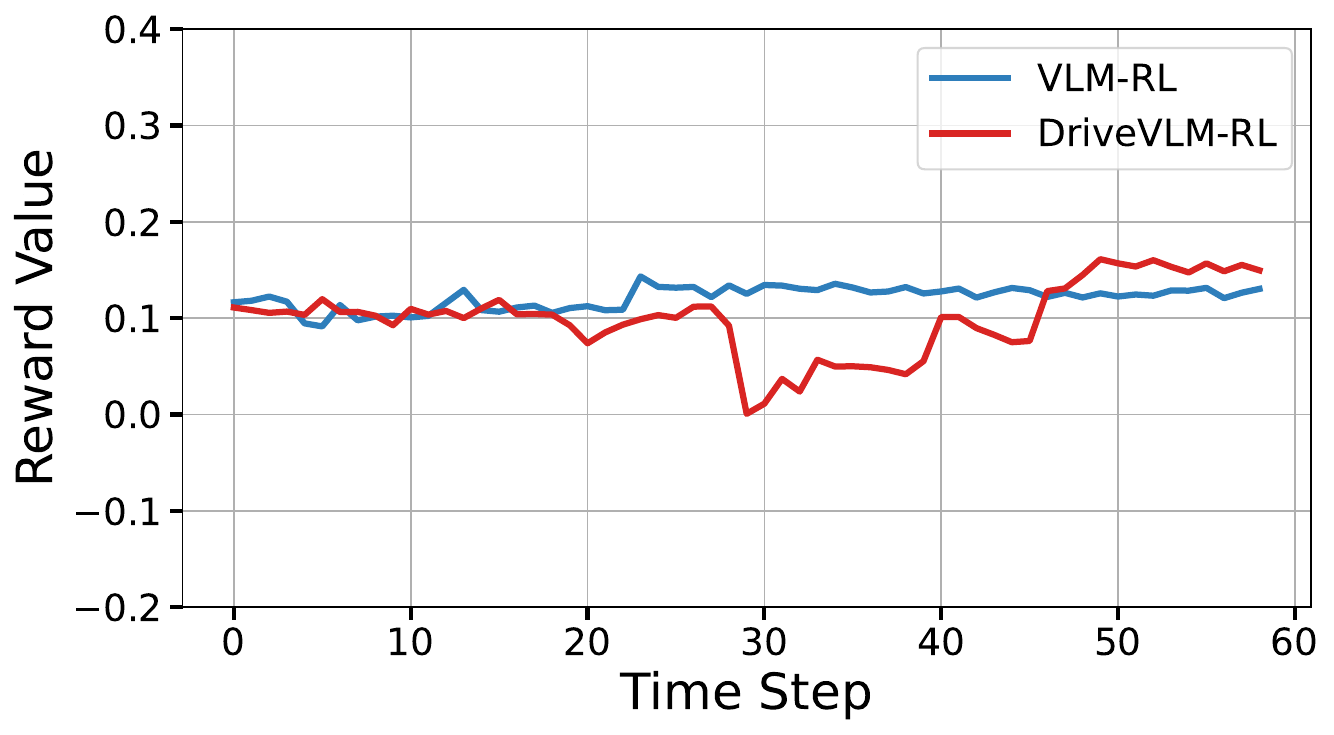}
        \caption{Reward signals}
        \label{fig:exp2_reward}
    \end{subfigure}\hfill
    \begin{subfigure}[b]{0.48\linewidth}
        \centering
        \includegraphics[width=\linewidth]{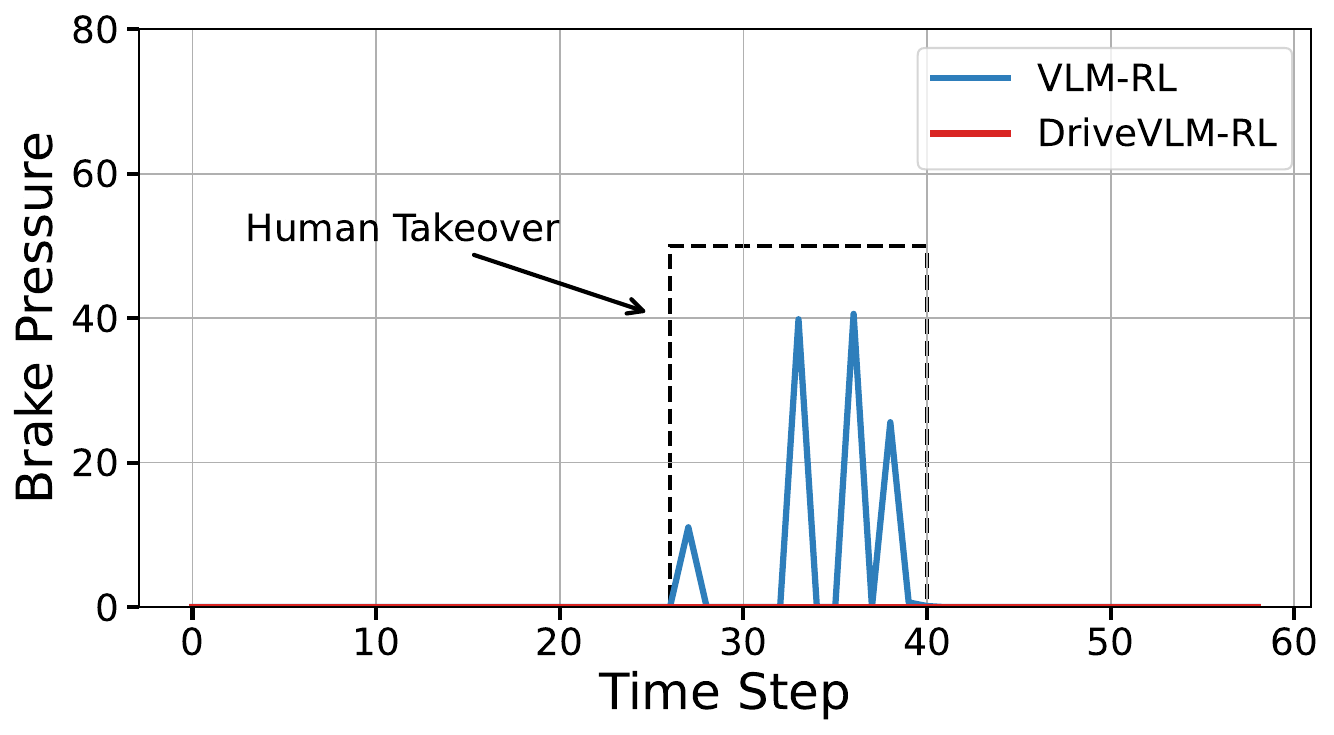}
        \caption{Brake pressure profiles}
        \label{fig:exp2_brake}
    \end{subfigure}\\[2mm]
    \begin{subfigure}[b]{0.48\linewidth}
        \centering
        \includegraphics[width=\linewidth]{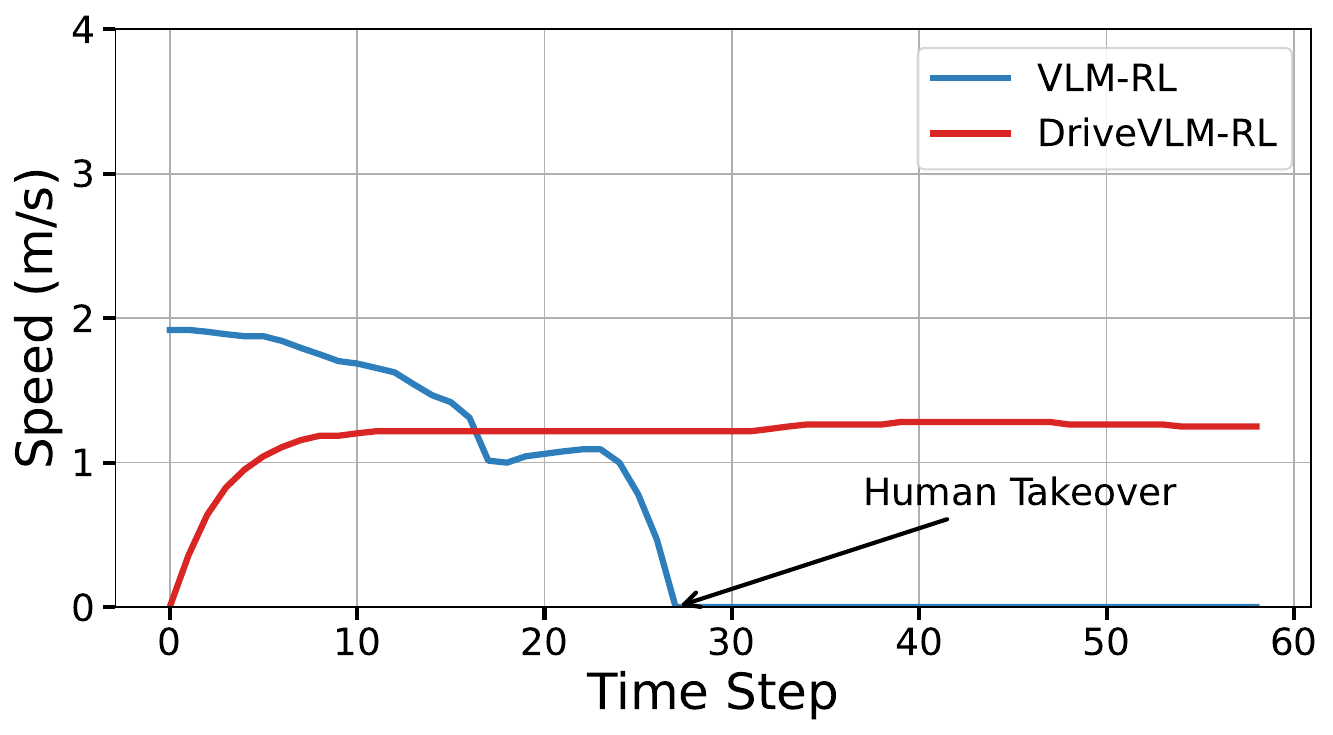}
        \caption{Vehicle speed profiles}
        \label{fig:exp2_speed}
    \end{subfigure}\hfill
    \begin{subfigure}[b]{0.48\linewidth}
        \centering
        \includegraphics[width=\linewidth]{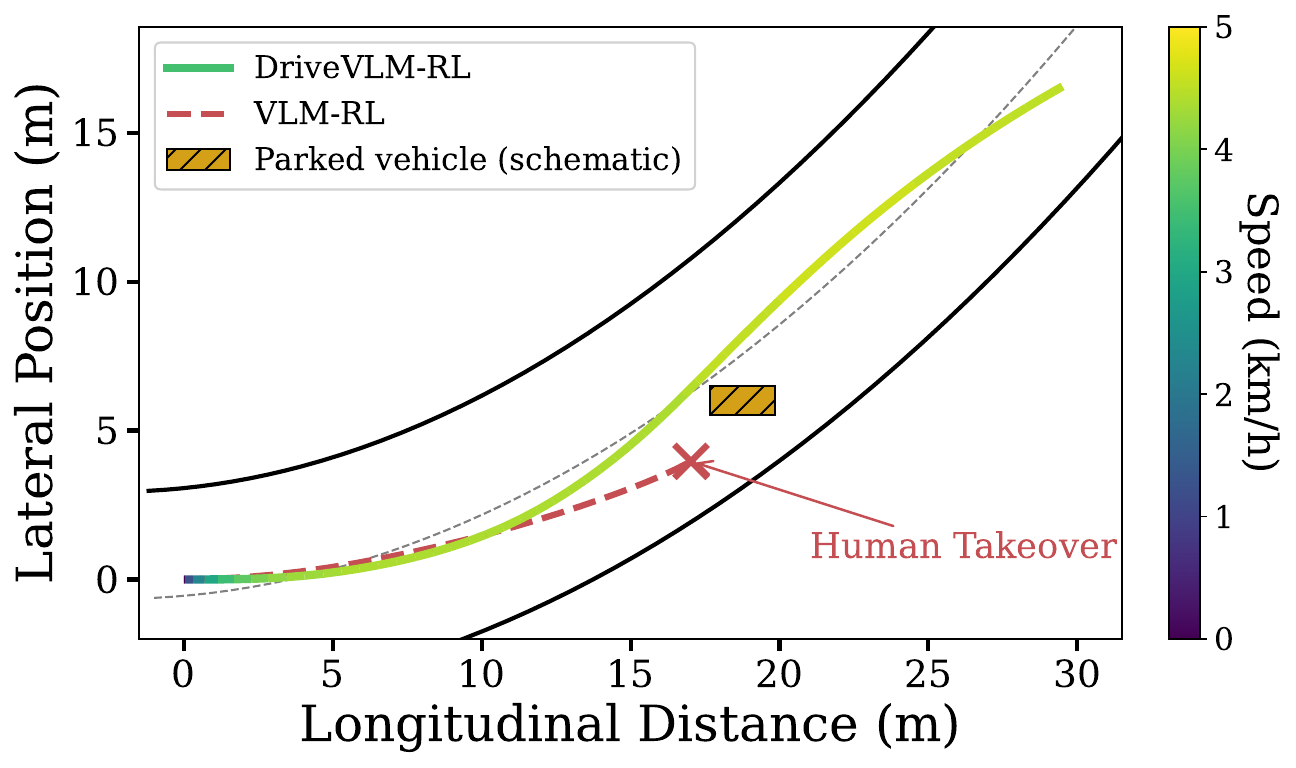}
        \caption{Vehicle trajectory}
        \label{fig:rw_traj}
    \end{subfigure}

    \caption{Representative S2 (static obstacle avoidance) episode. (a) First-person (top, policy input) and bird's-eye (bottom, shadow-mode reward) views; VLM-RL continues forward without a bypass and requires safety-driver intervention. (b) Reward signals: VLM-RL's reward stays flat (no dynamic pathway), whereas DriveVLM-RL's dips when its dynamic pathway activates on the obstacle. (c) Brake: VLM-RL decelerates but does not redirect; DriveVLM-RL decelerates in a controlled way before the bypass. (d) Speed: VLM-RL approaches the obstacle without bypass; DriveVLM-RL preserves forward progress. (e) Vehicle trajectory (on-vehicle logs, ego frame rotated by the initial heading; lane markings and parked-vehicle box are schematic): DriveVLM-RL bypasses the obstacle and continues, whereas VLM-RL stops short and is taken over ($\times$).}
    \label{fig:exp2_overall_results}
\end{figure*}

\subsubsection{Qualitative Case Studies}
\label{sec:closedloop_qual}
To further illustrate how Sim2Real-AD preserves 
algorithm-specific behavioral properties through 
the sim-to-real pipeline, we analyze 
representative behaviors in each scenario.

\textbf{S1: Routine Car-Following.}
Fig.~\ref{fig:case_carfollow} illustrates a 
representative DriveVLM-RL episode from two
camera views: the first-person view
used for policy control, and the bird's-eye view
used for shadow-mode reward analysis.
Throughout the episode, the ego vehicle follows 
a leading vehicle along a roadway without 
abrupt maneuvers or external disturbances. As shown in
Figs.~\ref{fig:exp1_overall_results}(b)--(c),
the reward signal stays stable and positive for both
methods throughout this scenario. The dynamic pathway
does not trigger, since the attentional gate detects
no safety-critical objects within the detection radius,
so DriveVLM-RL's reward is carried entirely by its
static pathway. This is expected: routine
car-following involves only a leading vehicle
and no semantically critical events such as
pedestrians, obstacles, or rule violations,
which are the situations the dynamic pathway
is designed to handle. The resulting control behaviors are compared
in
Figs.~\ref{fig:exp1_overall_results}(d)--(e). 
VLM-RL exhibits noticeable speed oscillations 
and abrupt steering corrections, whereas 
DriveVLM-RL achieves smoother speed profiles 
and more gradual steering adjustments. This 
difference demonstrates that Sim2Real-AD 
successfully transfers the behavioral 
smoothness advantage of DriveVLM-RL's 
dual-pathway design to the real vehicle: in 
routine scenarios, the dynamic pathway 
introduces no unnecessary semantic 
intervention, and the policy's learned 
conservative control style transfers intact 
through the GOB and PAM bridges.

\textbf{S2: Static Obstacle Avoidance.}
Fig.~\ref{fig:case_obstacle} shows a 
representative case from two camera views: 
the first-person view used as policy
input, and the bird's-eye view used for
shadow-mode reward analysis. A parked vehicle 
partially occupies the driving lane and is 
accompanied by traffic cones, indicating a 
temporary obstruction that should be bypassed 
rather than waited for. As shown in
Fig.~\ref{fig:exp2_overall_results}(b), 
VLM-RL's static CLIP-based reward provides 
no spatial disambiguation between a stopped 
vehicle and a temporary obstacle. Without a 
signal to initiate lateral bypass, the policy 
continues forward until the vehicle approaches 
the obstacle dangerously close, necessitating 
safety driver intervention. The brake pressure
profile in
Fig.~\ref{fig:exp2_overall_results}(c) and
speed profile in
Fig.~\ref{fig:exp2_overall_results}(d) are consistent with this: the policy decelerates but does not redirect, resulting in a straight-line approach 
toward the obstacle rather than a bypass
maneuver.
\begin{figure*}[!t]
    \centering
    \begin{subfigure}[b]{\linewidth}
        \centering
        \includegraphics[width=1\linewidth]{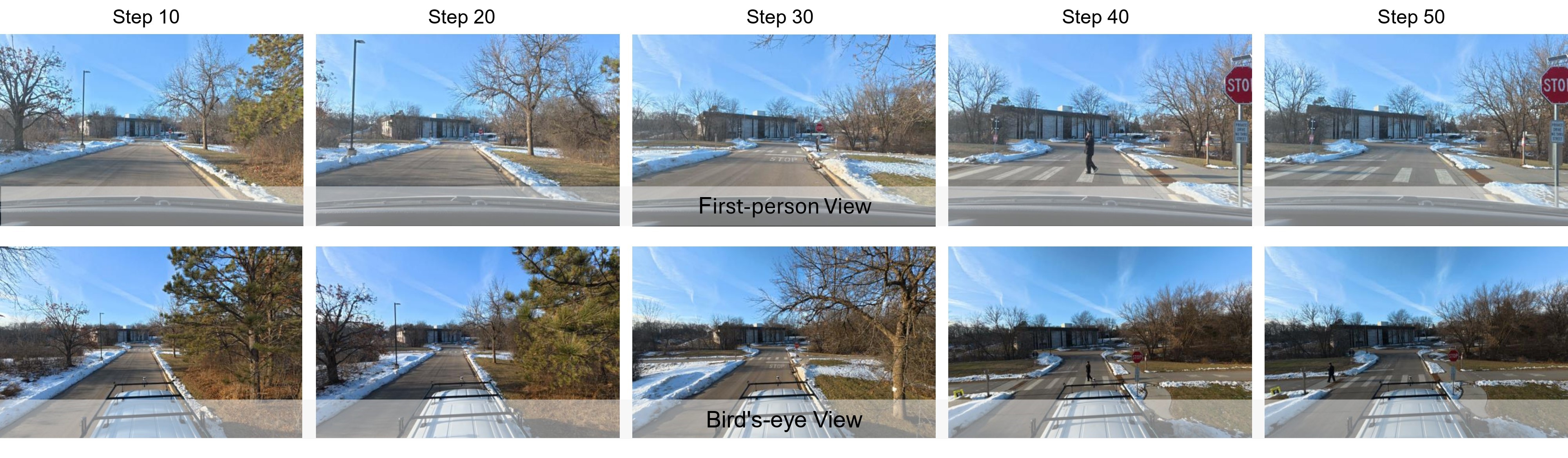}
        \caption{Representative stop-sign episode: first-person view (top) and bird's-eye view (bottom); a pedestrian is at the crosswalk ahead of the stop sign}
        \label{fig:case_stopsign}
    \end{subfigure}

    \vspace{0.25cm}

    \begin{subfigure}[b]{0.48\linewidth}
        \centering
        \includegraphics[width=\linewidth]{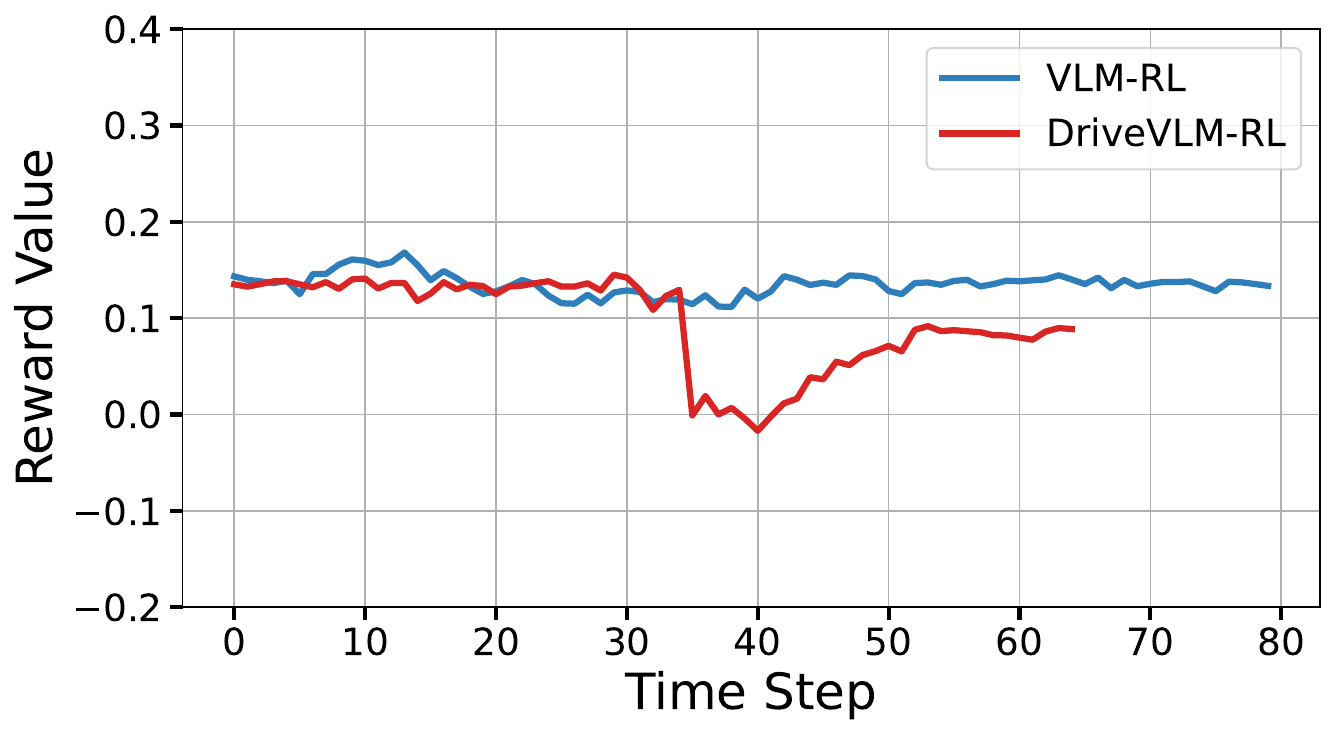}
        \caption{Reward signals}
        \label{fig:exp3_reward}
    \end{subfigure}\hfill
    \begin{subfigure}[b]{0.48\linewidth}
        \centering
        \includegraphics[width=\linewidth]{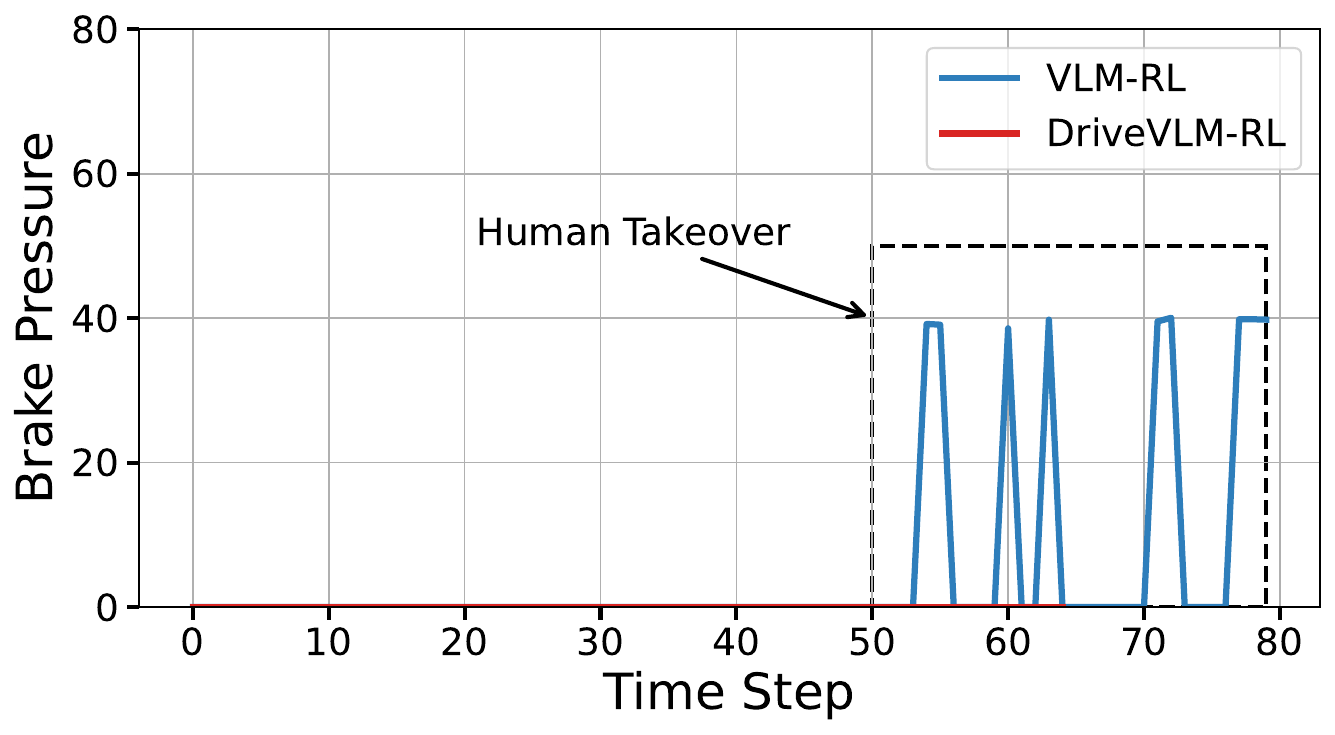}
        \caption{Brake pressure profiles}
        \label{fig:exp3_brake}
    \end{subfigure}\\[2mm]
    \begin{subfigure}[b]{0.48\linewidth}
        \centering
        \includegraphics[width=\linewidth]{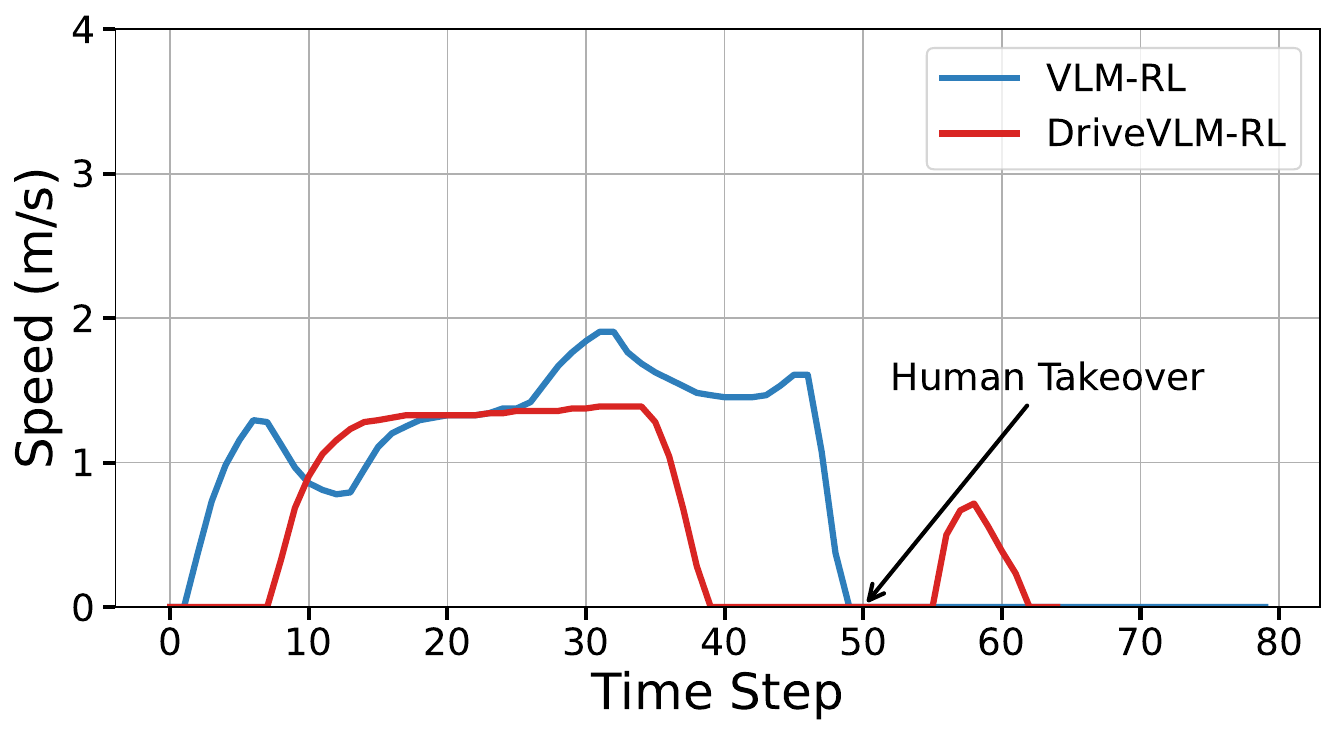}
        \caption{Vehicle speed profiles}
        \label{fig:exp3_speed}
    \end{subfigure}\hfill
    \begin{subfigure}[b]{0.48\linewidth}
        \centering
        \includegraphics[width=\linewidth]{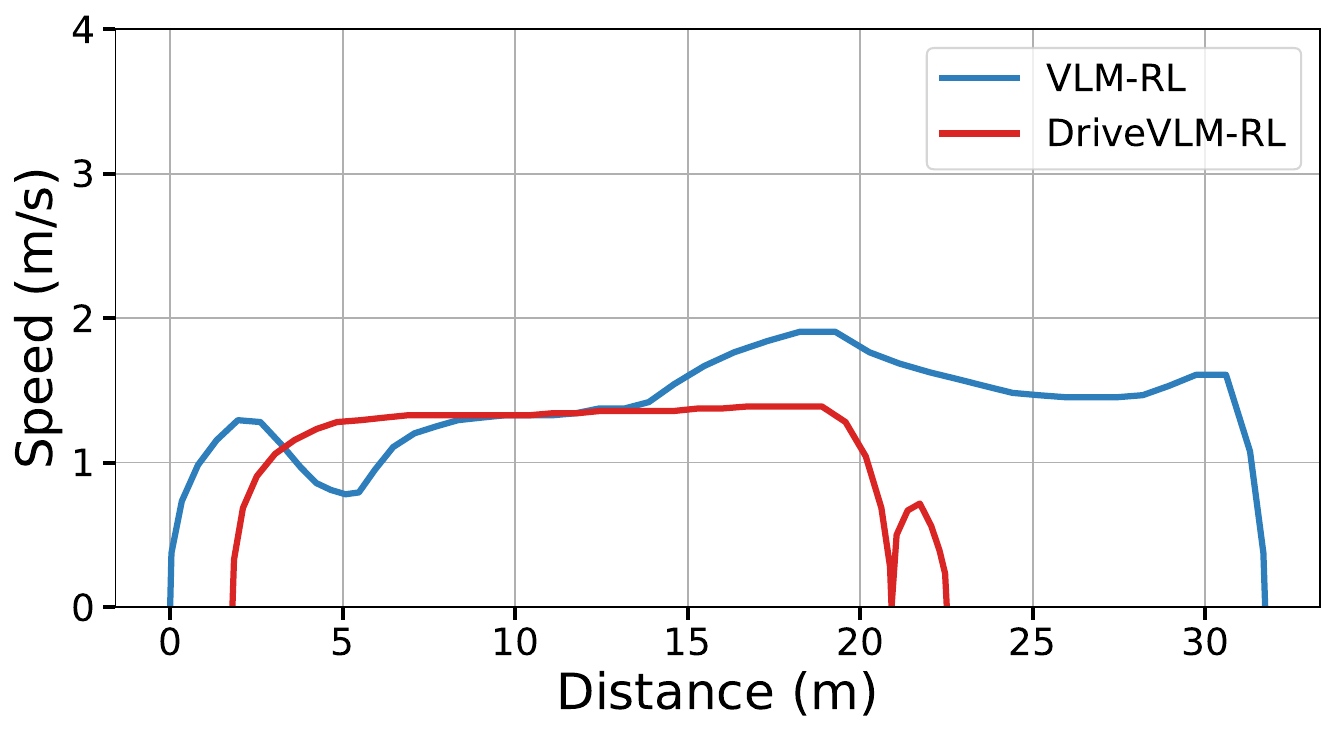}
        \caption{Distance and speed}
        \label{fig:rw_distspeed}
    \end{subfigure}

    \caption{Representative S3 (stop-sign interaction with a pedestrian) episode. (a) First-person (top, policy input) and bird's-eye (bottom, shadow-mode reward) views; VLM-RL fails to yield and requires safety-driver intervention, whereas DriveVLM-RL stops smoothly before the crosswalk. (b) Reward signals: VLM-RL's reward stays flat (no dynamic pathway), whereas DriveVLM-RL's turns negative when its dynamic pathway activates on the pedestrian. (c) Brake: VLM-RL brakes abruptly only at takeover; DriveVLM-RL decelerates smoothly and early. (d) Speed: VLM-RL holds speed until intervention; DriveVLM-RL decelerates to a complete stop before the crosswalk. (e) Distance and speed (on-vehicle logs): DriveVLM-RL increases clearance to the pedestrian while stopping before the crosswalk, whereas VLM-RL fails to yield.}
    \label{fig:exp3_overall}
\end{figure*}
In contrast, DriveVLM-RL successfully completes
the task. As Fig.~\ref{fig:exp2_overall_results}(b) shows, its reward dips when the obstacle enters the critical detection region: the attentional gate activates and triggers the dynamic pathway, whose LVLM selects the language goal ``An obstacle is on the road'' and introduces a negative reward term that penalizes continued forward approach and drives the bypass decision, whereas VLM-RL's reward stays flat. The resulting trajectory in Fig.~\ref{fig:exp2_overall_results}(e) shows DriveVLM-RL completing the maneuver past the obstacle and continuing without any safety-driver intervention, whereas VLM-RL stops short and is taken over. This case study illustrates how Sim2Real-AD transfers DriveVLM-RL's
context-dependent semantic reasoning to the 
real vehicle. The dynamic pathway's ability to 
detect safety-critical scenes and shift reward 
polarity, i.e., preserved intact through the GOB 
and PAM bridges, is precisely what enables 
correct disambiguation in this scenario, a 
capability that static reward paradigms cannot 
replicate.

\textbf{S3: Semantic-Critical Stop Sign 
Interaction.}
Fig.~\ref{fig:case_stopsign} shows a 
representative episode from two camera views: 
the first-person view used as policy
input, and the bird's-eye view for
shadow-mode reward analysis. As the ego vehicle 
approaches a stop-controlled intersection, a 
pedestrian enters the crosswalk region, creating 
a semantically critical situation that cannot 
be resolved through geometric cues alone. As shown in
Fig.~\ref{fig:exp3_overall}(b), VLM-RL's
static reward fails to encode the stop sign 
rule and pedestrian right-of-way, providing 
no signal to initiate early deceleration. The
vehicle continues forward without yielding,
approaching the pedestrian dangerously close and
necessitating safety-driver intervention. This 
is reflected in the brake pressure profile in
Fig.~\ref{fig:exp3_overall}(c) and speed
profile in
Fig.~\ref{fig:exp3_overall}(d), where abrupt
braking spikes coincide with safety driver
takeover.
In contrast, as shown in
Fig.~\ref{fig:exp3_overall}(b), DriveVLM-RL
selectively activates the dynamic pathway when 
the attentional gate detects the
safety-critical event. When the gate activates, the LVLM reasons over sequential visual observations and selects the description ``A pedestrian is crossing the road ahead,'' which drives the dynamic reward sharply negative and shifts the combined reward polarity to penalize continued forward motion. This
triggers an early, smooth deceleration, as illustrated by Fig.~\ref{fig:exp3_overall}(e), which
shows DriveVLM-RL maintaining increasing 
clearance distance to the pedestrian while 
progressively reducing speed to a complete 
stop before the crosswalk, without any
safety driver intervention. This case study illustrates how Sim2Real-AD transfers DriveVLM-RL's
multi-frame semantic reasoning capability to
the real vehicle. The dynamic pathway's ability 
to aggregate sequential observations, apply 
semantic filtering, and select among candidate 
language descriptions (all learned 
exclusively in CARLA simulation) translates 
directly into correct traffic rule compliance 
under real-world conditions through the GOB 
and PAM bridges. This is the most demanding 
validation of Sim2Real-AD's transfer fidelity: 
the policy must not only navigate correctly 
but interpret symbolic traffic semantics and 
reason about vulnerable road users, a 
capability that static reward paradigms cannot 
encode.

\subsubsection{Real-Time Feasibility}
\label{sec:realtime}

Any sim-to-real deployment framework must let the transferred policy meet real-time control requirements on onboard hardware.
Table~\ref{tab:latency} profiles the onboard compute pipeline of Sim2Real-AD. The measured compute latency of $26.8 \pm 1.5$~ms satisfies the 20~Hz control rate with a 23.2~ms per-cycle margin (46.4\% of the 50~ms budget remaining). GPU work, namely SegFormer-B0 segmentation (6.2~ms) and policy inference (2.0~ms), accounts for only about 8~ms; the dominant costs are CPU-side image preprocessing (8.4~ms) and IPM/BEV construction (10.1~ms), which are straightforward to optimize and leave substantial headroom.
\begin{table*}[!t]
\centering
\caption{Per-stage onboard compute latency over 300 frames; sensor capture and CAN-bus transmission are platform I/O and excluded. Control budget: 50~ms at 20~Hz.}
\label{tab:latency}
\renewcommand{\arraystretch}{1.25}
\small
\begin{tabular*}{\textwidth}{@{\extracolsep{\fill}}lcc@{}}
\toprule
Pipeline Stage & Latency (ms) 
  & Budget Used (\%) \\
\midrule
Image preprocessing (resize, normalize)
  & $8.4 \pm 0.4$ & 16.8 \\
Semantic segmentation (SegFormer-B0)
  & $6.2 \pm 0.3$ & 12.4 \\
IPM projection + BEV construction (GOB)
  & $10.1 \pm 0.9$ & 20.3 \\
Policy inference (SAC forward)
  & $2.0 \pm 0.2$ & 4.0 \\
PAM (bicycle model + PID)
  & $<0.1$ & $<0.2$ \\
\midrule
\textbf{Total (onboard compute)}
  & $\mathbf{26.8 \pm 1.5}$
  & \textbf{53.6} \\
\bottomrule
\end{tabular*}

\end{table*}
\begin{figure*}[!t]
    \centering
    \begin{subfigure}[b]{0.48\linewidth}
        \centering
        \includegraphics[width=\linewidth]
        {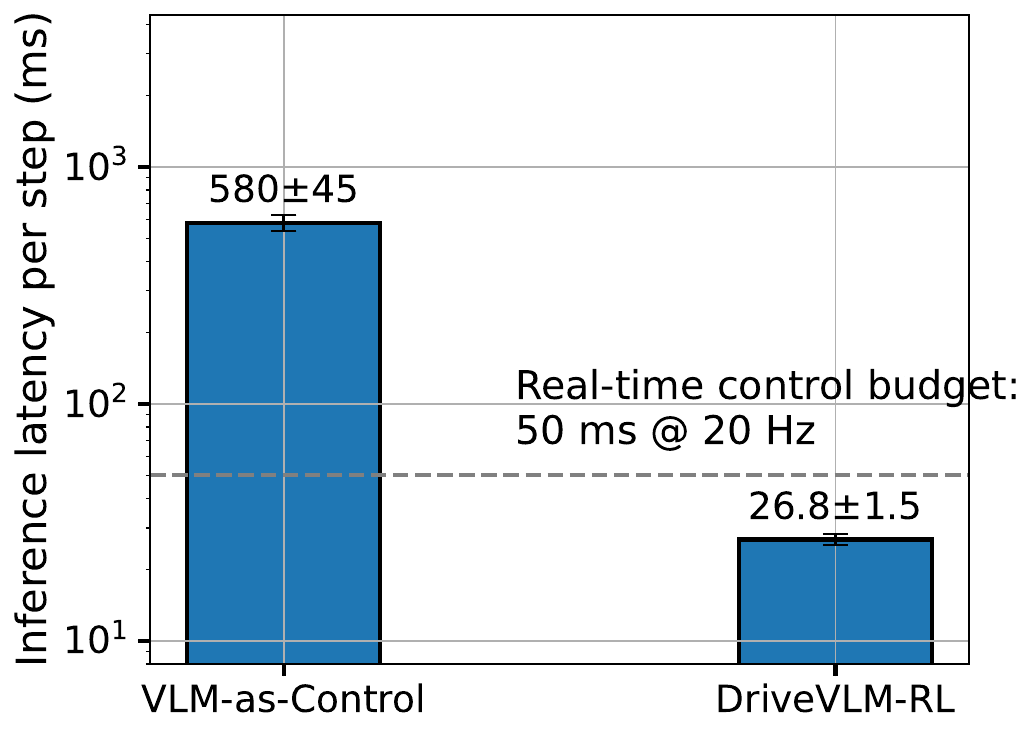}
        \caption{End-to-end inference latency
        comparison}
        \label{fig:latency}
    \end{subfigure}
    \hfill
    \begin{subfigure}[b]{0.48\linewidth}
        \centering
        \includegraphics[width=\linewidth]
        {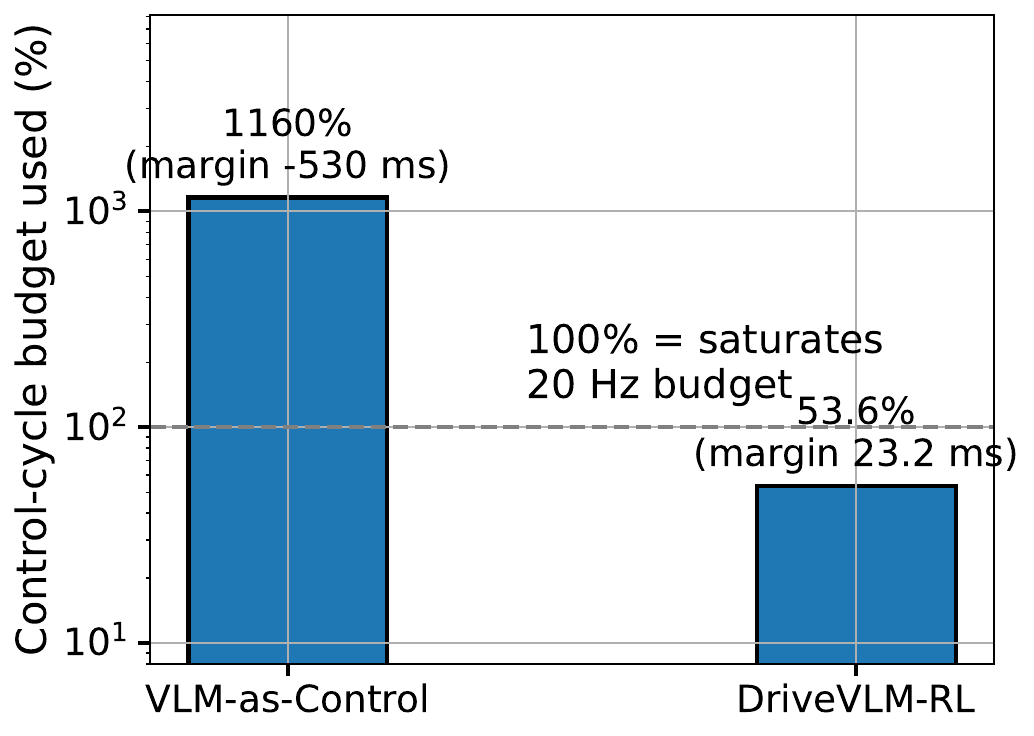}
        \caption{Control budget utilization}
        \label{fig:utilization}
    \end{subfigure}
    \caption{Real-time feasibility on the Ford E-Transit van onboard GPU. Sim2Real-AD satisfies the 20 Hz control budget with an average onboard compute latency of 26.8 ms, whereas direct VLM-as-control inference is far beyond the real-time limit.}
    \label{fig:realtime_performance}
\end{figure*}
To contextualize this result, we benchmark
against LeapAD \citep{mei2024continuously}, a representative
VLM-as-Control approach, comparing computational
efficiency rather than safety. Rather than deploy the
full LeapAD, we map camera observations to high-level
meta-actions (e.g., accelerate, decelerate, idle) using
Qwen3-VL, translated into low-level commands through a
predefined action library. Both pipelines run on the same
onboard GPU, and reported latencies cover end-to-end
perception and inference (excluding sensor acquisition
and CAN communication). As shown in
Fig.~\ref{fig:realtime_performance}(a), a 
single Qwen3-VL forward pass incurs 
$580 \pm 45$~ms per step, limiting the 
achievable control frequency to 1--2~Hz and 
consuming 1160\% of the 20~Hz budget, 
making real-time closed-loop operation 
infeasible. In contrast, Sim2Real-AD's 
deployed pipeline requires 53.6\% of 
the budget, as illustrated in 
Fig.~\ref{fig:realtime_performance}(b). The 
deployed system also occupies approximately 
0.4~GB of GPU memory for model weights and activations (GOB backbone + SAC policy), 
compared to 24.5~GB required to load
Qwen3-VL under FP16 precision. This order-of-magnitude smaller memory and compute footprint translates into lower onboard power draw, especially valuable on battery-electric platforms, where energy spent on computation trades directly against driving range. These results confirm that Sim2Real-AD's training-deployment decoupling, with all VLM and LVLM components confined to offline training, is not merely a design convenience but a practical necessity for real-time vehicle control on standard onboard hardware (Section~\ref{sec:overview}).

\section{Discussion and Limitations}
\label{sec:limitations}

We additionally evaluated cross-town generalization: trained on Town~2 and tested on the unseen Towns~1,~3,~4, and~5, both the original DriveVLM-RL policy and its Sim2Real-AD counterpart collapse to near-floor success (SR $0.00$--$0.10$, per-town differences within run-to-run noise). The bottleneck is the policy's exposure to a \emph{single} training town, not the observation- or action-space transfer introduced by Sim2Real-AD. This matches our central principle: the gap to a new environment separates into a \emph{sensing-and-dynamics} domain gap (real camera and actuators vs.\ simulator), which GOB and PAM close by re-projecting inputs and outputs onto the simulator-trained manifold, and an orthogonal \emph{task-and-geometry} gap (novel topology, traffic density, signals), which the transfer modules do not address. Cross-town evaluation isolates the latter and collapses, whereas the on-vehicle scenarios keep it small while the sensing-and-dynamics domain gap is large, precisely the regime Sim2Real-AD targets. The real-world success rates should be read in this light: the deployment scenarios use in-support road geometries and are deliberately simpler than the dense, signal-regulated 3000\,m simulation routes, so the high success rates are evidence of \emph{sensing-and-dynamics domain transfer}, not of generalization to geometrically novel or higher-complexity environments. Closing the latter would require diversifying the training distribution (multi-town or multi-route training, or domain randomization over topology and traffic density), which is complementary to the transfer mechanisms studied here and left to future work.

\section{Conclusion}
\label{sec:conclusion}

This paper addressed how to deploy simulation-trained, foundation-model-guided driving policies on physical vehicles. Its central principle is that the simulation-to-reality gap decomposes into a sensing-and-dynamics domain gap and an orthogonal task-and-geometry gap, the former closable without any real-world policy training. We realized this principle as \textbf{Sim2Real-AD}, a modular framework that bridges the observation (GOB) and action (PAM) spaces, stabilizes transfer with a two-phase curriculum (TPT), and runs in real time (RDP). It transfers a CARLA-trained policy to a full-scale Ford E-Transit van with only $\sim$30 minutes of calibration and no real-world RL data, backed by a formal transfer guarantee that bounds the deployment gap by three independently controllable error terms. Simulation experiments confirmed transfer across reward paradigms while preserving the semantically grounded reward's safety advantage and validating each module; the central real-world demonstration is zero-shot closed-loop deployment across three scenarios, as a proof-of-concept case study. The framework is reward-agnostic, validated across ChatScene-SAC, VLM-RL, and DriveVLM-RL. By deploying lightweight, energy-efficient driving policies zero-shot on electrified vehicles, Sim2Real-AD offers a practical step toward intelligent and sustainable transportation.

Several directions remain open for future
work. The 15~km/h speed limit is
safety-conservative; higher speeds require
tighter PID calibration and possibly
multi-camera sensing for longer-horizon BEV
coverage, reducing the PAM and GOB error terms.
The forward-only monocular setup limits GOB
quality at sharp turns and wide intersections,
which surround-view sensing would improve. GPS
route dependency could be removed via online map
matching or vision-based lane extraction.
Finally, incorporating a few real-world rollouts
into Phase~2 (analogous to Rapid Motor Adaptation
\citep{kumar2021rma}) could reduce the residual
observation-distribution gap and further close the
zero-shot performance ceiling. More broadly, this decomposition principle is not specific to driving: it offers a general, theoretically grounded route for moving simulation-trained, foundation-model-guided policies onto physical platforms across robotics and embodied control, where real-world data is costly or unsafe.

\section*{Acknowledgment}
This work was supported by the University of Wisconsin-Madison's Center for Connected and Automated Transportation (CCAT), a part of the larger CCAT consortium, a USDOT Region 5 University Transportation Center funded by the U.S. Department of Transportation, Award \#69A3552348305. The contents of this paper reflect the views of the authors, who are responsible for the facts and the accuracy of the data presented herein, and do not necessarily reflect the official views or policies of the sponsoring organization.

\clearpage
\appendix
\setcounter{lemma}{0}
\setcounter{theorem}{0}
\setcounter{corollary}{0}
\setcounter{proposition}{0}
\setcounter{remark}{0}
 
 \section{Pseudocode for Training and Deployment}
\label{app:algorithms}

This appendix provides complete pseudocode for the 
Two-Phase Progressive Training procedure 
(Algorithm~\ref{alg:tpt}) and the Real-Time 
Deployment Pipeline (Algorithm~\ref{alg:rdp}) of
Sim2Real-AD.

\begin{algorithm}[H]
\caption{Two-Phase Progressive Training (TPT)}
\label{alg:tpt}
\begin{algorithmic}[1]

\Require
  DriveVLM-RL reward oracle $\mathcal{R}(\cdot)$~\citep{huang2026drivevlmrl},
  GOB pipeline $\mathcal{G}$ (SegFormer-B0 + IPM, 
  calibrated with $\mathbf{K}, h, \alpha, \beta$),
  PAM limits $\kappa_{\max},\, v_{\max}$,
  steps $T_1 = 1\times10^6$, $T_2 = 5\times10^5$,
  SAC hyperparameters ($\lambda,\tau,B,\mathcal{D}$),
  reward interval $\Delta$, warmup $N_{\text{warmup}}$

\medskip
\State \textbf{// Phase~1: Action-Space Adaptation}
\State Initialize policy $\pi_\phi$, Q-functions 
       $Q_\theta$, target $\theta^{-} \leftarrow 
       \theta$,\; $N_{\text{ready}} \leftarrow 0$

\For{$t = 1, 2, \ldots, T_1$}

  \State Observe $o_t = (\hat{o}_t^{\mathrm{sim}},\, 
         o_t^{\text{cam}},\, s_t,\, w_t)$ from CARLA
  \State $(\hat{a}_{1,t},\, \hat{a}_{2,t}) \sim 
         \pi_\phi(\cdot \mid o_t)$
  \State Decode PAM:\;
         $\kappa_t \leftarrow 
         \hat{a}_{1,t}\,\kappa_{\max}$,\quad
         $v_t^d \leftarrow 
         \tfrac{\hat{a}_{2,t}+1}{2}\,v_{\max}$
  \State Execute $(\kappa_t, v_t^d)$ in CARLA via 
         bicycle-model steering and simulated PID 
         speed tracking; observe $o_{t+1}$
  \State Store $(o_t,\, \hat{a}_t,\, o_{t+1},\;
         r_t \leftarrow \texttt{NaN},\;
         \texttt{ready} \leftarrow 0)$ in $\mathcal{D}$

  \If{$t \bmod \Delta = 0$}
    \State Sample mini-batch from $\mathcal{D}$ 
           where $\texttt{ready}=0$
    \State Annotate each transition with 
           $R_{\text{final}} \leftarrow 
           \mathcal{R}(o_i, o_{i+1})$;\;
           set $\texttt{ready} \leftarrow 1$,\;
           $N_{\text{ready}} \mathrel{+}= 1$
  \EndIf

  \If{$N_{\text{ready}} \geq N_{\text{warmup}}$}
    \State Sample mini-batch 
           $\{(o_i, \hat{a}_i, r_i, o_{i+1})\}$
           from $\mathcal{D}$ with $\texttt{ready}=1$
    \State SAC update: minimize $J_Q(\theta)$ for $Q_\theta$, 
           maximize Eq.~\eqref{eq:rl_objective} for $\pi_\phi$
    \State Soft-update targets:\;
           $\theta^{-} \leftarrow 
           (1-\tau)\,\theta^{-}+\tau\,\theta$
  \EndIf

\EndFor
\State $\theta_1 \leftarrow \phi$ \Comment{checkpoint}

\medskip
\State \textbf{// Phase~2: Observation-Space Adaptation}
\State $\phi \leftarrow \theta_1$,\;
       $\theta^{-} \leftarrow \theta$,\;
       $N_{\text{ready}} \leftarrow 0$
       \Comment{Warm-start}

\For{$t = 1, 2, \ldots, T_2$}

  \State Capture rendered front-view image $I_t$ 
         from CARLA camera
  \State Apply GOB:\;
         $S_t \leftarrow f_{\text{seg}}(I_t)$,\;
         $\tilde{S}_t \leftarrow 
         f_{\text{ipm}}(S_t;\,
         \mathbf{K},h,\alpha,\beta)$,\;
         $\hat{o}_t^{\mathrm{sim}} \leftarrow 
         \mathrm{Encode}(\tilde{S}_t)$
  \State $o_t \leftarrow 
         (\hat{o}_t^{\mathrm{sim}},\, 
         o_t^{\text{cam}},\, s_t,\, w_t)$
         \Comment{GOB-BEV $\hat{o}_t^{\mathrm{sim}}$ 
         replaces GT-BEV; all else unchanged}
  \State Decode PAM, execute, store as in Phase~1
  \If{$t \bmod \Delta = 0$}
    \State Annotate batch using $\mathcal{R}(\cdot)$;\;
           set $\texttt{ready} \leftarrow 1$
  \EndIf
  \If{$N_{\text{ready}} \geq N_{\text{warmup}}$}
    \State Update $Q_\theta$, $\pi_\phi$ via SAC;
           soft-update $\theta^{-}$
  \EndIf

\EndFor

\State \Return trained policy 
       $\pi_{\theta_2} \leftarrow \pi_\phi$

\end{algorithmic}
\end{algorithm}
 
\begin{algorithm}[H]
\caption{Sim2Real-AD Real-Time Deployment Pipeline 
(RDP)}
\label{alg:rdp}
\begin{algorithmic}[1]

\Require
  Trained policy $\pi_{\theta_2}$,
  GOB pipeline $\mathcal{G}$ with calibrated
  $(\mathbf{K}, h, \alpha, \beta)$,
  PAM parameters
  $\Theta = \{L,\, \delta_{\max},\, K_p,\, K_i,\, 
  K_d\}$,
  PAM limits $\kappa_{\max},\, v_{\max}$,
  pre-recorded GPS route $\mathcal{W}_{\text{route}}$,
  safety bounds
  $(v_{\text{limit}},\, d_{\text{limit}},\,
   \Delta\delta_{\max},\, r_{\text{safe}})$,
  anti-windup clip $e_{\max}$,
  control period $\Delta t = 1/20$\,s

\medskip
\State Initialize:\;
       $e_{\text{int}} \leftarrow 0$,\;
       $e_{\text{prev}} \leftarrow 0$,\;
       $\delta_{\text{prev}} \leftarrow 0$

\While{\textbf{not} stop-signal received}

  \State \textbf{// Stage~1: Perception} \hfill
         \textit{($\approx$14\,ms)}
  \State $I_t \leftarrow \mathrm{Camera.capture}()$
  \State $S_t \leftarrow f_{\text{seg}}(I_t)$
  \State $\hat{o}_t^{\mathrm{sim}} \leftarrow 
         \mathrm{Encode}(f_{\text{ipm}}(S_t;\,
         \mathbf{K},h,\alpha,\beta))$

  \State \textbf{// Stage~2: Route and State} \hfill
         \textit{($<$1\,ms)}
  \State $w_t \leftarrow \mathrm{MatchRoute}(
         \mathrm{GPS}_t,\, \mathcal{W}_{\text{route}})$
         \Comment{Nearest-point match to
         vehicle-frame waypoints}
  \State $(v_t,\, \delta_t^{\text{cur}},\, \tau_t)
         \leftarrow \mathrm{CAN.read}()$
  \State $s_t \leftarrow
         (v_t / v_{\max},\;
          \delta_t^{\text{cur}} / \delta_{\max},\;
          \tau_t / 100)$
         \Comment{Normalized ego-state}

  \State \textbf{// Stage~3: Policy Inference} \hfill
         \textit{($\approx$2\,ms)}
  \State $(\hat{a}_{1,t},\, \hat{a}_{2,t})
         \leftarrow \pi_{\theta_2}(
         \hat{o}_t^{\mathrm{sim}},\, s_t,\, w_t)$

  \State \textbf{// Stage~4: PAM} \hfill
         \textit{($<$1\,ms)}
  \State $\kappa_t \leftarrow 
         \hat{a}_{1,t}\,\kappa_{\max}$,\quad
         $v_t^d \leftarrow 
         \tfrac{\hat{a}_{2,t}+1}{2}\,v_{\max}$
  \State $\delta_t^{\text{des}} \leftarrow 
         \arctan(L\,\kappa_t)$,\quad
         $u_t^\delta \leftarrow 
         \delta_t^{\text{des}} / \delta_{\max}$
  \State $e_t \leftarrow v_t^d - v_t$
  \State $e_{\text{int}} \leftarrow
         \mathrm{clip}(e_{\text{int}} + 
         e_t\,\Delta t,\;
         {-e_{\max}},\; e_{\max})$
  \State $u_t^v \leftarrow
         K_p\,e_t + K_i\,e_{\text{int}} +
         K_d\,(e_t - e_{\text{prev}})/\Delta t$
  \State $e_{\text{prev}} \leftarrow e_t$

  \State \textbf{// Stage~5: Safety Layer}
  \State $u_t^v \leftarrow 
         \min(u_t^v,\; v_{\text{limit}})$
         \Comment{Speed cap}
  \State $u_t^\delta \leftarrow 
         \mathrm{clip}(u_t^\delta,\;
         \delta_{\text{prev}} - \Delta\delta_{\max},\;
         \delta_{\text{prev}} + \Delta\delta_{\max})$
         \Comment{Steering-rate limit}
  \If{$\;\mathrm{isNaN}(u_t^\delta)$
      $\vee\; \mathrm{isNaN}(u_t^v)$
      $\vee\; \mathrm{LaneDev}(
      \hat{o}_t^{\mathrm{sim}}) > d_{\text{limit}}$
      $\vee\; \mathrm{ObstacleIn}(r_{\text{safe}})$
      $\vee\; \mathrm{EStop}()$
      $\vee\; \mathrm{OutsideGeofence}()$}
    \State $(u_t^\delta,\, u_t^v) \leftarrow 
           (0,\; -1)$
           \Comment{Emergency brake: zero steer, 
           full brake}
  \EndIf
  \If{$\mathrm{DriverTakeover}()$}
    \State \textbf{break}
           \Comment{Safety driver override: 
           release control}
  \EndIf

  \State \textbf{// Stage~6: Actuation} \hfill
         \textit{($<$1\,ms)}
  \State $\mathrm{CAN.send}(u_t^\delta,\; u_t^v)$
  \State $\delta_{\text{prev}} \leftarrow u_t^\delta$
  \State $\mathrm{sleep}(\Delta t - t_{\text{elapsed}})$
         \Comment{Maintain 20\,Hz loop}

\EndWhile

\end{algorithmic}
\end{algorithm}
 
\section{Theoretical Analysis of Sim2Real-AD}
\label{app:theory}

This appendix provides formal statements and proofs 
for the theoretical claims in 
Section~\ref{sec:framework}. We proceed in five 
steps: Assumptions~(\ref{app:assumptions}), GOB 
perceptual bound~(Proposition~\ref{prop:gob_bound}), 
PAM tracking 
bound~(Proposition~\ref{prop:pam_tracking}), TPT 
distribution-shift bound~(Proposition~\ref{thm:tpt_bound}), 
and the main zero-shot transfer 
guarantee~(Theorem~\ref{thm:main_formal}). All 
results are stated with respect to the POMDP 
formulation in Section~\ref{sec:preliminaries}, with 
reward $r_t = \mathcal{F}(r_t^{\mathrm{task}},\, 
r_t^{\mathrm{sem}},\, s_t)$ as defined in 
Eq.~\eqref{eq:reward_combined}.

\subsection{Notation and Assumptions}
\label{app:assumptions}

\begin{assumption}[Lipschitz Policy]
\label{ass:lipschitz}
The policy $\pi_\theta^{\text{sim}}$ is 
$L_\pi$-Lipschitz with respect to its BEV input: 
for any two BEV tensors 
$\hat{o}, \hat{o}' \in 
\hat{\mathcal{O}}^{\mathrm{sim}}$,
\begin{equation}
\bigl\|
  \pi_\theta^{\text{sim}}(\hat{o}) - 
  \pi_\theta^{\text{sim}}(\hat{o}')
\bigr\|_2
\;\leq\;
L_\pi \bigl\| \hat{o} - \hat{o}' \bigr\|_1.
\end{equation}
\end{assumption}

\begin{remark}
Neural networks with bounded weights and smooth 
activation functions (ELU, Tanh) are Lipschitz. In 
practice, $L_\pi$ can be estimated via spectral 
normalization or empirical Jacobian 
bounds~\citep{miyato2018spectral}. The SAC-trained 
policy uses gradient clipping and weight decay, which 
empirically constrain the Lipschitz constant.
\end{remark}

\begin{assumption}[Lipschitz Reward]
\label{ass:lipschitz_reward}
The reward $r_t = \mathcal{F}(r_t^{\mathrm{task}},\, 
r_t^{\mathrm{sem}},\, s_t)$ from 
Eq.~\eqref{eq:reward_combined} is $L_r$-Lipschitz 
with respect to vehicle state and satisfies 
$|r_t| \leq R_{\max}$ for all $t$.
\end{assumption}

\begin{remark}
The CLIP-based~\citep{radford2021learning} semantic 
term $r_t^{\mathrm{sem}}$ is 1-Lipschitz in the 
visual embedding since cosine similarity is 
1-Lipschitz on the unit sphere. The task reward term 
$r_t^{\mathrm{task}}$ comprises differentiable, 
bounded driving signals (speed, lane deviation, 
collision indicator) and is Lipschitz by 
construction. The composite reward 
Eq.~\eqref{eq:reward_combined} is therefore 
Lipschitz with bounded constant $L_r$.
\end{remark}

\begin{assumption}[Bounded Segmentation Error]
\label{ass:seg}
The SegFormer-B0 model produces segmentation maps 
such that the expected $L_1$ distance between sim 
and real BEV tensors satisfies:
\begin{equation}
\mathbb{E}\!\left[\,
\bigl\| \hat{o}_t^{\mathrm{sim}} - 
\hat{o}_t^{\mathrm{real}} \bigr\|_1
\right]
\;\leq\;
\epsilon_{\mathrm{seg}},
\end{equation}
where the expectation is over the joint randomness 
in lighting, texture, sensor noise, and scene 
sampling.
\end{assumption}

\begin{remark}
SegFormer-B0 achieves 37.4~mIoU on ADE20K under 
distribution shift~\citep{xie2021segformer}. In our 
setting, paired rollouts can be collected in 
CARLA~\citep{dosovitskiy2017carla} by rendering the 
same scenario simultaneously from a privileged 
semantic view and a camera view, enabling direct 
empirical measurement of $\epsilon_{\mathrm{seg}}$.
\end{remark}

\begin{assumption}[Bounded Path Deviation]
\label{ass:pid}
The executed path deviates from the intended path 
by at most $\epsilon_{\mathrm{pid}} > 0$ in 
curvature at each timestep:
\begin{equation}
\bigl|\kappa_t^{\text{executed}} - \kappa_t\bigr| 
\leq \epsilon_{\mathrm{pid}},
\end{equation}
where $\kappa_t^{\text{executed}}$ is the curvature 
realized by the physical vehicle and $\kappa_t$ is 
the curvature commanded by PAM. This bound 
encompasses both direct curvature tracking error 
($\epsilon_\kappa$) and the effective curvature 
deviation induced by speed tracking error, and can 
be measured directly during the calibration 
procedure of Section~\ref{sec:pam}.
\end{assumption}

\begin{remark}
The path deviation bound $\epsilon_{\mathrm{pid}}$ 
is tunable through PID gain selection and is not a 
fixed system constant. Tighter gains reduce 
$\epsilon_{\mathrm{pid}}$ at the cost of increased 
control effort and potential actuation saturation; 
the calibration protocol in Section~\ref{sec:pam} 
identifies gains that keep $\epsilon_{\mathrm{pid}}$ 
small while maintaining stable closed-loop behavior 
on the Ford E-Transit van.
\end{remark}

\subsection{GOB Perceptual Bound}
\label{app:gob}

\begin{proposition}[GOB Observation Error Bound]
\label{prop:gob_bound}
Under Assumptions~\ref{ass:lipschitz} 
and~\ref{ass:seg}, the expected $L_2$ deviation in 
policy output caused by the perceptual gap satisfies:
\begin{equation}
\mathbb{E}\!\left[\,
\bigl\|
  \pi_\theta^{\text{sim}}(\hat{o}_t^{\mathrm{sim}})
  - \pi_\theta^{\text{sim}}(\hat{o}_t^{\mathrm{real}})
\bigr\|_2
\right]
\;\leq\;
L_\pi \cdot \epsilon_{\mathrm{seg}}.
\end{equation}
\end{proposition}

\begin{proof}
Applying the Lipschitz condition 
(Assumption~\ref{ass:lipschitz}) and then the 
expectation bound (Assumption~\ref{ass:seg}):
\begin{equation}
\begin{aligned}
\mathbb{E}\!\left[\,
\bigl\|
  \pi_\theta^{\text{sim}}(\hat{o}_t^{\mathrm{sim}})
  - \pi_\theta^{\text{sim}}(\hat{o}_t^{\mathrm{real}})
\bigr\|_2
\right]
&\;\leq\;
L_\pi \cdot \mathbb{E}\!\left[\,
\bigl\| \hat{o}_t^{\mathrm{sim}} -
\hat{o}_t^{\mathrm{real}} \bigr\|_1
\right] \\
&\;\leq\;
L_\pi \cdot \epsilon_{\mathrm{seg}}.
\end{aligned}
\end{equation}
\end{proof}

\begin{remark}
Proposition~\ref{prop:gob_bound} shows that the GOB 
contribution to sim-to-real performance degradation 
scales linearly with $\epsilon_{\mathrm{seg}}$. 
Replacing SegFormer-B0 with a stronger segmentation 
model directly reduces $\epsilon_{\mathrm{seg}}$ and 
tightens the bound without requiring any policy 
retraining.
\end{remark}

\subsection{PAM Tracking Error Bound}
\label{app:pam}

\begin{proposition}[PAM Lateral Position Tracking 
Bound]
\label{prop:pam_tracking}
Under Assumption~\ref{ass:pid} and a kinematic 
bicycle model with wheelbase $L$, the lateral 
position error accumulated over a control horizon 
of $T$ steps satisfies:
\begin{equation}
\bigl| y_T^{\text{real}} - y_T^{\text{intended}} 
\bigr|
\;\leq\;
\frac{v_{\max}^2 T^2}{2}\,\epsilon_{\mathrm{pid}},
\label{eq:pam_bound}
\end{equation}
where $v_{\max}$ is the maximum vehicle speed.
\end{proposition}

\begin{proof}
Under the bicycle model, the lateral dynamics 
satisfy $\dot{y} \approx v\psi$ for small heading 
angle $\psi$, and the heading rate satisfies 
$\dot{\psi} = v\kappa$. A curvature tracking error 
$\Delta\kappa_t = |\kappa_t^{\text{executed}} - 
\kappa_t| \leq \epsilon_{\mathrm{pid}}$ induces a 
heading angle error:
\begin{equation}
|\Delta\psi(t)| \leq \int_0^t 
v(\tau)\,|\Delta\kappa(\tau)|\,d\tau
\;\leq\; v_{\max}\,\epsilon_{\mathrm{pid}}\,t.
\end{equation}
Integrating the resulting lateral position error:
\begin{equation}
\begin{aligned}
\bigl| y_T^{\text{real}} - y_T^{\text{intended}}
\bigr|
&\leq \int_0^T v(t)\,|\Delta\psi(t)|\,dt \\
&\;\leq\; v_{\max} \int_0^T
v_{\max}\,\epsilon_{\mathrm{pid}}\,t\,dt \\
&= \frac{v_{\max}^2 T^2}{2}\,\epsilon_{\mathrm{pid}}.
\end{aligned}
\end{equation}
\end{proof}

\begin{remark}
The lateral error bound grows quadratically with 
the horizon $T$ and quadratically with $v_{\max}$. 
This motivates both the 20~Hz closed-loop control 
frequency (shorter $\Delta t$ reduces the effective 
horizon over which errors accumulate) and the 
conservative 15~km/h speed cap during initial 
testing (at $v_{\max} \approx 4.2$~m/s the 
quadratic term is substantially smaller than at 
highway speeds).
\end{remark}

\subsection{TPT Distribution Shift Bound}
\label{app:tpt}

\begin{proposition}[TPT Performance Bound]
\label{thm:tpt_bound}
Let $J(\pi, \mathcal{O})$ denote the expected 
discounted return Eq.~\eqref{eq:rl_objective} of 
policy $\pi$ under observation distribution 
$\mathcal{O}$. Under 
Assumption~\ref{ass:lipschitz_reward}, for any 
policy $\pi$ and any two observation distributions 
$\mathcal{O}_2$ and $\mathcal{O}^{\text{real}}$:
\begin{equation}
J\!\left(\pi,\; \mathcal{O}^{\text{real}}\right)
\;\geq\;
J\!\left(\pi,\; \mathcal{O}_2\right)
-
\frac{2\,R_{\max}\,
  d_{\mathrm{TV}}\!\left(\mathcal{O}_2,\, 
  \mathcal{O}^{\text{real}}\right)
}{(1-\gamma)^2}.
\label{eq:tpt_bound}
\end{equation}
\end{proposition}

\begin{proof}
We use a Bellman-recursion sensitivity argument 
in two steps.

\medskip
\noindent\textbf{Step~1: Per-state value 
difference.}
For any state $s$, define the value function
$V^{\pi,\mathcal{O}}(s) = \mathbb{E}_\pi[
\sum_{t\geq 0} \gamma^t r_t \mid s_0 = s, 
\mathcal{O}]$. Since $|r_t| \leq R_{\max}$, the 
Bellman operator is a $\gamma$-contraction and
$|V^{\pi,\mathcal{O}}(s)| \leq R_{\max}/(1-\gamma)$.

By the standard total variation inequality, for 
any bounded measurable function $f$ with 
$\|f\|_\infty \leq M$ and distributions $P$, $Q$,
$|\mathbb{E}_P[f] - \mathbb{E}_Q[f]| \leq 
2M \cdot d_{\mathrm{TV}}(P,Q)$~\citep{levin2017markov}, the 
single-step reward difference at state $s$ 
satisfies:
\begin{equation}
\begin{aligned}
&\bigl| r(s,\pi,\mathcal{O}^{\text{real}}) -
r(s,\pi,\mathcal{O}_2) \bigr| \\
&\quad\;\leq\; 2\,R_{\max} \cdot
d_{\mathrm{TV}}\!\left(\mathcal{O}^{\text{real}}(s),
\, \mathcal{O}_2(s)\right).
\end{aligned}
\label{eq:single_step_bound}
\end{equation}
Applying this bound at each step of the Bellman 
recursion and summing the geometric series:
\begin{equation}
\begin{aligned}
&\bigl| V^{\pi,\mathcal{O}^{\text{real}}}(s) -
V^{\pi,\mathcal{O}_2}(s) \bigr| \\
&\quad\;\leq\;
\frac{2\,R_{\max}}{1-\gamma} \sup_{s'} \,
d_{\mathrm{TV}}\!\left(\mathcal{O}^{\text{real}}(s'),
\,\mathcal{O}_2(s')\right).
\end{aligned}
\label{eq:value_diff}
\end{equation}

\medskip
\noindent\textbf{Step~2: Return difference via 
state visitation.}
The performance difference 
$J(\pi,\mathcal{O}^{\text{real}}) -
J(\pi,\mathcal{O}_2)$ can be written in terms of 
the discounted state visitation measure 
$d^{\pi,\mathcal{O}}$ of each environment:
\begin{equation}
\begin{aligned}
&J(\pi,\mathcal{O}^{\text{real}}) -
J(\pi,\mathcal{O}_2) \\
&\quad =
(1-\gamma)\,\mathbb{E}_{s \sim
d^{\pi,\mathcal{O}^{\text{real}}}}\!
\left[V^{\pi,\mathcal{O}^{\text{real}}}(s)
     - V^{\pi,\mathcal{O}_2}(s)\right] \\
&\qquad + \Delta_{\text{vis}},
\end{aligned}
\end{equation}
where $\Delta_{\text{vis}}$ accounts for the 
difference in state visitation between the two 
environments. Since
$|V^{\pi,\mathcal{O}}(s)| \leq R_{\max}/(1-\gamma)$, 
the visitation mismatch contributes at most
\[
\frac{2R_{\max}}{1-\gamma}\,d_{\text{TV}}\!\left(
d^{\pi,\mathcal{O}^{\text{real}}},\,
d^{\pi,\mathcal{O}_2}\right).
\]
The total variation distance between state
visitation measures is bounded by the per-step
observation TV via the simulation
lemma~\citep{kakade2002approximately}:
\[
d_{\text{TV}}\!\left(d^{\pi,\mathcal{O}^{\text{real}}},\,
d^{\pi,\mathcal{O}_2}\right) \leq
\frac{1}{1-\gamma}\sup_{s'}
d_{\text{TV}}\!\left(\mathcal{O}^{\text{real}}(s'),\,
\mathcal{O}_2(s')\right).
\]
Combining with Eq.~\eqref{eq:value_diff} and
applying
\[
\sup_{s'} d_{\text{TV}}\!\left(\mathcal{O}^{\text{real}}
(s'),\, \mathcal{O}_2(s')\right)
\leq d_{\text{TV}}\!\left(\mathcal{O}^{\text{real}},\,
\mathcal{O}_2\right)
\]
(TV is non-increasing under marginalization):
\begin{equation}
J(\pi,\mathcal{O}^{\text{real}}) - 
J(\pi,\mathcal{O}_2)
\;\geq\;
-\frac{2\,R_{\max}}{(1-\gamma)^2}\,
d_{\mathrm{TV}}\!\left(\mathcal{O}_2,\, 
\mathcal{O}^{\text{real}}\right).
\end{equation}
Rearranging yields Eq.~\eqref{eq:tpt_bound}.
\end{proof}

\begin{remark}
The $(1-\gamma)^{-2}$ factor has a clear two-factor 
interpretation: one $(1-\gamma)^{-1}$ from summing 
the per-step reward difference over the discounted 
horizon (Step~1), and a second $(1-\gamma)^{-1}$ 
from the mismatch in state visitation measures 
between the two environments (Step~2). Note that 
$L_\pi$ does not appear here: the Lipschitz policy 
constant connects BEV distance to action deviation 
(Proposition~\ref{prop:gob_bound}), but is not 
needed in the Bellman-recursion sensitivity 
argument. This also motivates Phase~2 of TPT: 
reducing $d_{\mathrm{TV}}(\mathcal{O}_2, 
\mathcal{O}^{\text{real}})$ by training on 
IPM-generated observations directly tightens 
the bound.
\end{remark}

\subsection{Main Zero-Shot Transfer Guarantee}
\label{app:main}

\begin{theorem}[Zero-Shot Transfer Guarantee]
\label{thm:main_formal}
Under Assumptions~\ref{ass:lipschitz}--\ref{ass:pid}, 
the expected cumulative reward of the Sim2Real-AD 
composed policy 
$\pi^{\text{real}} = \mathcal{M} \circ 
\pi_\theta^{\text{sim}} \circ \mathcal{G}$
(abbreviated; see Eq.~\eqref{eq:composed_policy} 
for the full definition including $s_t$ and $w_t$)
on the real vehicle satisfies:
\begin{equation}
\begin{aligned}
\mathbb{E}\!\left[\sum_{t=0}^{T} \gamma^t
r_t^{\text{real}}\right]
\;\geq\;{}&
\mathbb{E}\!\left[\sum_{t=0}^{T} \gamma^t
r_t^{\text{sim}}\right] \\
&-
\underbrace{\frac{L_r L_\pi\,
\epsilon_{\mathrm{seg}}}{1-\gamma}}_{\text{GOB term}}
-
\underbrace{\frac{L_r\,v_{\max}^2 T^2\,
\epsilon_{\mathrm{pid}}}{2(1-\gamma)}}_{\text{PAM term}} \\
&-
\underbrace{\frac{2\,R_{\max}\,
  d_{\mathrm{TV}}\!\left(\mathcal{O}_2,\,
  \mathcal{O}^{\text{real}}\right)
}{(1-\gamma)^2}}_{\text{TPT term}},
\end{aligned}
\label{eq:main_formal}
\end{equation}
where $r_t^{\text{sim}} = \mathcal{F}(
r_t^{\mathrm{task}},\, r_t^{\mathrm{sem}},\, s_t)$
is the VLM-guided reward 
Eq.~\eqref{eq:reward_combined} and $L_r$ is its 
Lipschitz constant.
\end{theorem}

\begin{remark}[Correspondence with Simplified 
Theorem]
\label{rem:C_values}
The constants $C_1$, $C_2$, $C_3$ in 
Theorem~\ref{thm:main_informal}
(Section~\ref{sec:theory_informal}) correspond to 
the explicit expressions in 
Eq.~\eqref{eq:main_formal} as follows:
\begin{equation}
C_1 = \frac{L_r L_\pi}{1-\gamma},
\qquad
C_2 = \frac{L_r\,v_{\max}^2 T^2}{2(1-\gamma)},
\qquad
C_3 = 2\,R_{\max}.
\end{equation}
Note that $C_3$ does not depend on $L_\pi$: the 
TPT bound (Theorem~\ref{thm:tpt_bound}) is 
controlled by reward magnitude $R_{\max}$ alone, 
not by how sensitively the policy maps BEV inputs 
to actions. All three constants decrease as the 
reward becomes less sensitive to state perturbations 
(smaller $L_r$), the policy becomes more 
Lipschitz-regular (smaller $L_\pi$, relevant for 
$C_1$), or the planning horizon shortens (smaller 
$T$, relevant for $C_2$).
\end{remark}

\begin{proof}
We introduce three intermediate reward sequences 
to make the triangle inequality argument explicit. 
Define:
\begin{itemize}
  \item $r_t^{(0)} \triangleq r_t^{\text{sim}}$: 
  reward under ideal simulation (ground-truth BEV 
  $\hat{o}_t^{\mathrm{sim}}$, perfect curvature 
  tracking, observation distribution $\mathcal{O}_1$);
  \item $r_t^{(1)}$: reward after substituting 
  GOB-generated BEV $\hat{o}_t^{\mathrm{real}}$ for 
  ground-truth BEV $\hat{o}_t^{\mathrm{sim}}$, while 
  keeping perfect curvature tracking and operating 
  under $\mathcal{O}_1$;
  \item $r_t^{(2)}$: reward after further introducing 
  PAM curvature tracking error 
  $|\Delta\kappa_t| \leq \epsilon_{\mathrm{pid}}$, 
  while operating under the Phase~2 observation 
  distribution $\mathcal{O}_2$;
  \item $r_t^{\text{real}}$: actual deployment reward 
  (GOB error + PAM error + residual distribution gap 
  between $\mathcal{O}_2$ and $\mathcal{O}^{\text{real}}$).
\end{itemize}
By the triangle inequality:
\begin{equation}
r_t^{\text{sim}} - r_t^{\text{real}}
= \underbrace{\left(r_t^{(0)} - r_t^{(1)}
\right)}_{\text{GOB error}}
+ \underbrace{\left(r_t^{(1)} - r_t^{(2)}
\right)}_{\text{PAM error}}
+ \underbrace{\left(r_t^{(2)} - r_t^{\text{real}}
\right)}_{\text{TPT residual}}.
\end{equation}

\medskip
\noindent\textbf{Step~1 (GOB term).}
Substituting $\hat{o}_t^{\mathrm{real}}$ for 
$\hat{o}_t^{\mathrm{sim}}$ changes the policy 
output by at most 
$\|\pi_\theta^{\text{sim}}(\hat{o}_t^{\mathrm{sim}}) 
- \pi_\theta^{\text{sim}}(
\hat{o}_t^{\mathrm{real}})\|_2 \leq L_\pi
\epsilon_{\text{seg}}$ 
(Proposition~\ref{prop:gob_bound}). By the 
$L_r$-Lipschitz reward 
(Assumption~\ref{ass:lipschitz_reward}), the 
per-step reward difference satisfies
$|r_t^{(0)} - r_t^{(1)}| \leq L_r L_\pi 
\epsilon_{\text{seg}}$.
Summing over the discounted horizon using
$\sum_{t=0}^T \gamma^t \leq 1/(1-\gamma)$:
\begin{equation}
\sum_{t=0}^{T} \gamma^t\,\mathbb{E}\!\left[
r_t^{(0)} - r_t^{(1)}\right]
\;\leq\;
\frac{L_r L_\pi\,\epsilon_{\mathrm{seg}}}{1-\gamma}.
\end{equation}

\medskip
\noindent\textbf{Step~2 (PAM term).}
Imperfect curvature tracking 
($|\Delta\kappa_t| \leq \epsilon_{\text{pid}}$)
causes a lateral position error bounded by
$\frac{v_{\max}^2 T^2}{2}\epsilon_{\text{pid}}$
(Proposition~\ref{prop:pam_tracking}). By the 
$L_r$-Lipschitz reward:
\begin{equation}
\sum_{t=0}^{T} \gamma^t\,\mathbb{E}\!\left[
r_t^{(1)} - r_t^{(2)}\right]
\;\leq\;
\frac{L_r\,v_{\max}^2 T^2\,
\epsilon_{\mathrm{pid}}}{2(1-\gamma)}.
\end{equation}

\medskip
\noindent\textbf{Step~3 (TPT residual).}
The remaining gap between operating under 
$\mathcal{O}_2$ and $\mathcal{O}^{\text{real}}$ 
is bounded by Proposition~\ref{thm:tpt_bound}:
\begin{equation}
\sum_{t=0}^{T} \gamma^t\,\mathbb{E}\!\left[
r_t^{(2)} - r_t^{\text{real}}\right]
\;\leq\;
\frac{2\,R_{\max}\,
  d_{\mathrm{TV}}\!\left(\mathcal{O}_2,\,
  \mathcal{O}^{\text{real}}\right)
}{(1-\gamma)^2}.
\end{equation}

\noindent Summing Steps~1--3 and rearranging 
yields Eq.~\eqref{eq:main_formal}.
\end{proof}

\begin{corollary}[Convergence under Ideal 
Conditions]
\label{cor:convergence}
If $\epsilon_{\mathrm{seg}} \to 0$, 
$\epsilon_{\mathrm{pid}} \to 0$, and
$d_{\mathrm{TV}}(\mathcal{O}_2, 
\mathcal{O}^{\text{real}}) \to 0$, then:
\begin{equation}
\mathbb{E}\!\left[\sum_{t=0}^{T} \gamma^t 
r_t^{\text{real}}\right]
\;\longrightarrow\;
\mathbb{E}\!\left[\sum_{t=0}^{T} \gamma^t 
r_t^{\text{sim}}\right].
\end{equation}
\end{corollary}

\begin{proof}
Immediate from Theorem~\ref{thm:main_formal} by 
taking all three error terms to zero.
\end{proof}

\begin{remark}
Corollary~\ref{cor:convergence} establishes that 
Sim2Real-AD is \emph{asymptotically lossless}: as 
each module approaches its theoretical ideal, 
real-vehicle performance converges to simulation 
performance under the VLM-guided RL objective 
Eq.~\eqref{eq:rl_objective}. The three error terms 
in Theorem~\ref{thm:main_formal} are independently 
controllable: GOB via improved segmentation 
($\epsilon_{\mathrm{seg}} \downarrow$), PAM via 
tighter path tracking 
($\epsilon_{\mathrm{pid}} \downarrow$), and TPT via 
longer Phase~2 training
($d_{\mathrm{TV}}(\mathcal{O}_2, 
\mathcal{O}^{\text{real}}) \downarrow$), defining 
three orthogonal axes along which the framework 
can be improved.
\end{remark}

\bibliographystyle{model1-num-names}
\bibliography{mybibfile}

\end{document}